\documentclass{article}

\usepackage{amsmath,dsfont}
\usepackage{lmodern}
\usepackage{amssymb,amsthm}
\usepackage{anysize,hyperref,xcolor}
\usepackage{enumerate}
\usepackage{cleveref}
\usepackage{bbold}
\usepackage{kpfonts}
\usepackage{ulem}
\usepackage{hyperref}
\newcommand{\footremember}[2]{%
	\footnote{#2}
	\newcounter{#1}
	\setcounter{#1}{\value{footnote}}%
}
\newcommand{\footrecall}[1]{%
	\footnotemark[\value{#1}]%
}

\newcommand{\prox}{\operatorname{prox}}
\newcommand{\proj}{\operatorname{proj}}

\def\argmax{\operatornamewithlimits{arg\,max}}
\def\argmin{\operatornamewithlimits{arg\,min}}

\newcommand{\dom}{\text{dom}}

\newcommand{\eni}{\begin{equation}}
\newcommand{\enf}{\end{equation}}
\newcommand{\be}{\begin{equation}}
\newcommand{\ee}{\end{equation}}

\newcommand{\br}{\begin{rem}}
\newcommand{\er}{\end{rem}}
\newcommand{\bex}{\begin{ex}}
\newcommand{\eex}{\end{ex}}

\newcommand{\bc}{\begin{center}}
\newcommand{\ec}{\end{center}}
\newcommand{\bn}{\begin{enumerate}}
\newcommand{\en}{\end{enumerate}}
\newcommand{\bi}{\begin{itemize}}
\newcommand{\ei}{\end{itemize}}
\newcommand{\beas}{\begin{eqnarray*}}
\newcommand{\eeas}{\end{eqnarray*}}
\newcommand{\bea}{\begin{eqnarray}}
\newcommand{\eea}{\end{eqnarray}}

\newcommand{\bt}{\begin{thm}}
\newcommand{\et}{\end{thm}}
\newcommand{\bl}{\begin{lem}}
\newcommand{\el}{\end{lem}}
\newcommand{\bd}{\begin{defi}}
\newcommand{\ed}{\end{defi}}
\newcommand{\bp}{\begin{proof}}
\newcommand{\ep}{\end{proof}}
\newcommand{\bq}{\begin{que}}
\newcommand{\eq}{\end{que}}
\newcommand{\bas}{\begin{ass}}
\newcommand{\eas}{\end{ass}}

\newtheorem{thm}{Theorem}[section]
\newtheorem{defi}{Definition}[section]
\newtheorem{lem}{Lemma}[section]
\newtheorem{prop}{Proposition}[section]

\newtheorem{rem}{Remark}[section]
\newtheorem{ex}{Example}[section]
\newtheorem{ass}{Assumption}[section]
\newtheorem{que}{Question}[section]

\providecommand{\nor}[1]{\left\lVert {#1} \right\rVert}
\providecommand{\bf}[1]{\mathbf{#1}}

\providecommand{\abs}[1]{\lvert{#1}\rvert}

\providecommand{\scal}[2]{\left\langle{#1},{#2}\right\rangle}

\providecommand{\hi}[1]{\left|{#1} \right|_+}

\newcommand{\R}{\mathbb R}

\newcommand{\hh}{\mathcal H}

\newcommand{\la}{\lambda}

\newcommand{\Z}{Z}

\usepackage{graphicx}

\usepackage{colortbl}

\definecolor{greenyy}{rgb}{0.0,0.6,0.3} 
\newcommand{\vas}[1]{\color{greenyy}{#1} \color{black}} 

\newcommand{\grandO}[1]{O\mathopen{}\left(#1\right)}

\providecommand{\norme}[1]{\left\lVert {#1} \right\rVert}
\newcommand{\off}[1]{}
\usepackage{algorithm,algorithmic}
\usepackage{ dsfont }
\usepackage{ amssymb }
\providecommand{\ackn}[2]{\textbf{Acknowledgements :} #2}

\newtheorem{theorem}{Theorem}[section]
\newtheorem{assumption}{Assumption}

\numberwithin{equation}{section}

\title{Iterative regularization in classification via hinge loss diagonal descent }

\author{
Vassilis Apidopoulos \footremember{Dibris}{MaLGa, DIBRIS, Universit\`a di Genova, Genoa, Italy. E-mail: vassilis.apid@gmail.com (Vassilis Apidopoulos)}\footnote{Archimedes, Athena RC, Athens, Greece.}%
\and Tomaso Poggio\footremember{mit}{CBMM, MIT, MA, USA. E-mail: tp@ai.mit.edu (Tomaso Poggio), 	lorenzo.rosasco@unige.it (Lorenzo Rosasco)}%
\and Lorenzo Rosasco\footrecall{Dibris} ~\footrecall{mit} ~\footnote{Istituto Italiano di Tecnologia, Genoa, Italy.}
\and Silvia Villa\footnote{MaLGa, DIMA, Universit\`a di Genova, Genoa, Italy. E-mail: villa@dima.unige.it (Silvia Villa)}%
}
\date{}

\begin{document}

\maketitle

\begin{abstract}

Iterative regularization is a classic idea in regularization theory,  that has recently become popular  in machine learning. On the one hand, it allows to design efficient algorithms controlling at the same time  numerical and statistical accuracy. On the other hand it allows to shed light on the  learning curves observed while training neural networks. In this paper,  we focus on  iterative  regularization in the context of classification. After contrasting this  setting with that of linear inverse problems, we develop an iterative regularization approach based on the use of the hinge loss function. More precisely we consider a diagonal approach for a family of algorithms for which we prove convergence as well as rates of convergence and stability results for a suitable classification noise model.
Our approach compares favorably with other alternatives, as confirmed  by numerical simulations. 

\end{abstract}


\section{Introduction}

Estimating a quantity of interest from finite measurements is a central problem in a number of fields including  inverse problems but also machine learning, statistics and  signal processing. In this context, a key idea  is that reliable estimation  requires imposing some prior assumptions on the problem at hand. Regularization  theory for inverse problems provides a principled framework to formalize this  idea \cite{engl1996regularization}. The quantity of interest is typically seen as a function, or a vector, and prior assumptions take the form of suitable functionals, called regularizers. Following this idea, Tikhonov regularization provides a classic  approach to estimate solutions \cite{tihonov1963solution,thikhonov1976methodes}.  The latter are found  minimizing an empirical objective where a data fit term is penalized adding a chosen regularizer.
Other regularization approaches are classic in inverse problems,  and in particular iterative regularization has become popular in machine learning, see e.g. \cite{yao2007early,raskutti2014early,stankewitz2021inexact}. This approach is based on the observation that iterative optimization procedures have a self-regularizing property, so that  a chosen  regularization can be  enforced implicitly along the iterations. Such an observation seems to shed light on some theoretical properties of deep learning approaches in machine learning \cite{hardt2016train,neyshabur2017geometry,gunasekar2018characterizing, gunasekar2018implicit,rangamani2021dynamics}. More generally, iterative regularization provides an approach to design algorithms striking a balance between statistical accuracy and computational efficiency \cite{bottou201113,zhang20015Boosting,rosasco2015learning,matet2017don}. 

In this paper, we focus on iterative regularization in the context of classification, perhaps the most classical among  machine learning problems \cite{cortes1995support,steinwart2008support,shalev2014understanding}. After discussing the differences and similarities between classification and classical linear inverse problems, we recall how different iterative regularization schemes for classification can be  defined depending on the considered loss function. Then, we  focus on the hinge loss used in support vector machines \cite{steinwart2008support}. In this case, compared to other loss functions  such as the exponential and logistic loss  \cite{nacson2018convergence,soudry2018implicitGD,ji2019implicit}, a simple gradient approach does not allow to establish iterative regularization properties. Indeed,  we propose a diagonal approach in the same spirit of \cite{garrigos2018iterative} and its inertial version \cite{calatroni2019accelerated} and prove their regularization properties, including convergence  and stability. The proposed approach compares favorably to analogous results for the logistic loss \cite{nacson2018convergence,soudry2018implicitGD,ji2019implicit,ji2020gradient}, but also with recent  approaches considering the hinge loss \cite{molitor2020bias}. Indeed, we show that fast convergence rates are possible, largely  improving previous results. Further, we prove the first  stability results under a suitable classification noise model which is inspired by the deterministic noise classically  considered in linear inverse problems. We note that  in a noiseless setting,  our approach can also be seen as a way to solve the basic  separable support vector machines problem introduced in the seminal work \cite{cortes1995support}. In this view, relevant studies, among the rich literature, can be found for example in \cite{cortes1995support,freund1998large,cristianini2000introduction,chapelle2007training,shalev2011pegasos}.
 Our theoretical results are illustrated via numerical simulations, where we also investigate empirically  the stability properties of the proposed methods.

The rest of the paper is organized as follows. In Section \ref{section2}, we briefly discuss the ideas of explicit and implicit regularization approach for regression problems. In Section \ref{section3} we are reviewing these approaches in terms of classification problems. In Section \ref{sectionprel},  we describe the main proposed schemes based on a diagonal iterative regularization procedure for the hinge loss. In Section \ref{sectionmain} we present the main results and we provide the corresponding convergence and stability analysis. Finally in Section \ref{sectionnum} we illustrate the performance of the proposed algorithms on some simple numerical examples. Appendix \ref{appendixa} contains some general facts and all  the technical proofs and lemmas.

\subsection{Notation}
\label{sec:notation}
We first introduce some notations and recall a few basic notions that will be needed throughout the paper. The interested reader can consult \cite{bauschke2011convex} regarding the main tools and their associated properties used in this work. 
Let $\scal{\cdot}{\cdot}$ denote the standard Euclidean inner product and $\norme{\cdot}$ the associated Euclidean norm. For a linear operator $Z:\R^n\to \R^{m}$, we denote with $\Im(Z)$ and $\ker(Z)$ its range and kernel respectively We also note with $\norme{Z}_{op}$ and $\norme{Z}_{F}$ its operator and Frobenious norm respectively. We also denote with $\text{Id}$ the identity operator and with $\mathbf{1}$ the $n$-vector with entry $1$ in each coordinate. 

Given a convex and closed set $C\subset \R^{n}$, the distance of a point $x\in \R^{n}$ from the set $C$ is $\text{dist}(x,C)=\underset{y\in C}{\inf}\norme{x-y}$. The indicator function of $C$, $\iota_{C}(\cdot):\R^n\to
\bar{\R}:=\R\cup\{+\infty\}$ is defined as
$\iota_{C}(x)=\begin{cases}
	&0 ~ \text{ if } x\in C \\
	&+\infty ~ \text{ if not }
\end{cases}.$ 

For a proper, convex and lower semicontinuous function $f:\R^n\to\bar{\R}$, we define its subdifferential $\partial f : \R^n\to2^{\R^{n}}$, as $\partial f(x)=\{u\in \R^{n}~:~ f(y)\geq f(x) +\scal{u}{y-x} ,~ \forall y\in \R^{n} \}$.
For $\gamma>0$, the proximal operator of $f$, $\prox_{\gamma f}:\R^n\to\R^{n}$ with step $\gamma$, is defined by $\prox_{\gamma f}(x)=\underset{y\in \R^n}{\argmin}\{f(y)+\frac{1}{2\gamma}\norme{x-y}^{2}\}$, for all $x\in \R^{n}$. The projection operator onto a convex and closed set $C\subset \R^{n}$ is defined as $\proj_{C}:\R^{n}\to\R^{n}$ such that  $\proj_{C}(x)=\underset{y\in C}{\argmin}\norme{y-x}$.
We denote the Fenchel conjugate of $f$ as $f^{\ast}:\R^n\to\bar{\R}$, such that $f^{\ast}(x)=\underset{y\in \R^{n}}{\sup}\{\scal{x}{y}-f(y)\}$.

\section{Background: explicit and implicit regularization in regression}\label{section2}
The classical regression problem in supervised learning and statistics corresponds to estimating a function of interest $f$ ,  given a finite number of (possibly noisy) evaluations at a number of input points.
The problem takes a simple and familiar form if $f$ is assumed to be linear. Indeed, in this case it corresponds to  estimating  $w\in \R^d$, given a set of equations,   
\be\label{inter}
\scal{w}{x_i}= y_i, \qquad i=1, \dots, n,
\ee
where $x_i \in \R^d$ and $y_i\in \R$.

The above problem can be restated as the linear inverse problem 
\be\label{ip}
Xw=y,
\ee
where  $X:\R^d\to \R^n$ is a matrix  with rows the input points, and $y
\in \R^n$ is a vector with entries the outputs.
Moreover, the simple linear case can be generalized if $f$ belongs to a reproducing kernel Hilbert space (RKHS).
\br[Regression in RKHS \cite{steinwart2008support}]
 Recall that a Hilbert space $\mathcal{H}$ of real valued functions on a set $\cal X$  is called a RKHS if for all $f\in \hh$ and $x\in \cal X$, there exists $C_x\in \R$ such that  $ |f(x)|\le C_x\nor{f}_\hh.$
 From the above definition  and Riesz lemma it follows immediately that   there exists a $k_x\in \hh$ such that 
 $
 f(x)=\scal{f}{k_x}_\hh.
 $
 Then we can write the regression problem as the problem of finding $f\in \hh$, given a set of equations,   
$$
\scal{f}{k_{x_i}}_\hh= y_i, \qquad i=1, \dots, n,
$$
where $x_i \in \cal X$ and $y_i\in \R$. 
 \er
\noindent Keeping the above remark in mind, in the following we primarily focus on the linear case to ease the notation. 

A key observation is that the problem~\eqref{ip} might not admit solutions or admit multiple ones. This latter situation is the most common  in machine learning, where high dimensional (overparameterized), or even infinite dimensional (like RKHS) models are often considered. In the linear setting,  this corresponds to the case where $n<d$ in~\eqref{ip}, (assuming the inputs to be linearly independent). In this case, problem~\eqref{ip} admits infinitely many solutions and a classic way to select one is to consider the minimum norm solution
\begin{equation}\label{minimalnormsolution1}
w^{\dagger}= \argmin_{w\in \R^d} \big\{
\nor{w} ~ : ~  \scal{w}{x_{i}}= y_{i} , \quad i=1, \dots, n \big\}.
\end{equation} 
The minimum norm solution can be written in terms of the pseudoinverse of $X$ as
$w^{\dagger}=X^{\dagger} y$ \cite[Definition $2.2$]{engl1996regularization}. This makes it clear that instability  in the solution might occur when $X$ is ill conditioned. The basic idea of regularization is to consider a family of estimates $\{w_{\nu}\}_{\nu>0}$ that approaches $w^{\dagger}$ with better stability properties.
In this sense, the regularization property of such estimates is related with i) the convergence of $w_{\nu}$ to the minimal norm solution $w^{\dagger}$ and ii) its stability. Here the notion of stability expresses how close are the estimates $\{w_{\nu}\}_{\nu>0}$ and $\{\tilde{w}_{\nu}\}_{\nu>0}$ that are generated from the true output $y$ and a noisy version of it $\tilde{y}$ (respectively), provided $y$ and $\tilde{y}$ are close enough (see e.g. \cite[Section $3$ \& Definition $3.1$]{engl1996regularization} for a detailed definition).
We next recall two basic approaches with this property.

\paragraph{Tikhonov (explicit) regularization.} The most classic regularization approach is  Tikhonov regularization
\begin{equation}\label{tikhonovpath1}
w_{\lambda}=\argmin_{w\in R^d}
\nor{Xw-y}^2
+\lambda \nor{w}^2,
\end{equation} 
where $\la>0$ is called the regularization parameter. The set of solutions corresponding to different values of $\lambda$ defines the regularization method and a direct computation shows that 
$w_{\lambda}=(X^\top X+\la \text{Id})^{-1}X^\top y$. From this expression is easy to see that 
the sequence $w_\la$ converges to $w^\dagger$ as $\la$ tends to zero, \off{\sout{showing the regularization property of the method} ? } (see e.g. \cite[Theorem $5.2$]{engl1996regularization}).
 The above ideas can be generalized as we discuss next. 
\br[Loss and regularizers]
The ideas in~\eqref{minimalnormsolution1},~\eqref{tikhonovpath1} can be generalized  replacing the squared norm in~\eqref{tikhonovpath1} with other regularizers $R:\R^d\to  [0,\infty)$, the $\ell_1$ norm being a popular example. Similarly instead of the least squares error in~\eqref{tikhonovpath1} 
other error cost functions (loss functions) $\ell:\R\to [0,\infty)$ can be considered. In regression, loss functions depend on the difference $y-f(x)$, examples besides the square loss include the absolute value loss, or the $epsilon$-insensitive loss used in support vector machine regression 
\cite{steinwart2008support}. For general losses,  the corresponding  solutions might not admit a closed form expression, but it is still possible to prove that the Tikhonov regularized solutions converge in some proper sense to the corresponding minimum norm solutions (see e.g. \cite{dontchev92,attouch1996viscosity}). For sake of completeness we provide a proof in a general setting in Lemma \ref{tikhonovlemma} in Appendix \ref{appendixa}.
\er
 In machine learning, regularization {\em \`a la Tikhonov} is sometimes called explicit since a  penalty is added to the data fit term. We next recall  iterative regularization and discuss  why  it is called implicit in machine learning \cite{Gunasekar2017}. 
 
 \paragraph{Iterative (implicit)  regularization.}
 The simplest example of iterative regularization is  the gradient descent iteration of the least squares error, that is
\[
 w_{t+1}= w_t-\gamma 
 X^\top (Xw-y)
\]
for some suitable initialization $w_0$ and step size $\gamma$.
It is well known that the above iteration converges to the minimal norm solution~\eqref{minimalnormsolution1}, and further that the stability of the solution varies along the iterations \cite[Section $6.1$]{engl1996regularization}. In this view, the stopping time becomes the regularization parameter defining a family of regularized solutions. This kind of regularization is called implicit in machine learning since there is no explicit penalty or constraint in the optimization model \cite{neyshabur2014search,neyshabur2017geometry}, and stability relies on the self-regularizing properties of the optimization process.
 These ideas have recently received a lot of attention in machine learning, since they are useful to understand the theoretical properties of more complex non-linear systems like neural networks \cite{Gunasekar2017,chizat2020implicit}. Further, they have been advocated as a way to  design resource efficient machine learning algorithms controlling at the same time  numerical and statistical accuracy.  In the remainder of the paper we discuss how these ideas extend beyond regression to the classification setting. 

We first add some remarks. First,  we note that other optimization approaches than gradient descent can be, and have been considered, including accelerated \cite{engl1996regularization,Neubauer2016OnNA,pagliana2019implicit,kindermann2021optimal,zhang2023acceleration} and stochastic methods \cite{rosasco2014convergence,mucke2019beating,dieuleveut2020bridging}, as well as mirror descent approaches \cite{gunasekar2021mirrorless,jin2023stochastic}. Second we note that,  
compared to Tikhonov regularization,  for iterative regularization the extension to other losses and regularizers is not  straightforward. The case of regularizers that are norms in reflexive Banach space has been considered in \cite{schuster2012regularization,brianzi2013preconditioned}, whereas the case of strongly convex regularizers has been considered in \cite{matet2017don,gunasekar2018characterizing}. The case of convex regularizers has been recently considered in \cite{molinari2021iterative}. The case of smooth loss functions for regression is quite straightforward, considering the gradient iteration
\begin{equation}\label{gdr1}
 w_{t+1}= w_t-\gamma \sum_{i=1}^{n} x_i\ell'(y_i-\scal{w_t}{x_i}), \quad t= 0, \dots, T,
\end{equation}
 where $\ell'$ is the derivative of the loss. 
 The case of convex but non smooth loss function can also be considered using subgradient methods \cite{lin2016Iterative}. Note that, if the above iteration is initialized in the span of the input points, it remains in the span. This observation is at basis of the proof that the above iteration converges to the minimal norm interpolant~\eqref{minimalnormsolution1}, see Lemma \ref{gradientdescentlemma} in Appendix \ref{appendixa}. 
 
 Provided, with the above background we next discuss the case of classification.
%


\section{Implicit regularization: from regression to classification}\label{section3}

In this section, we introduce the problem of classification and investigate how the ideas reviewed in the previous section for regression and linear inverse problems translate to this setting. In particular, we first discuss a  notion of minimal norm solution for classification. 

Similarly to regression, in classification the goal is to estimate a functional relationship, but the difference is that the  outputs are  binary valued, that is $y_i\in\{-1,1\}$. Estimating a binary valued function is computationally unfeasible and a classic approach relies on estimating a real valued function $f$ of which the sign is then taken,  that is $\text{sign}(f(x))= 1$ if $f(x)>0$, and $\text{sign}(f(x))= -1$ otherwise (here ties are broken arbitrarily). In this context, a natural quantity is the product $yf(x)$, 
called the {\em margin} of $f$ at $(x,y)$.
If the margin is positive it means that $f$ will classify correctly the  input point,  if the margin is large we can intuitively expect a confident prediction. 

For linear functions we can formalize this idea (see e.g. \cite{cortes1995support,steinwart2008support}) considering the problem of finding $w\in \R^d$ satisfying the set of inequalities
\be\label{linsep}
y_i\scal{w}{x_i}>0, \quad i=1, \dots, n.
\ee
If such a $w$ exists we say that the data $(x_1, y_1), \dots, (x_n,y_n)$ are linearly separable. 
Deriving necessary and sufficient conditions for the above inequalities to be feasible is not an elementary problem.  As shown below, it is easy to see that overparametrization ($n<d$), hence interpolation, will be a sufficient condition for linear separability. Since this is the relevant setting for us, from now on, we assume that the data are linearly separable:
\begin{assumption}[Linear separability]\label{assumptionseparate}
There exists some ${w}\in \R^d$ that separates the data, i.e. \begin{equation} y_{i}\scal{{w}}{x_{i}} >0 \qquad \forall i=1,...,n
\end{equation}
\end{assumption}

In general, also in this setting we can expect multiple solutions,  the classical way 
to approach the problem is to consider the so called max margin solution.
The linear margin $M:\R^d\to \R$ for a given dataset is defined as, 
\be\label{margin}
M(w)= \min_{i=1, \dots, n} y_i\scal{w}{x_i},
\ee
and, correspondingly, the  maximum (max) margin problem is defined as 
\begin{equation}\label{max_sphere}\tag{MM}
w_{+}=\argmax \{ M(w) ~ : ~ \norme{w}=1\}.
\end{equation}
We make a few observations.
First, the intuition underling the above problem is that among all separating solutions,
we  are interested into  one for which the margin is maximized.  Second, we note that without the unit norm constraint the problem of maximizing the margin \eqref{margin} is degenerate and one can obtain trivial solutions by rescaling arbitrarily any separating solution. Indeed, since the margin is scale invariant, the max margin becomes a direction problem, which naturally leads to considering the constrained problem \eqref{max_sphere} (see e.g. \cite{cortes1995support}).
Third, it is possible to show that the max margin problem has an equivalent formulation, that highlights the connection to minimal norm solutions.  Indeed, consider the problem
\begin{equation}\label{min_norm}\tag{MN}
w_\ast= \argmin_{w\in \R^d} \big\{\norme{w}~ : ~ y_{i}\scal{w}{x_{i}}\geq 1 ~, ~ \forall i\leq n \big\}.
\end{equation}
Then the following result holds.
\begin{lem}\label{lemmaequivalencemaxmin}
Problem~\eqref{max_sphere} is equivalent to Problem~\eqref{min_norm}. 
In particular, if $w_\ast$ is a solution of Problem~\eqref{min_norm} then 
$$w_+=\frac{w_\ast}{\nor{w_\ast}}$$ 
is a solution of Problem~\eqref{max_sphere} and $M(w_{+})=\frac{1}{\norme{w_{\ast}}}$. 
Moreover, if $w_+$ is the solution of Problem~\eqref{max_sphere} then 
$$w_\ast=\frac{w_+}{M(w_+)}$$ 
is a solution of Problem~\eqref{min_norm}. 
Further,  it holds that $M(w_{\ast})=1$
\end{lem}

A few observations can be made.
The above lemma is a classical result in the theory of SVM. Indeed,  Problem~\eqref{min_norm} is called hard margin support vector machines  \cite{cortes1995support}. For sake of completeness we provide its proof (see Lemma \ref{lemmaClassificationEquivalence} in Appendix \ref{appendixa}). In terms of the constrained min-norm problem \eqref{min_norm}, Assumption \ref{assumptionseparate}, ensures that the feasible set is non-empty and hence such a solution $w_{\ast}$ (thus $w_{+}$) exists. In addition since \eqref{min_norm} can also be equivalently expressed as
	\begin{equation}\label{minorm}
		w_\ast= \argmin_{w\in \R^d} \left\{\frac{\norme{w}^{2}}{2}~ : ~ y_{i}\scal{w}{x_{i}}\geq 1 ~, ~ \forall i\leq n \right\},
	\end{equation} 
the solution $w_{\ast}$ (thus $w_{+}$) is unique, thanks to the strong convexity of the squared norm in \eqref{minorm}. Note that,from an optimization point of view, formulation \eqref{minorm} is more useful and will be used hereafter, instead of \eqref{min_norm}.

From a regularization perspective,  the above result shows that the max margin problem can  be see  as a minimum norm problem akin to Problem~\eqref{minimalnormsolution1} in linear inverse problems, but the linear equations are now replaced by inequalities. This can be made even more explicit noting that for binary valued outputs, it holds:
$$
y_{i}=\scal{w}{x_{i}}~\Leftrightarrow ~y_{i}\scal{w}{x_{i}}= 1
$$
This last expression also clarifies the earlier observation that interpolation implies separation, and that $n<d$ with linearly independent inputs is a sufficient condition for linear separability.

We next discuss how the ideas of explicit and implicit regularization can be adapted to the classification context. Note that in the following, in analogy to the regression case, we will say that a family of solutions has the regularization property if it converges to the max margin (min norm) solution \and is stable with respect to a noisy version of the true output $y$ (see the related discussion in paragraph \ref{section stability}).

\paragraph{Explicit regularization for classification.} 
To extend  Tikhonov regularization approach to classification, an appropriate loss function $\ell:\R\to [0,\infty)$ needs to be considered. As it turns out in classification, loss functions depend on the margin $yf(x)$,  rather than the difference $y-f(x)$ as in regression, and indeed are called margin loss functions. Some popular examples are the following:

\hspace{-6mm}\begin{minipage}{0.52\textwidth}
 	\begin{itemize}
		\item hinge loss : ~ \(\ell(a)=\abs{1-a}_{+}=\max\{0,1-a\}\)
		
		\item exponential loss : ~\(\ell(a)=e^{-a}\)
		
		\item logistic loss : ~\(\ell(a)=\ln(1+e^{-a})\)
	\end{itemize}
\end{minipage} \hfill
\begin{minipage}{0.48\textwidth}
	\begin{figure}[H]
		\includegraphics[trim=0mm 4mm 19cm 0mm, scale=0.8
		]{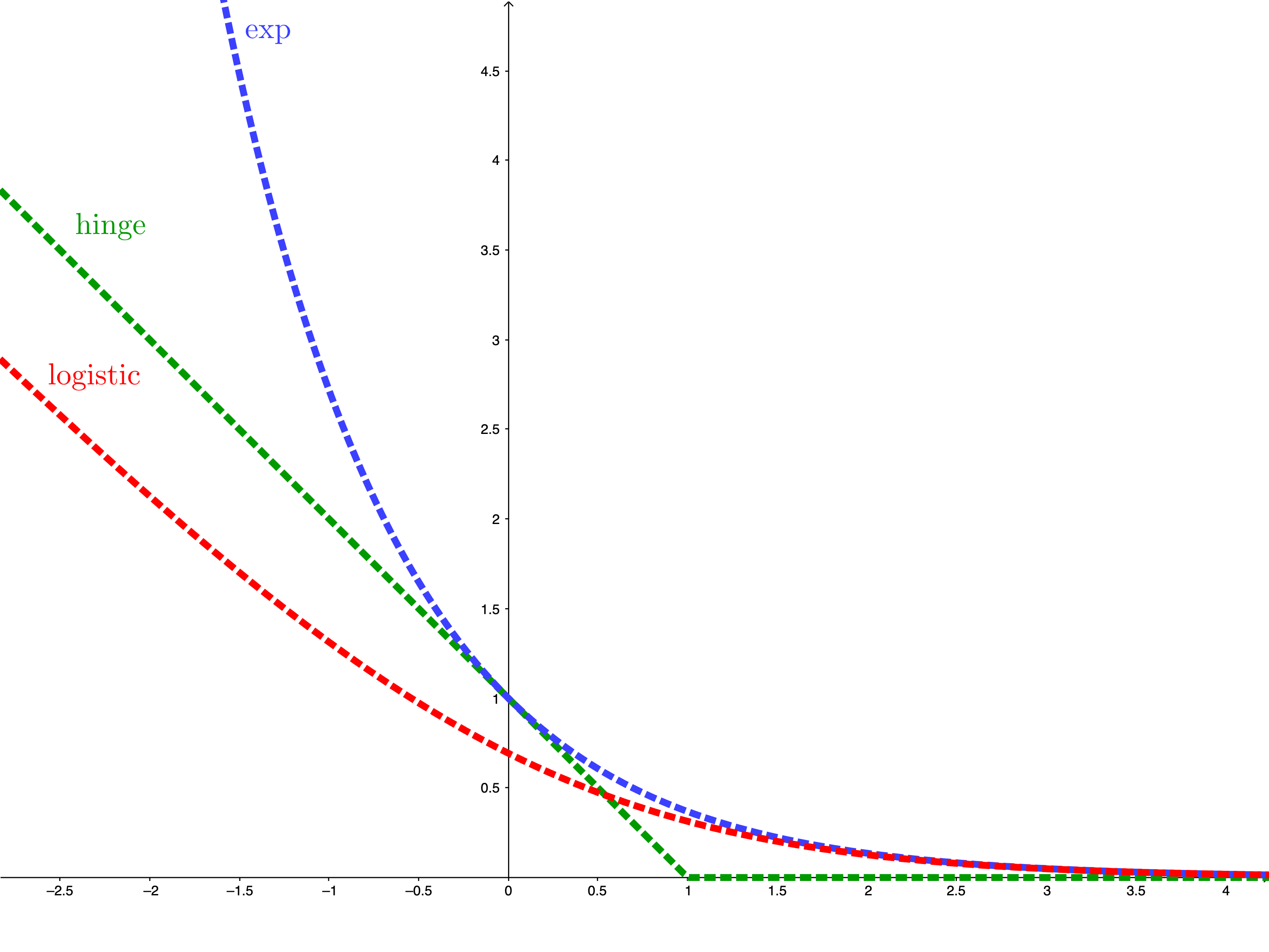}
		\caption{The exponential (blue), logistic (red) and hinge (green) loss function.}\label{Figurelosses}
	\end{figure}
\end{minipage} 

\vspace{5mm}

Given a convex  margin loss function, and considering again linear models, 
Tikhonov regularization in classification corresponds to 
\begin{equation}\label{tikhonovclass}
w_{\lambda}=\argmin_{w\in R^d}\sum_{i=1}^n 
\ell(y_i\scal{w}{x_i})
+\lambda \nor{w}^2,
\end{equation} 
for $\la>0$. It is natural to ask whether the above approach can be shown to be regularizing in the sense that the sequence $\{w_\la\}_{\la>0}$ converges to the minimum norm solution~\eqref{min_norm}. Indeed, for the hinge loss this can be proven, while in the cases of the exponential or logistic loss, where the set of minimizers is empty, one can show that $w_{\lambda}$ diverges (see e.g. Lemma \ref{tikhonovlemma} in Appendix \ref{appendixa}). Nevertheless under some suitable assumptions on the loss function, convergence in direction that is \[\frac{w_{\lambda}}{\norme{w_{\lambda}}} \underset{\lambda\to 0}{\to}\frac{w_{\ast}}{\norme{w_{\ast}}}=w_{+}\] has been proven for a wide class of margin loss functions (see e.g. \cite[Theorem $2.1$]{rosset2004margin} and \cite{hastie2004entire,rosset2004boosting}). Notice that according to the previous discussion and Lemma \ref{lemmaequivalencemaxmin}, since the max margin problem \eqref{max_sphere} is a direction problem, it is still relevant to consider convergence in direction. We next turn our attention to implicit regularization reviewing some recent results.

\paragraph{Implicit regularization for classification.}
Following the same reasoning as in regression for the minimal norm interpolating solution \eqref{minimalnormsolution1} and the gradient method \eqref{gdr1}, it is natural to ask whether it is possible to prove the regularization properties of the gradient iteration applied to a margin loss function, i.e.
\be\label{gdc}
 w_{t+1}= w_t-\gamma \sum_{i=1}^{n} y_ix_i\ell'(y_i\scal{w_t}{x_i}), \quad t= 0, \dots, T.
 \ee
 where $\gamma >0$ is a suitable step size.
The above iteration is well defined for smooth loss functions such as the exponential or the logistic loss. However, for losses like exponential or logistic which do not admit any minimizer, the gradient descent iteration in~\eqref{gdc} diverges. 
Interestingly, in a recent line of works (see \cite{soudry2018implicitGD,nacson2018convergence,ji2019implicit} and references therein), convergence in {\em}direction of gradient descent has been proved for these loss functions, i.e.
$$
\lim_{t\to \infty}\frac{w_t}{\nor{w_t}}= w_+ .
$$
In addition, rates of convergence for the normalized iterates $\norme{\frac{w_t}{\nor{w_t}} - w_+}$,  the angle gap  $1-\scal{\frac{w_{t}}{\norme{w_{t}}}}{w_{+}}$, and the normalized margin gap $M\left(\frac{w_{t}}{\norme{w_{t}}}\right)- M(w_{+})$ were also proved, but are very slow. For example for the logistic and exponential losses the margin rate is of order \(\grandO{\frac{1}{\log t}}\) (see \cite{soudry2018implicitGD}). Improved rates of order $\grandO{\frac{\log t}{\sqrt{t}}}$ were given recently considering a variable step-size gradient descent version in \cite{nacson2018convergence}. Similar works include also (accelerated) mirror descent approaches \cite{ji2020gradient,ji2021fast} that provide margin rates up to $\grandO{\frac{\log t}{t^{2}}}$.
Concerning  the hinge loss, the natural extension of the above idea is to consider a subgradient iteration 
\be\label{smc}
 w_{t+1}= w_t-\gamma_t \sum_{i=1}^{n} y_ix_i g_i(w_t), \quad t= 0, \dots, T.
 \ee
 for some sequence of step-sizes $(\gamma_t)_t$ and where $g_i(w_t) \in \partial \ell(y_i\scal{w_t}{x_i})$ is an element of the subgradient of the loss. In this case, the minimization problem 
 $$
\min_{w\in \R^d} \sum_{i=1}^{n} 
|1-y_i\scal{w_t}{x_i}|_+
$$
does have a solution and indeed the iteration in~\eqref{smc} converges to it. 
However, the solution minimizing the hinge loss error cannot be expected
to be the max margin (min norm) solution in general, and thus the subgradient iteration in \eqref{smc} does not a provide regularization properties.

Next, we focus on the hinge loss and provide two iterative regularization schemes via a diagonal principle. For ease of reading, we first present and describe the two main iterative methods and then the main convergence and stability results characterizing their regularization properties.

\section{Iterative regularization for hinge loss via diagonal descent}\label{sectionprel}

In this section,  we present  two  iterative regularization approaches based on the hinge loss. 
The first is given in  Algorithm \ref{algodualprojGD}, and the second  in Algorithm \ref{algodualinertialGD}   corresponds to a practically faster variant.



\begin{algorithm}[H]
	\caption{Projected iterative GD on the dual}\label{algodualprojGD}
	Let $\{\lambda_{t}\}_{t\geq0}$ be a decreasing-to-zero sequence of positive numbers, $u_{0}\in [-\lambda_{0}^{-1},0]^{n}$ and $0<\gamma\leq \norme{XX^\top}_{\text{op}}^{-1}$. For all $t\geq 0$, consider $g_{t}=(g_{t}^{i})_{i=1,...,n}$, $u_{t}=(u_{t}^{i})_{i=1,...,n}$ and $w_{t}=(w_{t}^{l})_{l=1,...,d}$, such that for all $i=1,...,n$ and $l=1,...,d$ :
	\begin{align}
		g_{t}^{i}& = 
		\off{ 
			\big(Id-\gamma XX^\top\big)u_{t} =
		}
		u_{t}^{i}-\gamma\sum_{j=1}^{n}y_{i}y_{j}\scal{x_{i}}{x_{j}}u_{t}^{j} \\
		u_{t+1}^{i}&=
	\begin{cases}
			-\frac{1}{\lambda_{t}}  & \text{ if }  g_{t}^i<\gamma-\frac{1}{\lambda_t}\\
			g_{t}^i-\gamma &  \text{ if } g_{t}^i\in[\gamma-\frac{1}{\lambda_t},\gamma]\\
			0 &  \text{ if } g_{t}^i>\gamma
		\end{cases}   \label{utdualGD}  \\
		w_{t+1}^{l} & = -\sum_{i=1}^{n}y_{i}u_{t+1}^{i}x_{i}^{l}
	\end{align}
\end{algorithm}


\begin{algorithm}[H]
	\caption{Projected iterative i-GD on the dual}\label{algodualinertialGD}
	Let $\{\lambda_{t}\}_{t\geq0}$ a decreasing-to-zero sequence of positive numbers, $u_{0}=u_{1}\in [-\lambda_{0}^{-1},0]^n$ and $0<\gamma\leq \norme{XX^\top}_{\text{op}}^{-1}$. For all $t\geq 1$, consider $q_{t}=(q_{t}^{i})_{i=1,...,n}$, $g_{t}=(g_{t}^{i})_{i=1,...,n}$, $u_{t}=(u_{t}^{i})_{i=1,...,n}$ and $w_{t}=(w_{t}^{l})_{l=1,...,d}$, such that for all $i=1,...,n$ and $l=1,...,d$ :
	\begin{align}
		q_{t}^{i} & = u_{t}^{i} +\alpha_{t}(u_{t}^{i}-u_{t-1}^{i}) \quad , \quad \alpha_{t}=\frac{t}{t+\alpha} \\
		g_{t}^{i}& = \off{
			\big(Id-\gamma XX^\top\big)q_{t} =
		} 
		q_{t}^{i}-\gamma\sum_{j=1}^{n}y_{i}y_{j}\scal{x_{i}}{x_{j}}q_{t}^{j} \\
		u_{t+1}^{i}&=\off{
			\prox_{\frac{\gamma}{\lambda_{t}}\mathcal{L}^{\ast}(\lambda_{t}\cdot)}\big(g_{t}\big) =\frac{1}{\lambda_{t}}\prox_{\gamma\lambda_{t} \mathcal{L}^{\ast}}\big(\lambda_{t} g_{t}\big) \label{inertialprojevaluationt+1} \\
			&= \frac{1}{\lambda_{t}}\biggl(\prox_{\gamma\lambda_{t} \ell^{\ast}}\big(\lambda_{t} g^{i}_{t}\big)\biggr)_{i\leq n}
			=}
		\begin{cases}
			-\frac{1}{\lambda_t} & \text{ if } g_{t}^i<\gamma-\frac{1}{\lambda_t}\\
			g_{t}^i-\gamma & \text{ if } g_{t}^i\in[\gamma-\frac{1}{\lambda_t},\gamma]\\
			0 & \text{ if } g_{t}^i>\gamma
		\end{cases}  \label{utdualiGD}  \\
		w_{t+1}^{l} & =-\sum_{i=1}^{n}y_{i}u_{t+1}^{i}x_{i}^{l}
	\end{align}
\end{algorithm}

We add a few comments before providing a detailed derivation of the two procedures above. 
In both Algorithms \ref{algodualprojGD} and \ref{algodualinertialGD}, $\gamma$ denotes a constant step-size and $\lambda_{t}$ a  vanishing-with-iterations  parameter.  Both procedures are based on  simple and easy to implement iterations  that require only vector multiplication and thresholding operations. 
 The sequence $\{u_{t}\}_{t\geq0}$ represents a dual variable and is computed coordinate-wise via a simple projection operation (see \eqref{utdualGD} and \eqref{utdualiGD}), while $g_{t}$ corresponds to a classical gradient (forward) step related to the square norm \({\norme{\cdot}^2}/{2}\). 
We note that the difference between the two schemes is that in Algorithm \ref{algodualinertialGD} the sequence $g_{t}$ is computed via the auxiliary sequence $q_{t}$ instead of $u_{t}$.  The sequence $q_{t}$ equals to $u_{t}$ extrapolated by the  term $\alpha_{t}(u_{t}-u_{t-1})$, called inertial term,  which plays an important role on the convergence speed of Algorithm \ref{algodualinertialGD}. Finally, $\{w_{t}\}_{t>0}$ corresponds to the primal sequence designed to approximate the min-norm solution $w_{\ast}$. 
In the next section, we discuss the derivation of the above two procedures.

\subsection{Diagonal methods via dual hinge loss}\label{subsectionIterativereg}

Algorithms \ref{algodualprojGD} and \ref{algodualinertialGD}  are based on a so called 
 diagonal regularization process \cite{bahraoui1994convergence} 
 applied to a suitably defined dual  problem.  We next describe these ideas in some detail.
 We note in passing that, while diagonal approaches have been considered before for inverse problems \cite{auslender1987penalty,martinet1978perturbation,boyer1974quelques,bahraoui1994convergence,garrigos2018iterative,calatroni2019accelerated}),  
 we are not aware of their application to classification.


The basic idea of diagonal approaches can perhaps be better explained recalling 
that the penalized functional such as~\eqref{tikhonovclass} converges to minimal norm separating solution~\eqref{minorm}, as the  regularization parameter goes to zero. The idea is to start by considering an optimization procedure for solving Problem~\eqref{tikhonovclass}, to then modify it by letting the regularization parameter decrease at each iteration. 
In this way, the obtained iteration no longer converges to a solution of Problem~\eqref{tikhonovclass},  but rather directly to the minimal norm separating solution~\eqref{minorm}.  To derive Algorithms  \ref{algodualprojGD} and \ref{algodualinertialGD} this basic idea is actually applied to the dual formulation of Problem~\eqref{tikhonovclass}.  

%
%

To describe the above ideas more precisely, we begin rewriting  Problem~\eqref{tikhonovclass} for the special case of hinge loss $\ell(\cdot)=\hi{1-\cdot}$. In order to ease the notation, we  define the matrix $Z=\text{diag}(Y)X$ and consequently $Z=(z_{i})_{i\leq n}$, with $z_{i}=(y_ix_{i})_{i\leq n}\in \R^{n\times d}$, we set  $\mathcal{L} : \R^{n}\to\R$, with $\mathcal{L}(u)=\sum_{i=1}^{n}\ell(u_i)$, and we get that solving Problem~\eqref{tikhonovclass} is equivalent to solve
\begin{equation}\label{generalReghinge}
	\underset{w\in\R^d}{\min} \frac{1}{\lambda}\mathcal{L}(Zw) + \frac{\norme{w}^2}{2}.
\end{equation}
The objective function in~\eqref{generalReghinge} is the sum of a smooth term, the squared norm, and a convex nonsmooth one $\frac{1}{\lambda}\mathcal{L}\circ Z$. 
Problems such as \eqref{generalReghinge} are called structured composite convex minimization problems, and can be often solved efficiently by  proximal-gradient methods \cite{combettes2005signal}, a class of first order methods splitting the contribution of the smooth and the nonsmooth part. At every iteration, the smooth part is activated through a gradient step, while for the nondifferentiable one the computation of the proximity operator is required. 
In order to implement a proximal gradient algorithm to solve problem ~\eqref{generalReghinge} we would need the computation of the proximity operator of $\mathcal{L}\circ Z$, which is not available in closed form and may be computationally expensive.  Therefore,  we consider the dual problem (see for example \cite[Definition $15.10$]{bauschke2011convex}) associated to \eqref{generalReghinge} (which is equivalent to \eqref{min_norm}) which is given by 
\begin{equation}\label{dualminproblem}
	\underset{u\in \R^{n}}{\min} \frac{\norme{Z^{\top}u}^2}{2} +\frac{1}{\lambda}\mathcal{L}^{\ast}(\la u), 
\end{equation}
where $\mathcal{L}^{\ast}$ denotes the Fenchel conjugate of $\mathcal{L}$ defined in Section~\ref{sec:notation}.  Its computation is simplified by the fact that 
$\mathcal{L}$ is a sum of separable functions, and can be written as  $\mathcal{L}^{\ast}(u)=\sum_{i=1}^{n}\ell^{\ast}(u^{i})= \sum_{i=1}^{n}u^{i} + \iota_{[-1,0]}(u^{i})$, 
which implies 
\begin{equation}
\label{eq:ellstar}
\frac{1}{\lambda}\mathcal{L}^{\ast}(\la u)=\sum_{i=1}^{n}u^{i} + \iota_{[-1/\lambda,0]}(u^{i}).
\end{equation}
 It is also important to write the dual problem associated to min-norm problem \ref{minorm} which is given by 
\begin{equation}\label{dualconstrainedmin}
	\min_{u\in \R^n} D_{\infty}(u)=\frac{\norme{Z^{\top}u}^{2}}{2}+\sum_{i=1}^{n}u^{i}+\iota_{(-\infty,0]^{n}}(u).
\end{equation}
Indeed, it is possible to show  (see Lemma \ref{remarkdecreasingDt}) that,  for $\lambda\to 0$,  the dual regularized Problem \eqref{dualminproblem} converges point-wise to Problem~\eqref{dualconstrainedmin}. To pass from convergence properties in the dual space to  the primal space,  recall that, if strong duality holds (see e.g. \cite[Section $15.2$]{bauschke2011convex}), the value of problem~\eqref{dualminproblem} is the same as the value of \eqref{generalReghinge} and,  for every $\lambda>0$, one can recover a solution $w_{\lambda}$ of the primal problem \eqref{generalReghinge}, from a solution $u_{\lambda}$ of the dual problem \eqref{dualminproblem} via the formula \cite[Section $19$]{bauschke2011convex}, 
\begin{equation}\label{dualtoprimalformula}
	w_{\lambda}=-Z^{\top}u_{\lambda}. 
\end{equation}
Problem~\ref{dualminproblem} is another composite convex optimization problem, where  \(\frac{\norme{Z^{\top}\cdot}}{2}\), has a  $\norme{XX^{\top}}_{\text{op}}$-Lipschitz gradient,  and  $\mathcal{L}^{\ast}$ is the nonsmooth part, of which  the proximity operator can be easily computed, as we show below. We can then implement a diagonal proximal gradient iteration on the dual function defined in \eqref{dualminproblem} as follows.
For a given a starting point $u_{0}$,  some step-size $\gamma>0$ and a decreasing sequence $\lambda_{t}$, the (diagonal) proximal gradient iteration corresponding to Problem \eqref{dualminproblem}  is given by, 
\begin{equation}\label{proxiteration}
	u_{t+1}=\prox_{\frac{\gamma}{\lambda_{t}}\mathcal{L}^{\ast}(\lambda_{t}\cdot)}\big(u_{t}-\gamma ZZ^{\top}u_{t}\big).
\end{equation}

%

We add several comments.  First, if $\la_t$ is taken to be constant then the above iteration solves Problem~\eqref{dualminproblem}, if the stepsize $\gamma\leq \norme{XX^\top}_{\text{op}}^{-1}$. Following, the previous discussion, by letting $\la_t$ decrease at each iteration  we obtain a diagonal process solving Problem \eqref{dualconstrainedmin}. 
Second,
the  computation of the proximity operator is simplified by the fact that 
$\mathcal{L}^{\ast}(\lambda_t\cdot)$ is a sum of separable functions, as can be seen from \eqref{eq:ellstar}.
Indeed, this allows to compute the proximal operator component-wise
$$
\prox_{\mathcal{L}^{\ast}(\lambda_{t}\cdot)}(u)=\big(\prox_{\ell^{\ast}(\lambda_{t}\cdot)}(u^{i})\big)_{i\leq n},
$$ 
and  derive the following iteration 
\begin{equation}\label{projevaluationt+1} 
	u_{t+1}= \biggl(\prox_{\frac{\gamma}{\lambda_{t}} \ell^{\ast}(\lambda_t\cdot)}\big(\left(u_{t}-\gamma ZZ^{\top}u_{t}\right)_{i}\big) \biggr)_{i\leq n}.
\end{equation}

%
Note that  the proximal operator of $\ell^{\ast}$, can be computed in closed form. Indeed, for any $\gamma,\lambda\in\mathbb{R}$, $p\in \R$,
\begin{equation}\label{computationprox}
	\begin{aligned}
		\prox_{\frac{\gamma}{\lambda} \ell^{\ast}(\lambda\cdot)}(p)&=\argmin_{s\in \R} \big\{s+\iota_{[-1/\lambda,0]}(s)+\frac{\norme{s-p}^{2}}{2\gamma}\big\} \\
		&=\proj_{[-1/\lambda,0]}\bigl(p-\gamma\bigr)=\begin{cases}
			-1/\lambda & p<\gamma-1/\lambda\\
			p-\gamma & p\in[\gamma-1/\lambda,\gamma]\\
			0 & p>\gamma.
		\end{cases}
	\end{aligned}
\end{equation}

Finally,  putting together all the above observations we  derive Algorithm \ref{algodualprojGD}. 
Algorithm \ref{algodualinertialGD} is derived  by considering a classic variation  of the proximal gradient iteration in Algorithm \ref{algodualprojGD}, namely a so called inertial step.
The latter corresponds to replacing $u_{t}$ in \eqref{projevaluationt+1} with $q_{t}=u_{t}+\frac{t}{t+\alpha}(u_{t}-u_{t-1})$. For both Algorithm \ref{algodualprojGD} and Algorithm \ref{algodualinertialGD} the last step corresponds to the dual-to-primal update $w_{t+1}=-Z^{\top}u_{t+1}$.

\begin{rem}\label{remarkfeatures}
Notice that Algorithms \ref{algodualprojGD} and \ref{algodualinertialGD} can be obtained by replacing $Z$ with $\text{diag}(Y)X$. In fact, our analysis also extends to considering general linearly parametrized models of the form $f_{w}(x)=\text{sign}(\scal{w}{\phi(x)})$, where $\phi : \R^{d}\to \mathcal{H}$ denotes some feature mapping (possibly infinite dimensional) which may be specified explicitly, or via some kernel operator, i.e. $K(x_{i},x_{j})=\scal{\phi(x_{i})}{\phi(x_{j})}$\footnote{The use of this feature-mapping $\phi$ allows to consider non-linear classifiers for possibly non-linearly separable data (see for example \cite{cortes1995support,steinwart2008support}).}. If we set $\phi(X)=(\phi(x_i))_{i\leq n}$, the penalized hinge-loss problem associated to the minimal norm separating problem and its corresponding dual are:
	\begin{align}
				&\underset{w\in \mathcal{H}}{\min}\left\{ \frac{1}{\lambda}\sum_{i=1}^{n}\ell(\scal{\phi(x_{i})}{w})+\frac{\norme{w}^2}{2}\right\} \label{featuresprimal}\\
				&\underset{u \in \R^n}{\min}\left\{ \sum_{i=1}^{n}\frac{\ell^{\ast}(\lambda u^{i})}{\lambda}+\frac{\norme{\phi(X)^{\top}u}^2}{2}= \sum_{i=1}^{n}\frac{\ell^{\ast}(\lambda u^{i})}{\lambda} + \frac{\sum_{i,j=1}^{n}u^{i}K(x_{i},x_{j})u_{j}}{2}\right\} \label{featuresdual}
			\end{align}
Since Algorithms \ref{algodualprojGD} and \ref{algodualinertialGD} are designed from the dual problem \eqref{featuresdual}, the information needed for the dual update $u_{t+1}$ is only the kernel evaluation at each data-point and not the one of $\phi(X)$ (see e.g. \cite{steinwart2008support}). In this case Algorithms \ref{algodualprojGD} and \ref{algodualinertialGD} can be rewritten by replacing $ZZ^{\top}$, with the operator $K$. Finally while the primal update $w_{t}$ may not be computable, the predictor can be still computed as $f_{w_{t}}(x)=\scal{w_{t}}{\phi(x)}=\scal{K(x_{i},x)}{u_{t}}$ via the dual iterates and the kernel evaluation.

\end{rem}

We end this section with two remarks commenting on some related literature and then analyze the properties of Algorithm \ref{algodualprojGD} and Algorithm \ref{algodualinertialGD} in the next section. 

\begin{rem}[Implicit regularization via homotopic subgradient]
Another  implicit regularization approach for the hinge loss was recently studied  in \cite{molitor2020bias} and derived using  an homotopic subgradient method for the primal problem \eqref{generalReghinge}. Our approach, in contrast, is based on a diagonal process on the dual problem \eqref{dualminproblem} and, as discussed later, leads to faster convergence rates.
\end{rem}

\begin{rem}[Hard-SVM]
The dual formulation \eqref{dualminproblem} is the one used for solving linear SVM problems (see \cite{steinwart2008support}). Indeed tackling the max-margin problem via its dual formulation \eqref{dualminproblem} is  popular, due to its favorable structure and there is a very rich literature on methods to solve it (see e.g. interior point-methods (\cite{boyd2004convex,ferris2002interior}) or decomposition methods (see e.g. \cite{cortes1995support,steinwart2008support}, \cite{ferris2002interior,ZANGHIRATI2003535} and \cite{platt1999fast,articleKeerthi,inproceedingsdorronsoro}. Compared to this methods our diagonal approach enjoys good theoretical guarantees  while providing a direct link to regularization methods. 
\end{rem}

\section{Main results and convergence analysis}\label{sectionmain}

In this section we present and discuss the main results of this work,  deferring the associated proofs to the Appendix. Before stating the main theorems, we discuss a key property of the sequence of dual problems. In particular, there exist some positive constants $\mu$, $M$ and $R$, such that, for all $t\geq 0$, each of the dual functions $D_{t}(\cdot)=\frac{1}{\lambda_{t}}\mathcal{L}^{\ast}(\lambda_{t}\cdot) + \frac{1}{2}\norme{Z^{\top}\cdot}^{2}$ related to problem \eqref{dualminproblem} satisfies the $\mu$-\L ojasiewicz condition in $[D_{t}\leq \min D_{t} + M] \cap \mathbb{B}(\mathbf{0},R)$, i.e. for all $u\in [D_{t}\leq \min D_{t} + M] \cap \mathbb{B}(\mathbf{0},R)$ and $t\geq 0$, it holds
	\begin{equation}\label{PLinequality}
		D_{t}(u) -\min D_{t} \leq \frac{1}{2\mu}\text{dist}\left(\partial D_{t}(u),0\right)^2,
	\end{equation}
as we will show in Lemma~\ref{lemmaPL}.  The previous condition is  a relaxation of  strong convexity, and it is well-known to imply linear convergence of the standard Forward-Backward scheme (see e.g. \cite{Loja63,bolte2017error,garrigos2017convergence,Karimi2016} and references therein). Theorem~\eqref{basicteoGD} extends these classical results to the diagonal setting. The value of  $\mu$ is crucial, since it determines the constant appearing in the linear convergence bound in \eqref{eq: basicteoGD}. As can be seen from \eqref{PLinequality}, $\mu$ is independent from $t$, and an explicit expression for $\mu$ can be given by
	\begin{equation}\label{eq:mu}
	\mu =  \frac{1}{8\tau^{2}\left( \left(3\sqrt{\norme{X^{\top}u_{0}}^{2}-\norme{X^{\top}u_{\ast}}^2 +\scal{\mathbf{1}}{u_{0}-u_{\ast}}} + \sqrt{2}\norme{X}_{op}\left(\norme{u_{0}}+2\norme{u_{\ast}}\right)\right)^{2} +2\right) },
	\end{equation}
	where $\tau$ is the Hoffman constant (see e.g. \cite{hoffman1952approximate,guler2010foundations} for a  definition) of a system of linear inequalities and equalities describing the set of minimizers of the dual objective function $D_t$. The explicit computation of this constant is expensive, but an expression is available in closed form, and is given by (see e.g. \cite[Lemma $15$]{wang2014iteration}): 
\begin{equation}\label{hoffmansconstant}
	\tau=\underset{(u,v) \in\R^{2n}\times\R^{d+1} }{\sup}\left\{ \left\lVert \begin{split}
		u \\ v
	\end{split}\right\rVert ~ : ~ \begin{split} &\norme{A^{\top}u+E^{\top}v}=1,~u\geq 0. \text{ The rows of $A$, $E$}  \\
&\text{ corresponding to non-zero components}\\ &\text{ of $u$ and $v$ are linearly independent.}
\end{split}\right\}
\end{equation}	
where $A=\begin{bmatrix}
 \text{Id}_{n} \\ -\text{Id}_{n}
\end{bmatrix}\in \R^{(2n)\times n}$ and $E=\begin{bmatrix}
Z^{\top} \\ \mathbf{1}^{\top}
\end{bmatrix}\in \R^{(d+1)\times n}$.
	
\off{
\begin{equation}\label{hoffmansconstant}
    \tau=\underset{B\in \mathcal{B}}{\max}\frac{1}{\sigma_{min}(B)},
\end{equation}
where $\mathcal{B}$ consists of the set of all matrices which are formed by taking the linearly independent rows of the matrix $E=\begin{bmatrix}
    X^{\top} \\ \mathbf{1}^{\top} \\ \text{Id}_{n} \\ -\text{Id}_{n}
\end{bmatrix}\in \R^{(d+1+2n)\times n}$.
}

We are now ready to state the main results. In Theorem \ref{basicteoGD}  we state the convergence results for the sequence generated by Algorithm \ref{algodualprojGD} to the minimal norm separating solution $w_{\ast}$ (see \eqref{min_norm}), and in Theorem \ref{basicteoiGD} we state the convergence results for the inertial Algorithm \ref{basicteoiGD}.  

\begin{theorem}\label{basicteoGD}
	Let $w_{\ast}$ the solution of \eqref{min_norm}, and $\{u_{t}\}_{t\geq 0}$ and $\{w_{t}\}_{t\geq 0}$ the sequences generated by Algorithm \ref{algodualprojGD}. Then $\{w_{t}\}_{t\geq 0}$ converges to $w_{\ast}$.
	In addition, if $u_{\ast}$ is a solution of the associated dual problem \eqref{dualconstrainedmin} and $\lambda_{0}\leq \nor{u_{\ast}}^{-1} $, then for all $t\geq 1$, the following estimate holds true:
	\begin{equation}\label{eq: basicteoGD}
		\norme{w_{t}-w_{\ast}}\leq C\left(1-\frac{\gamma\mu}{1+\gamma\mu}\right)^{\frac{t}{2}}
	\end{equation} 
	where $C=\sqrt{\norme{X^{\top}u_{0}}^{2}-\norme{X^{\top}u_{\ast}}^2 +2\scal{\mathbf{1}}{u_{0}-u_{\ast}}}$ and $\mu$ is defined in \eqref{eq:mu}.
	In addition, there exists some $t^{\ast}>0$, such that for all $t\geq t^{\ast}$, the following rates hold true for the angle and the margin gap (respectively):
	\begin{equation}\label{anglegapGD} 
		1-\frac{\scal{w_{t}}{w_{\ast}}}{\norme{w_{t}}\norme{w_{\ast}}}\leq\frac{C^2}{\norme{w_{\ast}}^{2}}\left(1-\frac{\gamma\mu}{1+\gamma\mu}\right)^{t}
	\end{equation}
	\begin{equation}\label{margingapGD} 
		M\left(\frac{w_{\ast}}{\norme{w_{\ast}}}\right)-M\left(\frac{w_{t}}{\norme{w_{t}}}\right)\leq\frac{2C\norme{X}_{F}}{\norme{w_{\ast}}^{2}}\left(1-\frac{\gamma\mu}{1+\gamma\mu}\right)^{\frac{t}{2}}
	\end{equation}
\end{theorem}

\begin{theorem}\label{basicteoiGD}
	Let $w_{\ast}$ the solution of \eqref{min_norm} and  $\{u_{t}\}_{t\geq 0}$ and $\{w_{t}\}_{t\geq 0}$ the sequences generated by Algorithm \ref{algodualinertialGD} with $\alpha\geq 3$. Then $\{w_{t}\}_{t\geq 0}$ converges to $w_{\ast}$. In particular, if $u_{\ast}$ is a solution of the associated dual problem \eqref{dualconstrainedmin} and $\lambda_{0}\leq \nor{u_{\ast}}^{-1} $, then for all $t\geq 1$, the following estimate holds true:
	\begin{equation}\label{eq: basicteoiGD}
		\norme{w_{t}-w_{\ast}}\leq \frac{C}{t+\alpha-1} 
	\end{equation} 
	where \(~C=(\alpha-1)\Big(\norme{X^{\top}u_{0}}^{2}-\norme{X^{\top}u_{\ast}}^2 +2\scal{\mathbf{1}}{u_{0}-u_{\ast}}+\frac{\norme{u_{0}-u_{\ast}}^2}{\gamma}\Big)^{1/2} \).

In addition, there exists some $t^{\ast}>0$, such that for all $t\geq t^{\ast}$, the following rates hold true for the angle and the margin gap (respectively):
	\begin{equation}\label{anglegapiGD}  
		1-\frac{\scal{w_{t}}{w_{\ast}}}{\norme{w_{t}}\norme{w_{\ast}}}\leq \frac{C^{2}}{ \norme{w_{\ast}}^{2}t^2}
	\end{equation}
	\begin{equation}\label{margingapiGD} 
		M\Big(\frac{w_{\ast}}{\norme{w_{\ast}}}\Big)-M\Big(\frac{w_{t}}{\norme{w_{t}}}\Big)\leq \frac{2C\norme{Z}_{F}}{\norme{w_{\ast}}t}
	\end{equation}
\end{theorem}

We add several remarks discussing the results in Theorems \ref{basicteoGD} and \ref{basicteoiGD} before comparing them to some recent related works and deriving their proofs.

\begin{rem}
	In  Theorem~\ref{basicteoGD} we  derived  the linear convergence of the sequence $\{w_{t}\}_{t\geq 0}$ thanks to condition \eqref{PLinequality}, as discussed above. 
	Even if the inertial version should give better convergence than the basic iteration, Theorem~\ref{basicteoiGD}  provides only a sublinear rate of convergence. 
	We believe this is due to technical, rather than fundamental reasons.  The numerical results in Section \ref{sectionnum} suggest that inertial variants can indeed provide faster convergence, but
	proving linear  rates of convergence for inertial variants is a challenging question and is an active area of research in the optimization literature, see e.g. the discussion in  \cite{garrigos2017convergence,apidopoulos2020convergence}. 
\end{rem}

\begin{rem}[Error metrics]
	Theorems \ref{basicteoGD} and \ref{basicteoiGD} provide  rates of convergence for the distance of the iterates \(w_{t}\) to the minimal norm solution, as well as the angle gap and the margin gap of the  normalized iterates \(\frac{w_{t}}{\norme{w_{t}}}\) to the max-margin solution \(w_{+}=\frac{w_{\ast}}{\norme{w_{\ast}}}\) , for Algorithms \ref{algodualprojGD} and \ref{algodualinertialGD} (respectively). As mentioned in Section \ref{section3}, since  the original max-margin problem \eqref{max_sphere} is a direction problem \cite{cortes1995support,steinwart2008support}, the margin and the angle gap are relevant quantities to measure the performance of the proposed methods, see \cite{nacson2018convergence,soudry2018implicitGD,molitor2020bias,ji2019implicit,ji2020gradient,ji2021fast}.
\end{rem}

\begin{rem}[Parameter choice]
	Both in Theorems \ref{basicteoGD} and \ref{basicteoiGD}, the requirement \(\lambda_{0} \leq \norme{u_{\ast}}^{-1}\), where $u_{\ast}\in \argmin D_{\infty}$ allows to  deduce the bounds\eqref{eq: basicteoGD} and \eqref{eq: basicteoiGD}, for all $t\geq 0$. If this condition is not verified the estimates in Theorems \ref{basicteoGD} and \ref{basicteoiGD} still hold true asymptotically due to the decreasing property of $\lambda_{t}$. In addition, one can freely choose the decay rate-to-zero of $\lambda_{t}$. In Section~\ref{sectionnum} we numerically evaluate the impact of different choices of $(\lambda_t) $ on the performance of the method.
\end{rem}

\subsection{Comparison to other convergence results for implicit regularization in classification}

We next compare the convergence results of Theorems \ref{basicteoGD} and \ref{basicteoiGD}, with existing  results in the related literature.  We begin by noting that Theorems \ref{basicteoGD} and \ref{basicteoiGD} provide  improved rates compared  to those for classical perceptron variants \cite{novikoff1963convergence} which are of order $\grandO{{1}/{\sqrt{t}}}$, see for example \cite[Theorem $4$]{ramdas2016towards}.  Margin rates 
similar to those in \eqref{margingapGD} and \eqref{margingapiGD}  have been derived for other optimization procedures applied to different losses and regularizers. For the iterates generated by gradient descent applied to exponentialy-tailed losses (such as logistic or exponential loss) a margin rate of order $\grandO{{1}/{\log(t)}}$ is derived in \cite{nacson2018convergence,soudry2018implicitGD,ji2019implicit}).  For the iterates of the  same algorithm with adaptive step size variants, the margin rates are of the order $\grandO{{\log(t)}/{\sqrt{t}}}$ (\cite[Theorem $5$]{nacson2018convergence}) or $\grandO{{1}/{\sqrt{t}}}$  \cite{clarkson2012sublinear,ramdas2016towards}. In all these cases, the rates are worse than the ones we obtain in  \eqref{margingapGD} and \eqref{margingapiGD}. The rates for the margin in Theorem \ref{basicteoiGD} for Algorithm \ref{algodualinertialGD} match the ones in \cite{ji2020gradient} (see Theorem $7$). They are slightly worse than those found in \cite[Theorem $3.1$]{ji2021fast} which are of order $\grandO{{\ln t}/{t^{2}}}$, and are obtained considering a mirror-descent method on the smoothed margin for the exponential and logistic loss \cite{lyu2019gradient,ji2020gradient,ji2021fast}.  
Finally, we compare to the results in \cite{molitor2020bias} considering  a different   implicit regularization approach, based on the use of a  homotopic subgradient method to minimize  the primal penalized hinge loss. The rates given in Theorems \ref{basicteoGD} and \ref{basicteoiGD} are considerably better than the ones  in \cite{molitor2020bias}, which are approximately of order $\grandO{t^{-\frac{1}{6}}}$ (see \cite[Corollary $1$, Lemma $2$]{molitor2020bias}).
None of the existing results provides linear rates, as the ones we derive for Algorithm \ref{algodualprojGD}.

\subsection{Stability}\label{section stability}

In Theorems \ref{basicteoGD} {and \ref{basicteoiGD}} we established the regularization properties of Algorithms \ref{algodualprojGD} {and \ref{algodualinertialGD}} in the sense of convergence to the minimal norm separating solution \eqref{min_norm} corresponding to  the true labels $Y=(y_{i})_{i\leq n}$.
In practice, labels are typically corrupted by noise and regularization methods should provide stable solutions. In this section,  we study the stability of Algorithm \ref{algodualprojGD} introducing a suitable notion of label noise (analogous results could be derived also for Algorithm \ref{algodualinertialGD} and are let for future study). 

For classical inverse problems noise is measured with respect to some norm in the data space. 
In the context of classification, a possible noise model is to consider a fraction of labels to be wrong 
\cite{angluin1988learning,kearns1992toward}.  
Since the data are binary valued, a natural way to measure the discrepancy between correct $Y=(y_{i})_{i\leq n}$ outputs and mislabeled outputs $\tilde{Y}=(\tilde{y}_{i})_{i\leq n}$ is to assume 
there exists $0\le N\le n/2$ such that 
\be\label{cnoise}
d_H (Y, \tilde Y )= N
\ee
where $d_H$ is the Hamming distance defined as,
\be
d_H (Y, \tilde Y)= \sum_{i : y_i\neq \tilde{y_i}} 1 
\ee
The Hamming distance replaces the norm in the data space to quantify the label noise.
Here, $N$ can be seen as the noise level.  The constraint $N\le n/2$ is natural  since higher values would correspond to simply renaming the classes. The case $N=0$ corresponds to noiseless data, that is where no mislabelling is present.  Note that, Assumption~\ref{cnoise} implies that there is a set of indices $S_{N}\subset{\{1,\dots,n\}}$ with cardinality $N$,
such that 
 \begin{equation}
\tilde{y}_{i}y_{i}=-1 ~, \quad  \text{ for all } i\in S_{N}.
\end{equation}

%
%

In the above setting, if $\tilde{w}(N)$ is a solution obtained using  labels with noise level $N$,
 then the goal is to derive error estimates  with respect to the true solution $w_{\ast}$ in terms of  $N$.  The error estimates should decrease in $N$ so that the correct solution is recovered as the noise decreases. In the case of  the iterative regularization procedure defined by Algorithm \ref{algodualprojGD},  this requires specifying a suitable choice for the stopping time. The following theorem provides such a choice and the corresponding error estimate. 
%
%
%

\begin{theorem}[Stability]\label{basicteostability}
	Let $w_{\ast}$ and $u_{\ast}$ be the solutions of problems \eqref{min_norm} and \eqref{dualconstrainedmin} respectively and $K=\underset{1\leq i\leq n}{\max}\{\norme{x_{i}}\}$. 
	Let  $\tilde{Y}=(\tilde{y}_{i})_{1\leq i\leq n}$ be a vector of  noisy outputs satisfying Assumption~\eqref{cnoise}.
	Finally, let $\{\tilde{u}_{t}\}_{t\geq 0}$ and $\{\tilde{w}_{t}\}_{t\geq 0}$ the sequences generated by Algorithm \ref{algodualprojGD} applied to the data $X=(x_{i})_{1\leq i\leq n}$ and $\tilde{Y}=(\tilde{y}_{i})_{1\leq i\leq n}$, with $\lambda_{0}\leq \nor{u_{\ast}}^{-1} $. Let $\rho=\frac{\gamma\mu}{1+\gamma\mu}$, with $\mu$ as defined in \eqref{eq:mu}.  Then for all $t\geq 1$, the following estimate holds true:
	\begin{equation}\label{optimal stability before}
	\norme{\tilde{w}_{t}-w_{\ast}} \leq C_{1}\sqrt{2(n+1-2N)N}t +  C_{2}\sqrt{N} +C(1-\rho)^{\frac{t}{2}}
	\end{equation}
In particular, for the  stopping time $t_{\ast}(N):= \max\left\{1 ,  \frac{2}{\ln\left(\frac{1}{1-\rho}\right)}\ln\left(\frac{C\ln\left(\frac{1}{1-\rho}\right)}{2C_{1}\sqrt{2(n+1-2N)N}}\right)\right\}$, \off{(here one could replace the constant $\frac{C\ln\left(\frac{1}{1-\rho}\right)}{2C_{1}}$) with some larger constant, in order to allow "larger" choices for the stopping time, which however would lead to a worst bound in terms of constant in $(5.14)$} the following bound holds,
\begin{equation}\label{optimal bound stability}
\begin{aligned}
\norme{\tilde{w}_{t_{\ast}(N)}-w_{\ast}} & \leq \frac{2C_{1}\sqrt{2(n+1-2N)N}}{\ln\left(\frac{1}{1-\rho}\right)}\ln\left(\frac{C\ln\left(\frac{1}{1-\rho}\right)}{C_{1}\sqrt{2(n+1-2N)N}}\right) +C_{2}\sqrt{N} +\frac{2C_{1}\sqrt{2(n+1-2N)N}}{\ln\left(\frac{1}{1-\rho}\right)} \\ 
&  = \grandO{\sqrt{N}}
\end{aligned}
\end{equation}	
where $C_{1}=2\sqrt{2}K^{2}\left(\norme{u_{0}-u_{\ast}}+\norme{u_{\ast}}\right)$, $C_{2}=2K\left(\norme{u_{0}-u_{\ast}}+\norme{u_{\ast}}\right)$, 	$C=\sqrt{\norme{X^{\top}u_{0}}^2-\norme{X^{\top}u_{\ast}}^2 +2\scal{\mathbf{1}}{u_{0}-u_{\ast}}}$.
\end{theorem}
The proof of the above result can be found in Section \ref{subsectionproofstability} of Appendix \ref{appendixa}. 
We add two remarks to discuss the above result.

\begin{rem}[Stopping time and stability]
The above result shows that the best stopping time choice arises from a trade-off between stability and convergence. More precisely,  the error estimate in \eqref{optimal stability before} is composed of  three terms.  The first two terms are related to the stability of the algorithm and  are increasing along the iterations due to the presence of the label noise. The last term is related to the convergence of the algorithm already analyzed in Theorem \ref{basicteoGD} in the absence of noise. The best stopping time is derived optimizing  the  bound \eqref{optimal stability before}. In this sense, this is an a priori choice. Deriving appropriate a posteriori choice is an interesting question left to a future study. Here,  we note that   the optimal stopping time is larger when the noise level is smaller, whereas the corresponding error decreases with respect to the noise. 
Another interesting question would be to derive corresponding lower bounds. 
\end{rem}

\begin{rem}\label{remark noise test}[Noise model]
The classification noise model considered above is simple and inspired by the  classic deterministic noise in inverse problems. It allows to take a first  step towards understanding the stability property of iterative regularization in classification.
We note that other, possibly more complex, noise models  are possible. 
For example, stochastic noise could be considered, possibly considering so called margin conditions \cite{massart2006risk}.  A more substantial development would be to consider random input data,  as often done in machine learning. This is likely to require results from empirical process theory \cite{boucheron} and possibly different statistical notions of stability already  used in machine learning \cite{bousquetstability}.

\end{rem}

\section{Numerical results}\label{sectionnum}

In this section,  we investigate numerically the properties of Algorithms \ref{algodualprojGD} and \ref{algodualinertialGD}. First, we analyze their  convergence and  stability on some synthetic datasets. Second, we study their performance on two benchmark datasets, and  compare them to  some recent related works.

\subsection{Synthetic data-set}

Following \cite{soudry2018implicitGD} and \cite{molitor2020bias}, we consider a solution vector $w_{\ast}=\left(\frac{1}{2},\frac{1}{2}\right)$ defining the maximal margin separator $f(x)=\scal{w_{\ast}}{x}$ and two pairs of support vectors $x_{1}=(\frac{1}{2},\frac{3}{2})$, $x_{2}=(\frac{3}{2},\frac{1}{2})$ labeled with $y_{1}=y_{2}=1$ and $x_{3}=(-\frac{1}{2},-\frac{3}{2})$, $x_{4}=(-\frac{1}{2},-\frac{3}{2})$ labeled with $y_3=y_4=-1$. We then generate  $80$ data-points and assign them to the  two classes, so that  the support vectors do not change, i.e. we  have a larger distance from $f(x)=\scal{w_{\ast}}{x}$ than the points $x_{1}$, $x_{2}$, $x_{3}$ and $x_{4}$, see Figure \ref{Figure1}. 

We test Algorithms \ref{algodualprojGD} and \ref{algodualinertialGD} for $T=1000$ iteration with regularization parameter \(\lambda_{t}=\frac{4}{t}\) for all $t\leq T$. In Figure \ref{Figure2} we illustrate the convergence results in terms of the margin and angle gap, and the error of the difference $w_{t}-w_{\ast}$ (i.e. $\norme{w_{t}-w_{\ast}}$), as found in Theorems \ref{basicteoGD} and \eqref{basicteoiGD}.

In this toy example we can notice that while the theoretical worst case bounds for Algorithm \ref{algodualprojGD} are better than the ones for \ref{algodualinertialGD} as expressed in Theorems \ref{basicteoGD} and \ref{basicteoiGD}, this is not necessarily reflected in Figure \ref{Figure1}. This is due to the pessimistic worst case bound found in Theorem \ref{basicteoiGD} for the inertial Algorithm \ref{algodualinertialGD}, rather than a numerical issue. Indeed this mismatch between theory and practice for inertial methods is commonly observed in similar settings and is an active area of research which is let for future study.  A second remark on this example concerns the influence of the over-relaxation parameter $\alpha$ for the convergence behavior of Algorithm \ref{algodualinertialGD}. As one can notice in Figure \ref{Figure1}, the choice of $\alpha$ can highly affect the performance of Algorithm \ref{algodualinertialGD} in relation with the stopping time.  This observation rises an interesting question about the tuning of the parameter $\alpha$ which is let for future study (see also discussion in \cite{apidopoulos2020convergence}).

\begin{figure}[H]
\begin{center}
 \includegraphics[trim=1.2cm 96mm 30mm 9.5cm, scale=0.36]{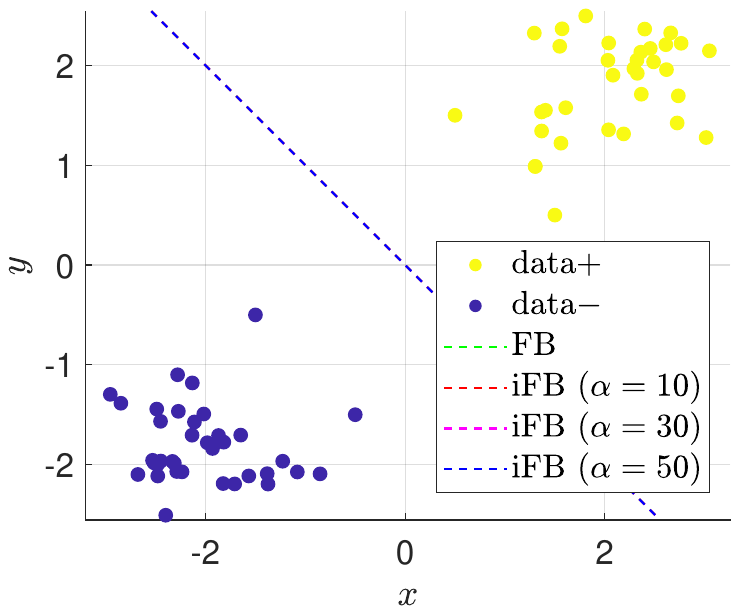}
\end{center}
\caption{Data-set consisting of $80$ labeled points with given support vector-points \(\pm(\frac{1}{2},\frac{3}{2})\) and \(\pm(\frac{3}{2},\frac{1}{2})\).
	In dashed lines the  (overlapping) max-margin separating hyperplanes formed by the last iterate of every scheme (Algorithms \ref{algodualprojGD} and \ref{algodualinertialGD} with $\alpha=10$, $30$ and $50$ respectively).}\label{Figure1}
\end{figure}

\begin{figure}
\begin{center}
\includegraphics[trim=6cm 94mm 49mm 9.5cm, scale=0.37]{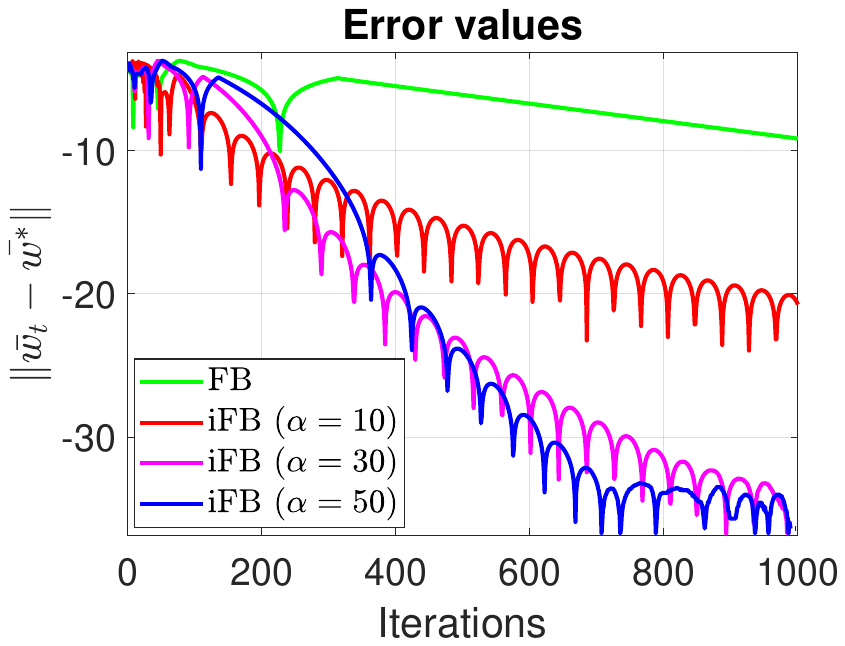}
\includegraphics[trim=2cm 94mm 45mm 9.5cm, scale=0.37]{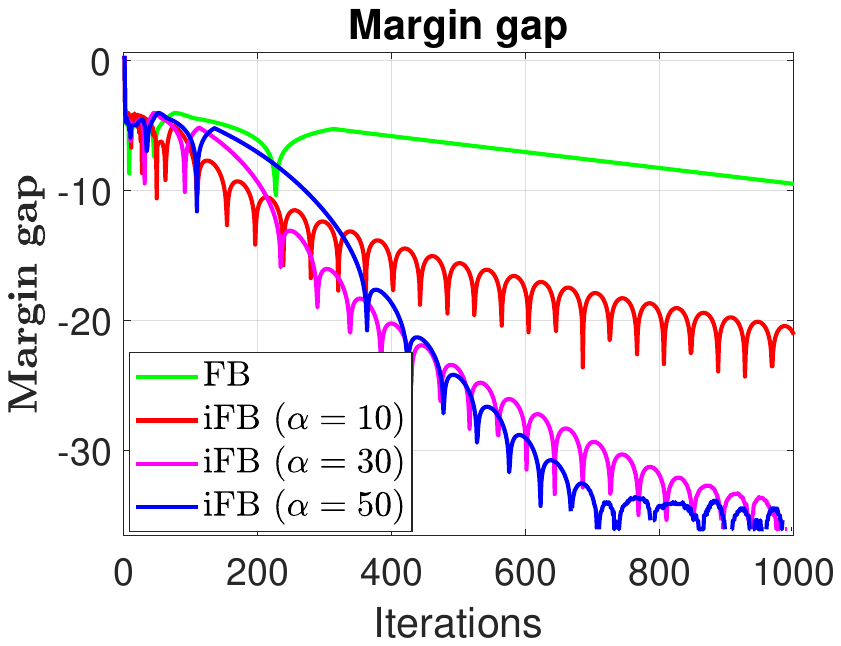}
\includegraphics[trim=2cm 94mm 60mm 9.5cm, scale=0.37
]{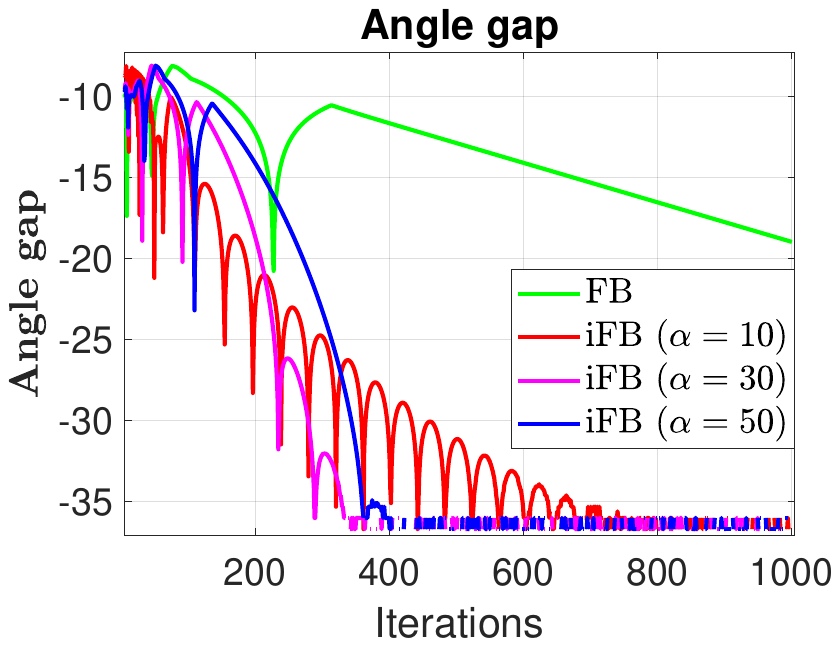}
\end{center}
\caption{Values of the normalized error gap \(\abs{{w_{t}}/{\norme{w_{t}}}-{w_{\ast}}/{\norme{w_{\ast}}}}\) (first figure), the normalized margin gap \({M(w_{\ast})}{/\norme{w_{\ast}}}-{M(w_{t})}/{\norme{w_{t}}}\)  (second figure) and the normalized angle gap \(1-\frac{\scal{w_{t}}{w_{\ast}}}{\norme{w_{t}}\norme{w_{\ast}}}\) (third figure), as a function of the iterations $t$. Here we illustrate the performance of Algorithms \ref{algodualprojGD}(green) and \ref{algodualinertialGD} with $3$ different choices for the parameter $\alpha$ in Algorithm \ref{algodualinertialGD}, $\alpha=10$ (red), $\alpha=30$ (magenta) and $\alpha=50$ (blue). As it appears in this example the choice of $\alpha$  decisively affects  the performance of algorithm \ref{algodualinertialGD}.}\label{Figure2}
\end{figure}

Second, we consider a data set of $1200$ points in $\R^{2}$, that consists of two classes distributed independently $\sim\mathcal{N}(\pm(\frac{1}{2},\frac{1}{2}),0.4)$ and we split them equally in training and test data. In this case, the training points are not linearly separable and algorithms \ref{algodualprojGD} and \ref{algodualinertialGD} are implemented with a Gaussian kernel with parameter $\sigma^{2}=0.15$, see Remark~\ref{remarkfeatures}. The total number of iterations (kernel evaluations) is set $T=2000$. In this experiment, we aim at illustrating  convergence but also the stability of the proposed methods in the presence of the {\em  noise} induced  by mislabeling an amount of the training data. In particular, we are interested in the effect of the parameter $\lambda_{t}$ on the convergence behavior, as also on the stability performance of the proposed methods \ref{algodualprojGD} and \ref{algodualinertialGD}.  In these experiments, the convergence is measured in terms of  the margin, while the stability via the test error on the test data. In Figure \ref{Figure3}, we plot the margin gap and the test error of Algorithms \ref{algodualprojGD} and \ref{algodualinertialGD} with parameter $\lambda_{t}={\lambda_{0}}/{t}$ with $\lambda_{0}\in \{0.01, 10, 100\}$, for three different levels of noise ($p=$ percentage of flipped training labels) starting from $p=0\%$ (no noise), $p=10\%$ (moderate noise) and $p=20\%$ (strong noise). In Figure \ref{Figure4}, we repeat the same experiment by choosing various orders of decay for $\lambda_{t}$, that is $\lambda_{t}\in \{{8}/{\log(t)},{8}/{\sqrt{t}},{8}/{t},{8}/{t^{2}},{8}/{2^{t}}\}$. Several comments can be made. First, we note that larger initialization $\lambda_{0}$ (e.g. $\lambda_{0}=100$ in Figure \ref{Figure3}) or slower decay rate  (e.g. $\lambda_{t}={8}/{\log(t)}$ or ${8}/{\sqrt{t}}$ in Figure \ref{Figure4}) lead to slower margin convergence, but better generalization properties especially in presence of errors (second and third rows in Figures \ref{Figure3} and \ref{Figure4}).  In addition, while Algorithm \ref{algodualprojGD} seems more robust with respect to the various changes of $\lambda_{t}$, the situation is different for the inertial variant \ref{algodualinertialGD}. Both in Figures \ref{Figure3} and \ref{Figure4}, for Algorithm \ref{algodualprojGD}, the behavior of the margin gap and the test error do not change radically unless the initialization is very large (e.g. $\lambda_{t}={100}/{t}$ in Figure \ref{Figure3}) or the decay rate is very slow (e.g. $\lambda_{t}={8}/{\log(t)}$ or ${8}/{\sqrt{t}}$ in Figure \ref{Figure4}) . On the other hand, Algorithm \ref{algodualinertialGD} seems more sensible to the different choices of the regularization parameter $\lambda_{t}$ both in terms of margin gap and test error.

 In the noiseless case (first row in Figures \ref{Figure3} and \ref{Figure4}) it seems that choosing $\lambda_{t}$ small enough (or decaying moderately fast to zero) can be a good policy. Note however that choosing too small initialization (or too fast decay to zero) does not offer any significant advantage with respect to more moderate choices of $\lambda_{t}$ (see in particular the blue, magenta and khaki lines in the first row of Figures \ref{Figure3} and \ref{Figure4} ). On the other hand, in presence of errors (second and third rows in Figures \ref{Figure3} and \ref{Figure4}), larger initialization (like $\lambda_{0}=10$ or $1$ in Figure \ref{Figure3}) and slower decay rate (as $\lambda_{t}={8}/{t}$ or ${8}/{\sqrt{t}}$ in Figure \ref{Figure4}) may offer a better trade-off between margin convergence and test error. 

\begin{figure}
	\begin{center}
		\includegraphics[trim=10.8cm 89mm 39mm 95mm, scale=0.29]{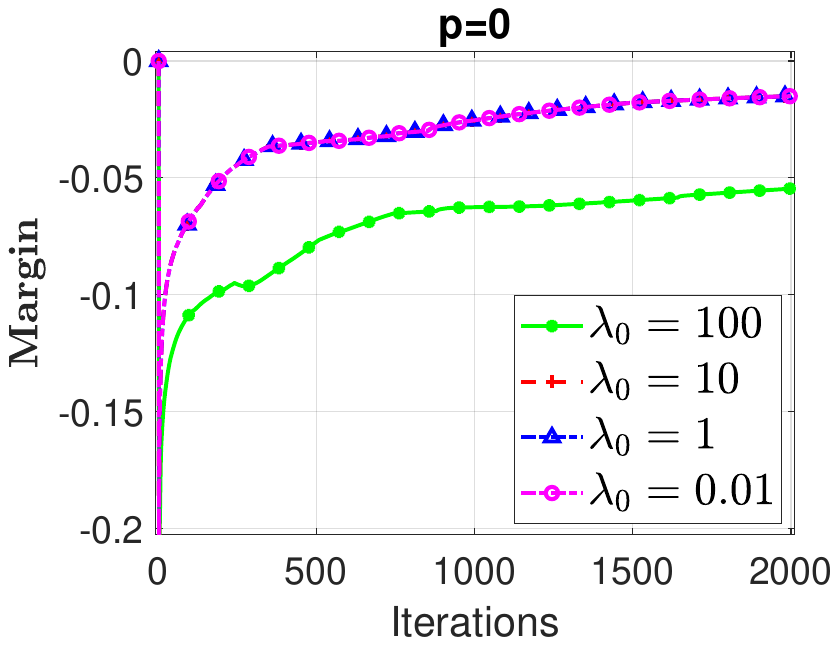}
		\includegraphics[trim=2.9cm 89mm 39mm 95mm, scale=0.29]{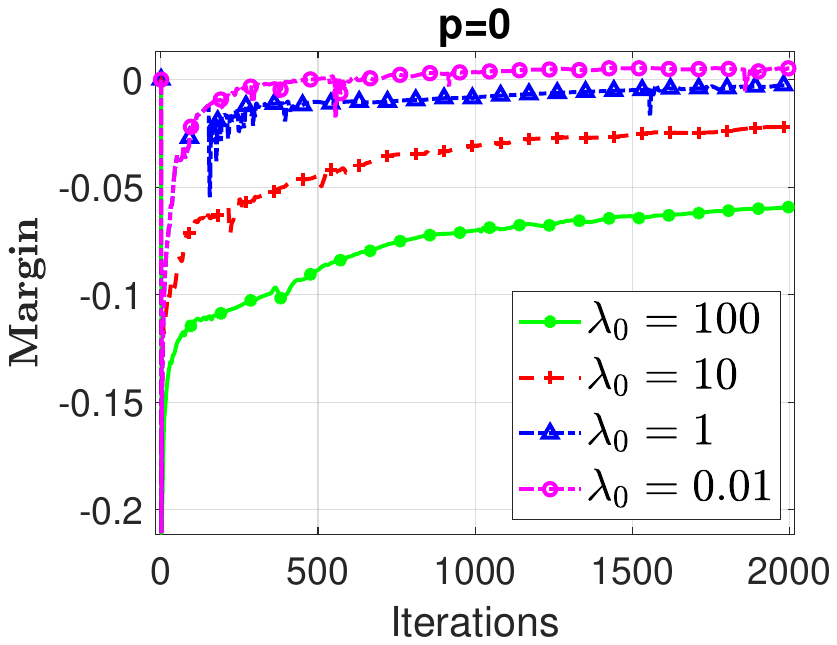} 
			\includegraphics[trim=2.9cm 89mm 39mm 95mm, scale=0.29]{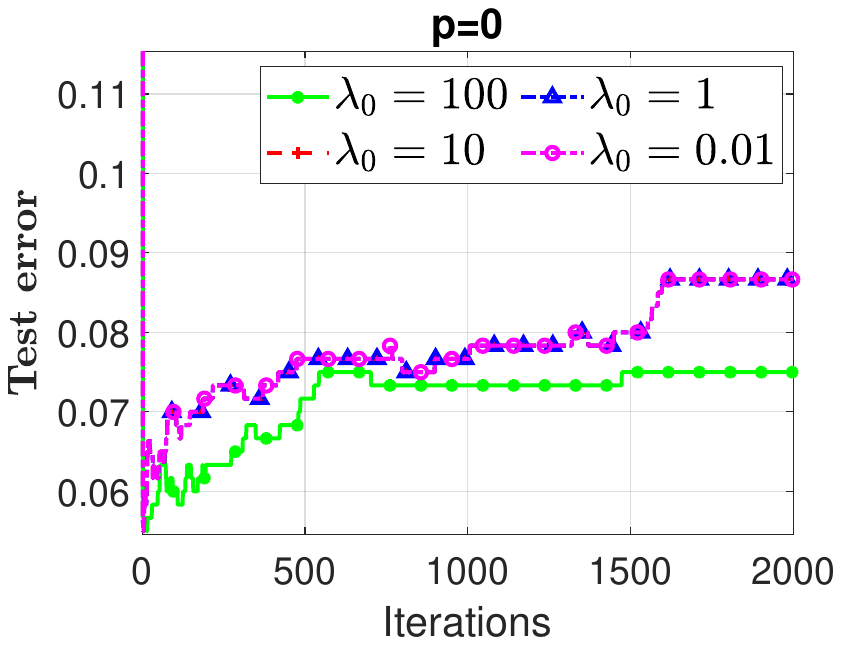}
		\includegraphics[trim=2.9cm 89mm 95mm 95mm, scale=0.29]{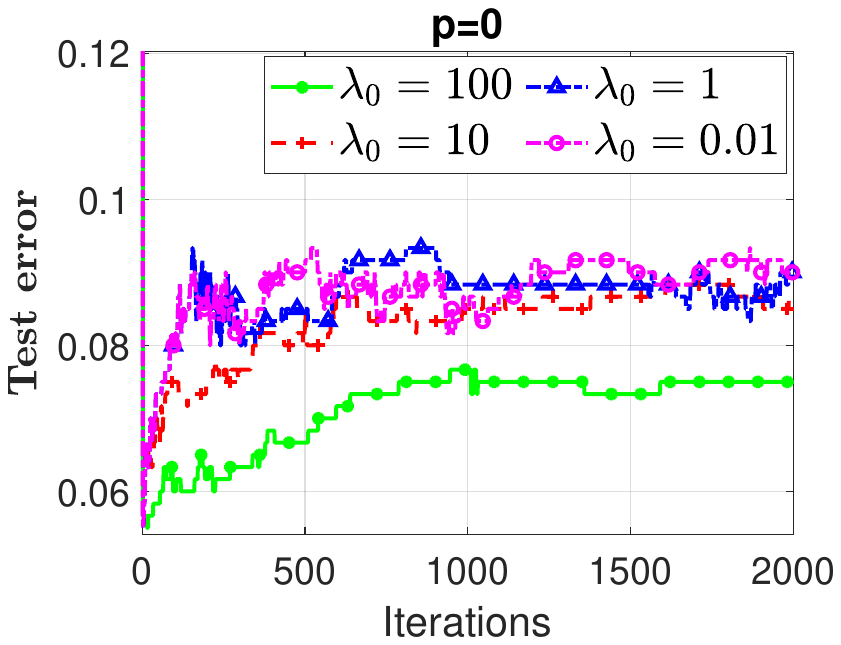} 
		\\
		\includegraphics[trim=10.8cm 89mm 39mm 94mm, scale=0.29]{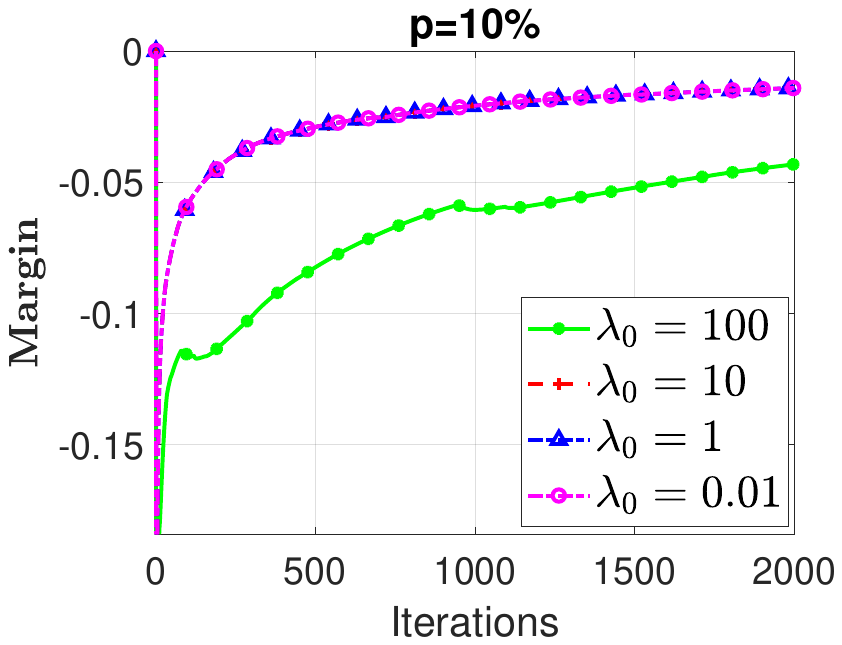}
	\includegraphics[trim=2.9cm 89mm 39mm 94mm, scale=0.29]{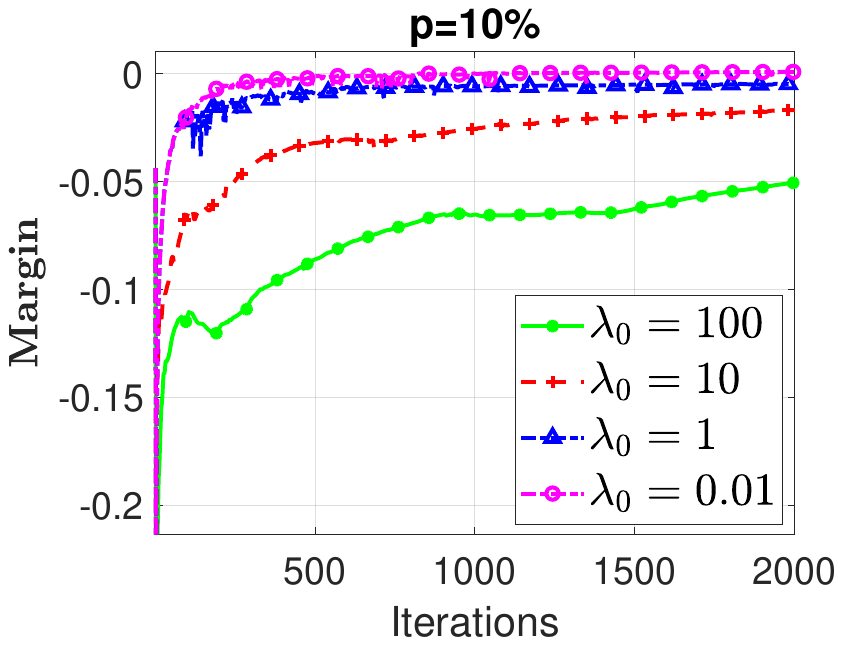} 
		\includegraphics[trim=2.9cm 89mm 39mm 94mm, scale=0.29]{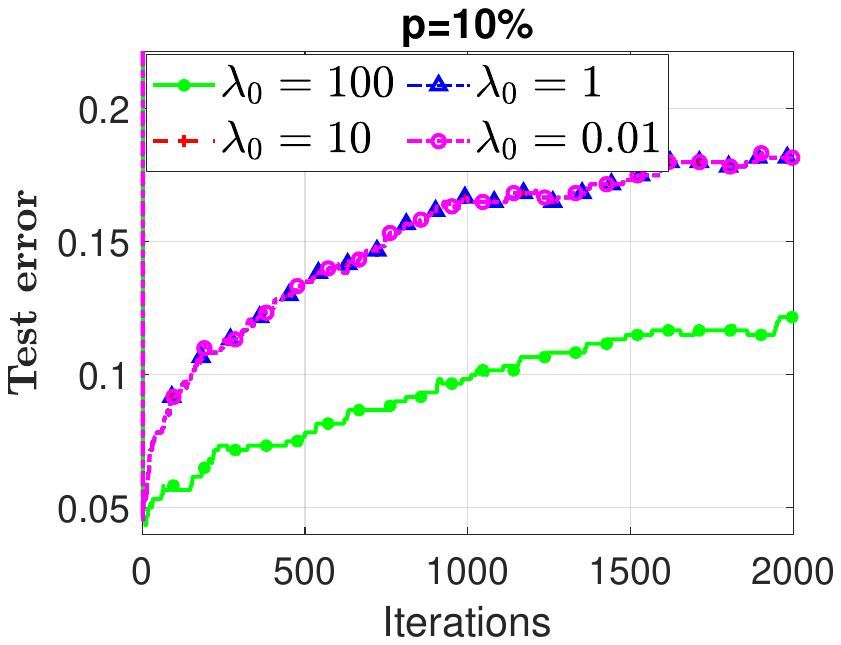}
	\includegraphics[trim=2.9cm 89mm 95mm 94mm, scale=0.29]{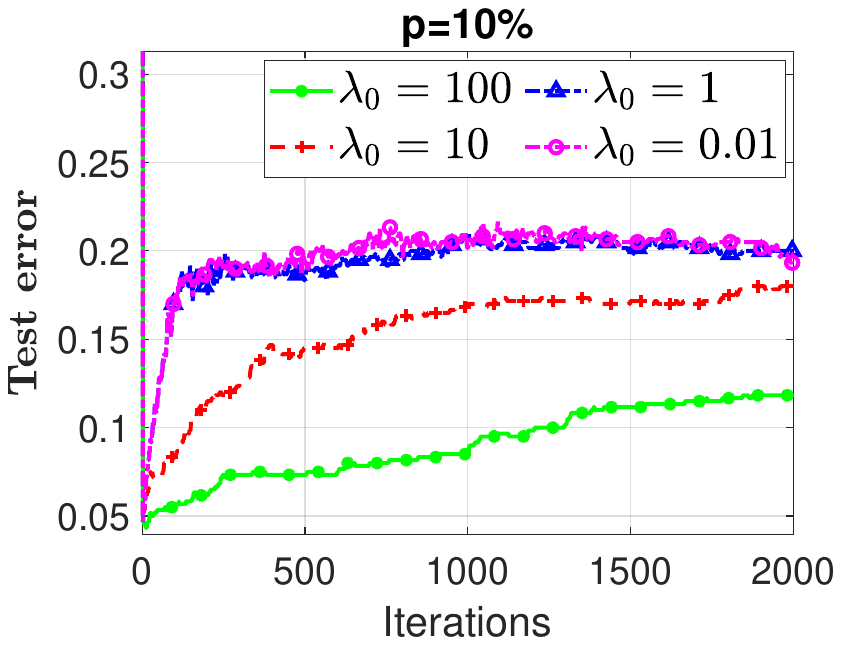} 
		\\
			\includegraphics[trim=10.8cm 99mm 39mm 94mm, scale=0.29]{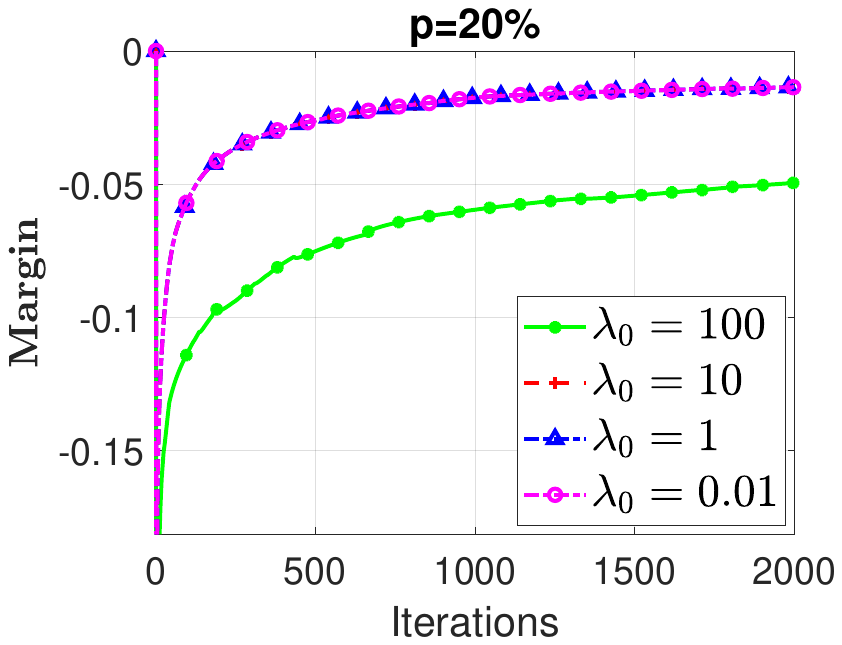}
		\includegraphics[trim=2.9cm 99mm 39mm 94mm, scale=0.29]{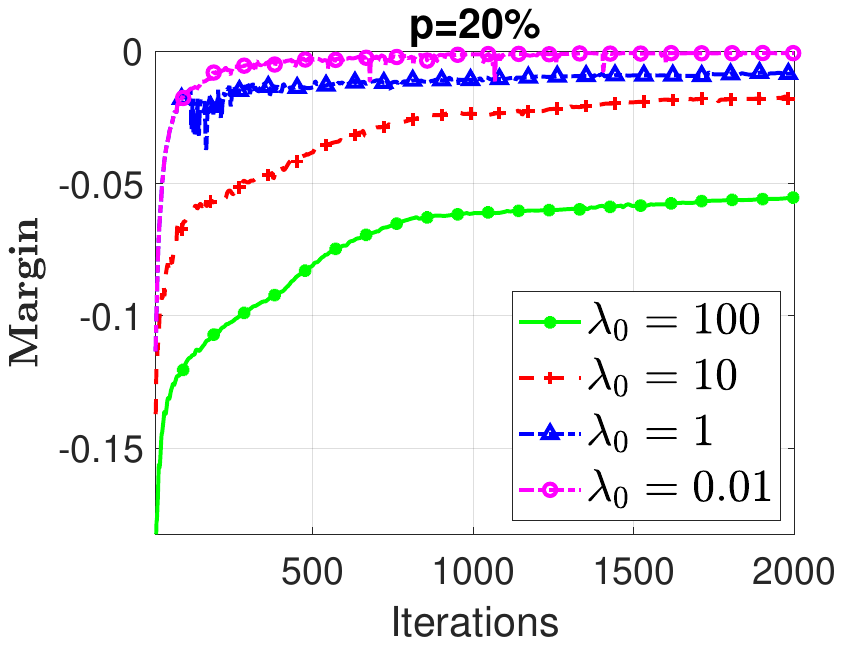} 
			\includegraphics[trim=2.9cm 99mm 39mm 94mm, scale=0.29]{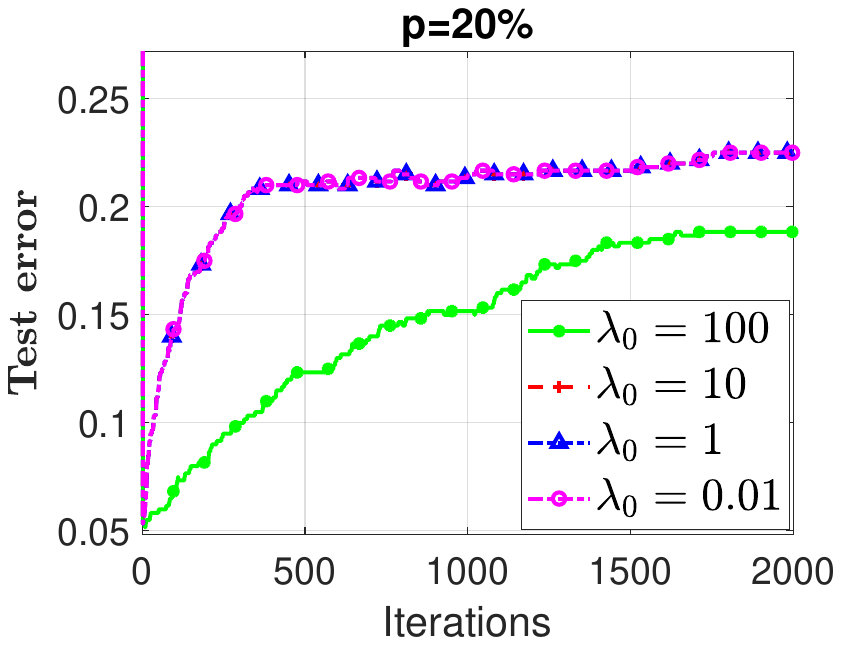}
		\includegraphics[trim=2.9cm 99mm 95mm 94mm, scale=0.29]{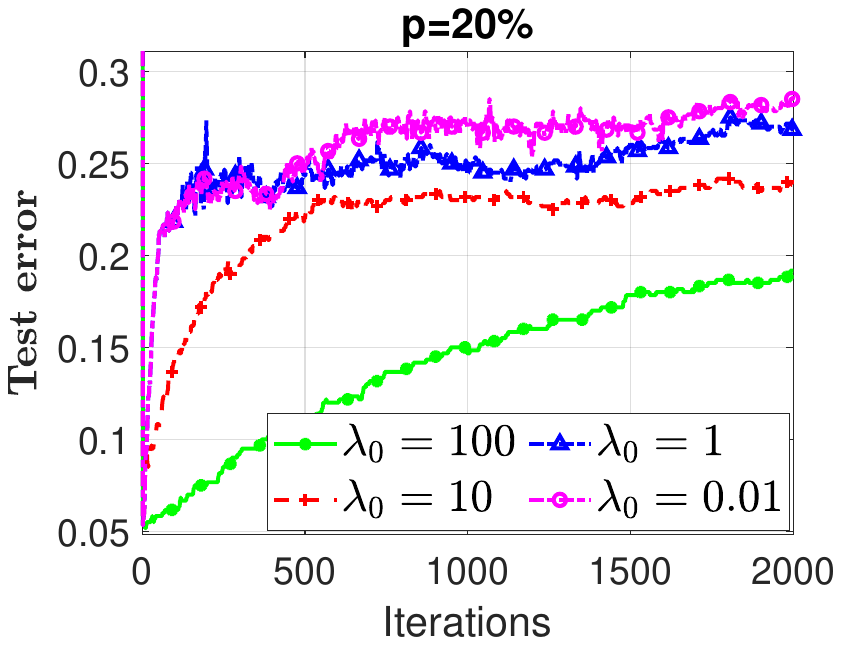} 
	\end{center}
	\caption{Margin and test error performance of Algorithms \ref{algodualprojGD} and \ref{algodualinertialGD} in noisy dataset with $\lambda_{t}={\lambda_{0}}/{t}$, for different initial values of $\lambda_{0}$ ($\lambda_{0}=100$ in green, $\lambda_{0}=10$ in red, $\lambda_{0}=1$ in blue and $\lambda_{0}=0.01$ in magenta). Each row corresponds to a different noise level ($\%$ of flipped labels) starting with $0\%$ (first row), $10\%$ (second row) and $20\%$ (third row). The first and third column illustrate the margin and test error of Algorithm \ref{algodualprojGD} respectively, while the second and the fourth column correspond to the margin (resp. test error) of Algorithm \ref{algodualinertialGD}. }\label{Figure3}
\end{figure}

\begin{figure}
\begin{center}
\includegraphics[trim=10.8cm 89mm 43mm 91mm, scale=0.29]{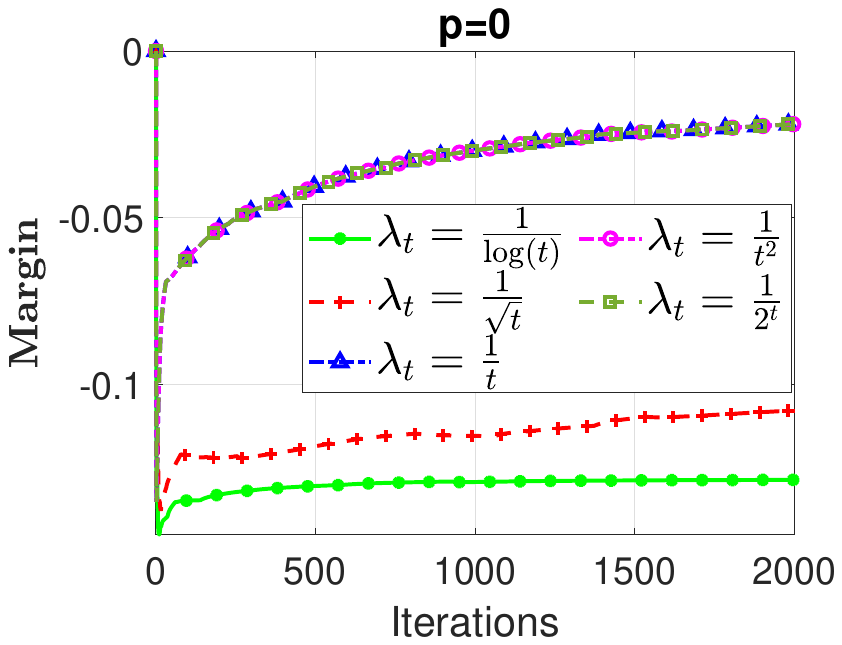}
\includegraphics[trim=2.8cm 89mm 43mm 91mm, scale=0.29]{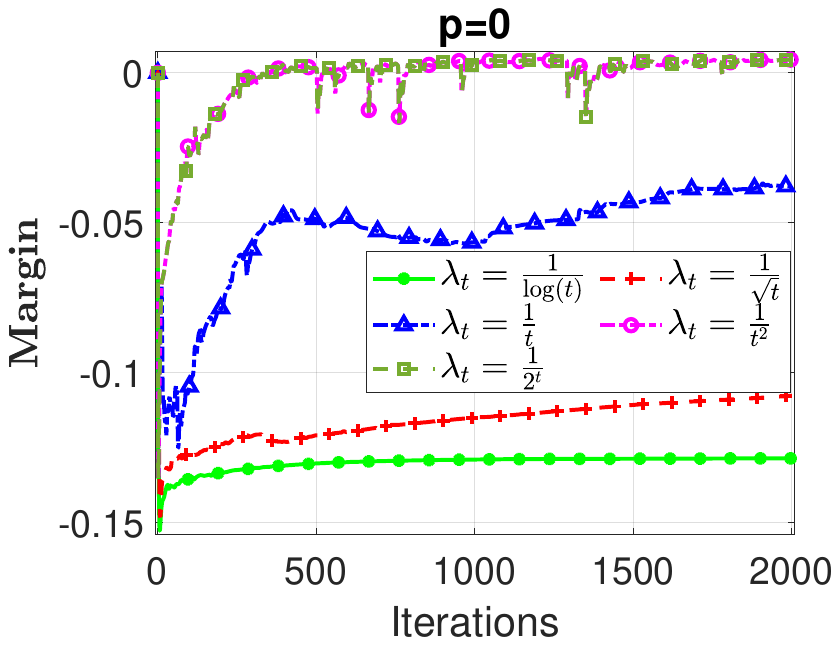} 
\includegraphics[trim=2.8cm 89mm 43mm 91mm, scale=0.29]{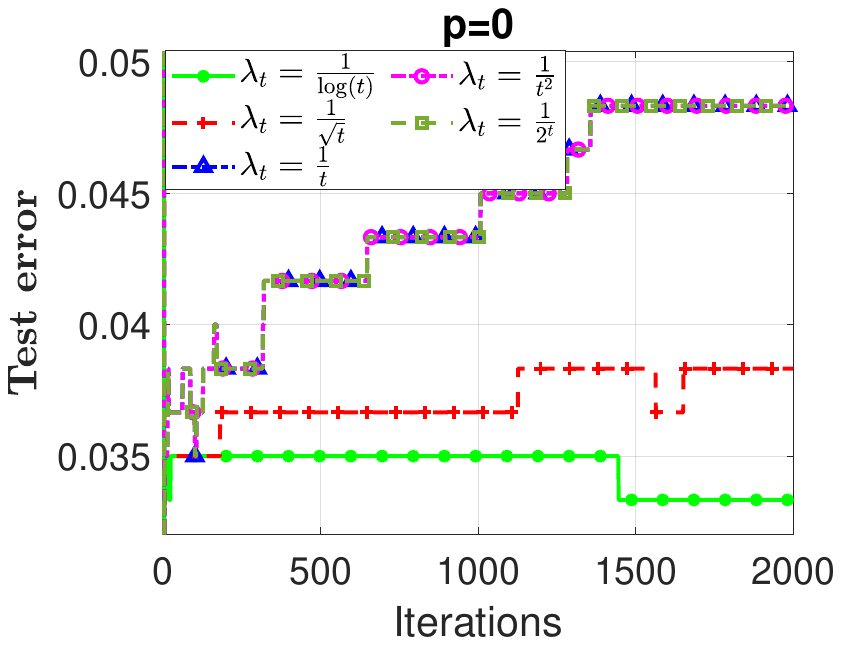}
\includegraphics[trim=2.8cm 89mm 95mm 91mm, scale=0.29]{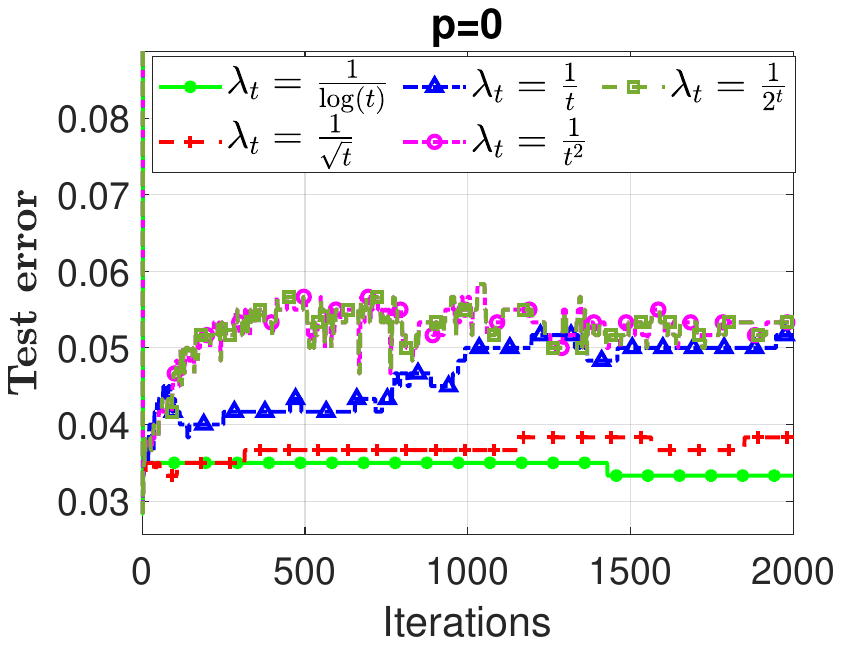} 
\\
\includegraphics[trim=10.8cm 89mm 43mm 95mm, scale=0.29]{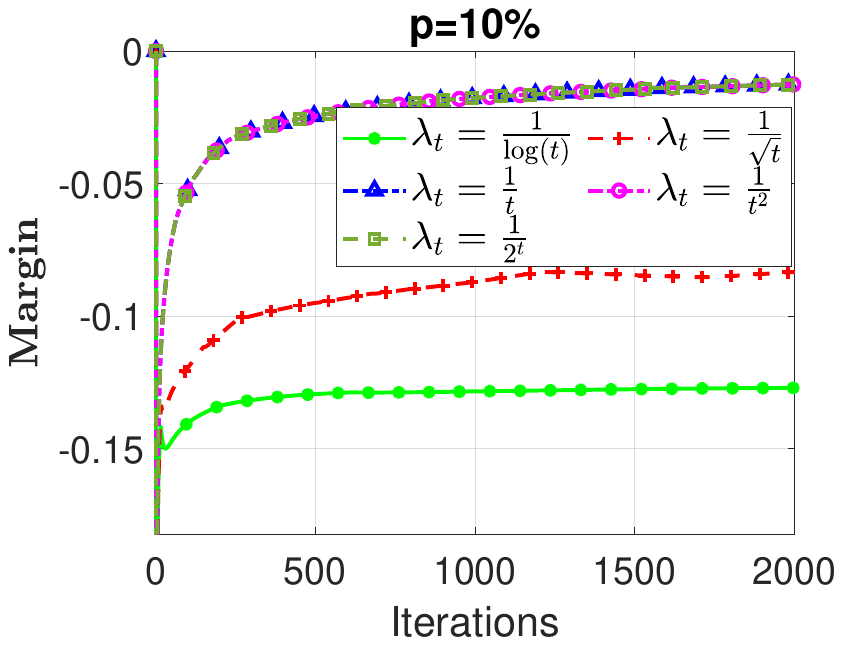}
\includegraphics[trim=2.8cm 89mm 43mm 95mm, scale=0.29]{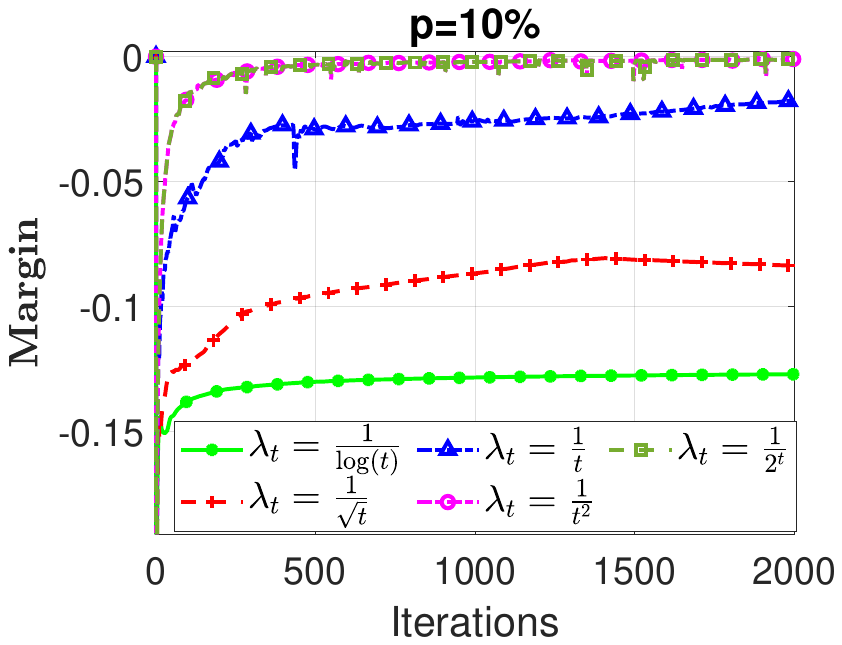} 
\includegraphics[trim=2.8cm 89mm 43mm 95mm, scale=0.29]{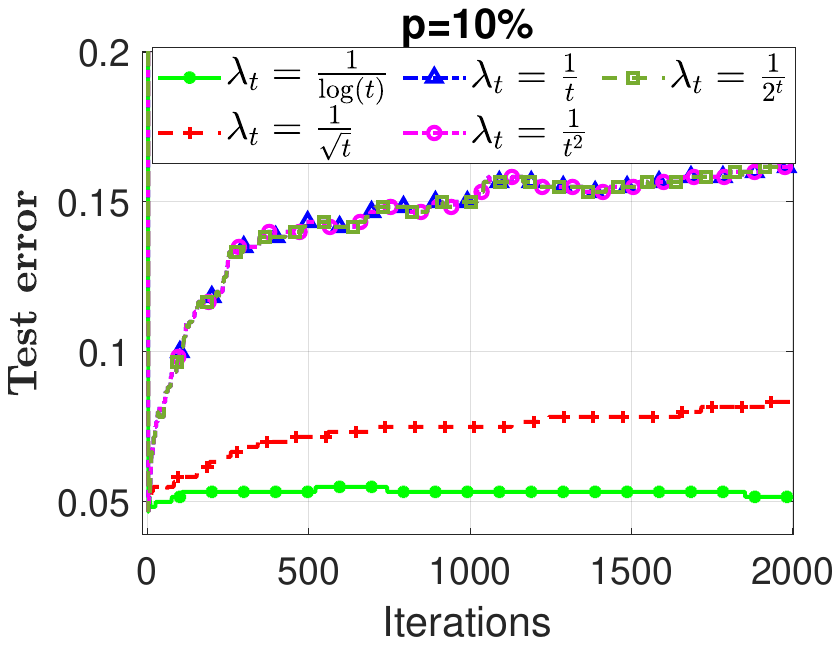}
\includegraphics[trim=2.8cm 89mm 95mm 95mm, scale=0.29]{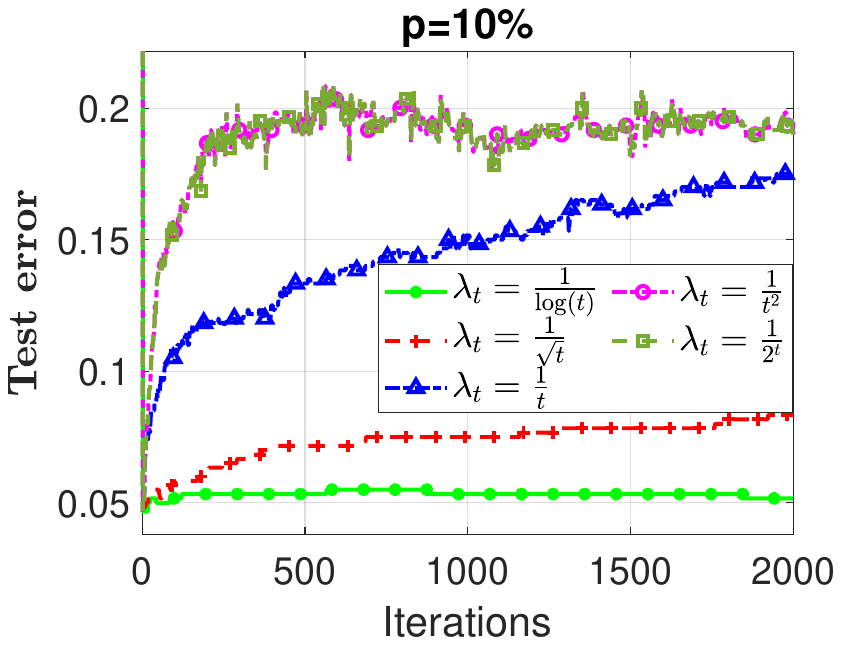} 
\\
\includegraphics[trim=10.8cm 99mm 43mm 95mm, scale=0.29]{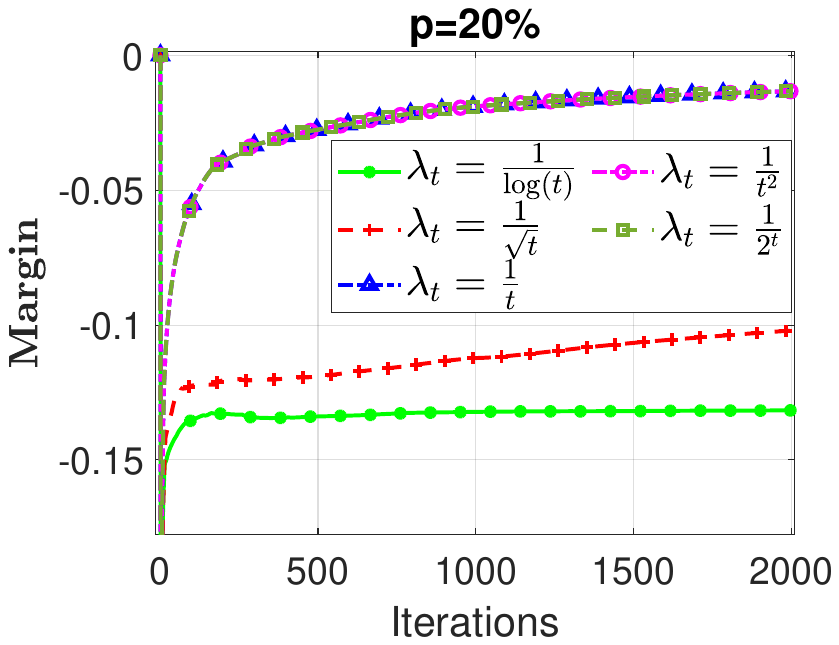}
\includegraphics[trim=2.8cm 99mm 43mm 95mm, scale=0.29]{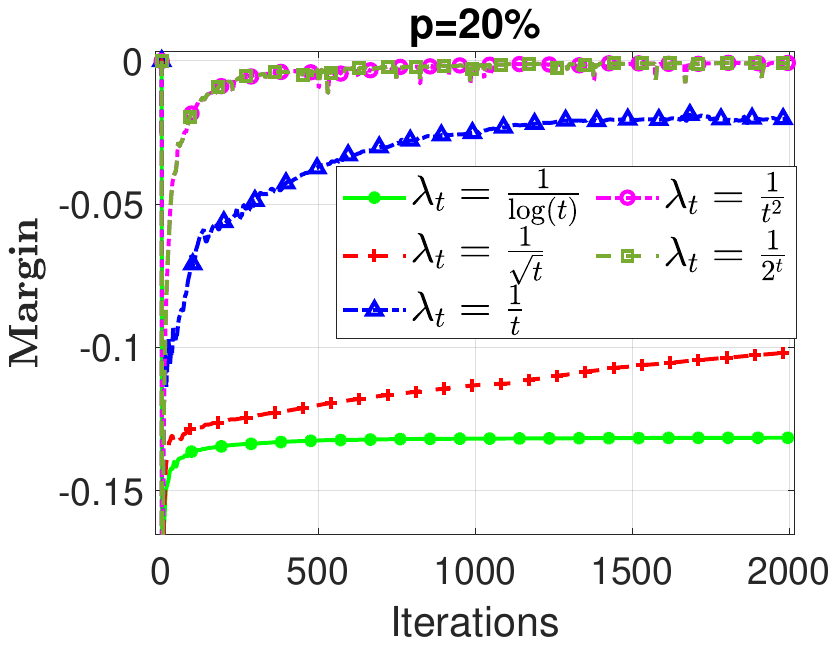} 
\includegraphics[trim=2.8cm 99mm 43mm 95mm, scale=0.29]{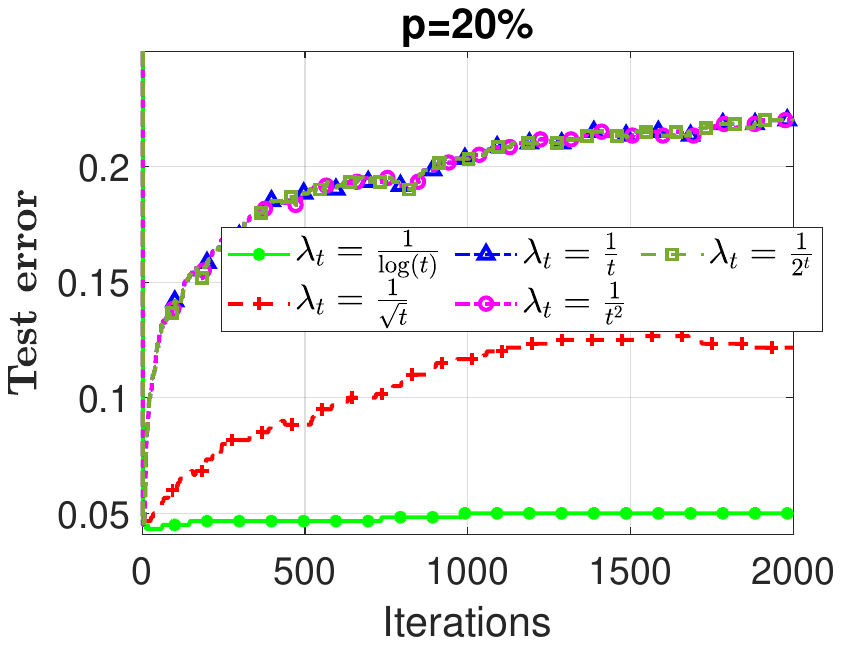}
\includegraphics[trim=2.8cm 99mm 95mm 95mm, scale=0.29]{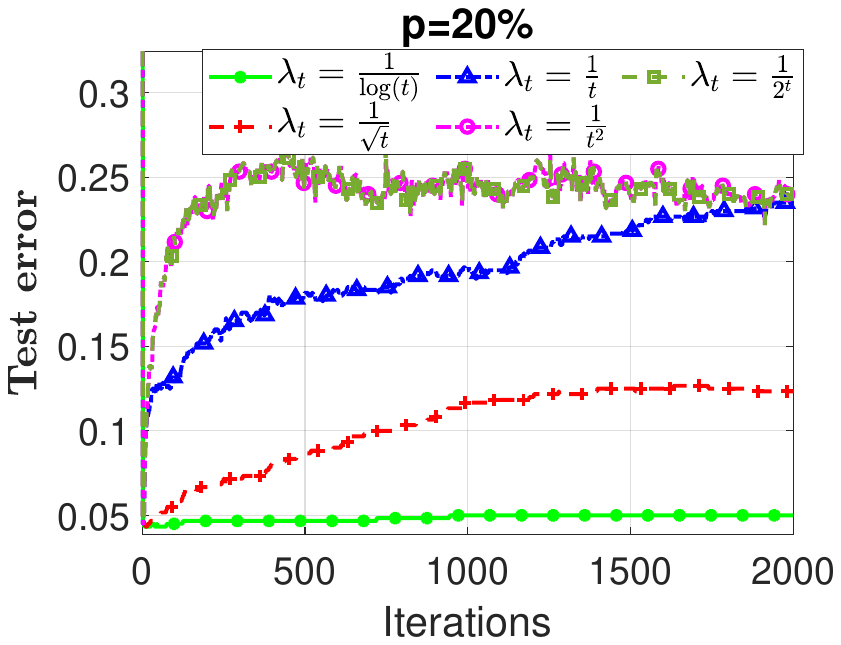} 
\end{center}
\caption{Margin and test error performance of Algorithms \ref{algodualprojGD} and \ref{algodualinertialGD} in noisy dataset for different decay rates of $\lambda_{t}$ ($\lambda_{t}=\frac{\lambda_{0}}{\log(t)}$ in green, $\lambda_{t}=\frac{\lambda_{0}}{\sqrt{t}}$  in red, $\lambda_{t}=\frac{\lambda_{0}}{t}$  in blue, $\lambda_{t}=\frac{\lambda_{0}}{t^2}$  in magenta and $\lambda_{t}=\frac{\lambda_{0}}{2^{t}}$ in khaki), where $\lambda_{0}=8$. Each row corresponds to a different noise level ($\%$ of flipped labels) starting with $0\%$ (first row), $10\%$ (second row) and $20\%$ (third row). The first and third column illustrate the margin and test error of Algorithm \ref{algodualprojGD} respectively, while the second and the fourth column correspond to the margin (resp. test error) of Algorithm \ref{algodualinertialGD}. }\label{Figure4}
\end{figure}

\subsection{Real data-set}

Finally, we test the proposed methods on the MNIST dataset (see \cite{Lecun1998}) and on the HTRU$2$ dataset\footnote{HTRU2 is a data set which describes a sample of pulsar candidates collected during the High Time Resolution Universe Survey (South)\cite{Keith2010}} (see \cite{lyon2016fifty,rando2022ada}). In particular we compare the performance of Algorithms \ref{algodualprojGD} and \ref{algodualinertialGD}, with some of the recent proposed methods for binary classification \cite{nacson2018convergence,soudry2018implicitGD}, and  \cite{ji2020gradient,ji2021fast} in terms of margin convergence and test error. The method in \cite{nacson2018convergence,soudry2018implicitGD} is a normalized gradient descent with variable stepsize on the exponential loss and the one in \cite{ji2020gradient,ji2021fast} is based on an accelerated mirror descent approach on the smoothed margin of the exponential loss. We want to stress out that this comparison is indicative for the theoretical convergence properties of the tested methods and is not meant to be exhaustive. The results are reported in Figure \ref{Figure5}.

\begin{figure}
	\begin{center}
		\includegraphics[trim=10cm 89mm 42mm 89mm, scale=0.298]{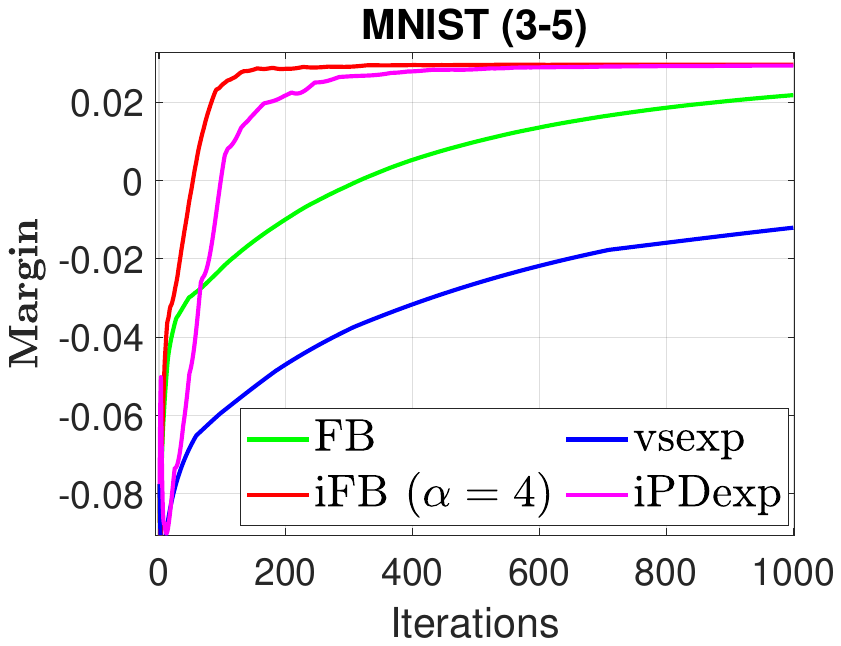}
		\includegraphics[trim=3cm 89mm 42mm 89mm, scale=0.298]{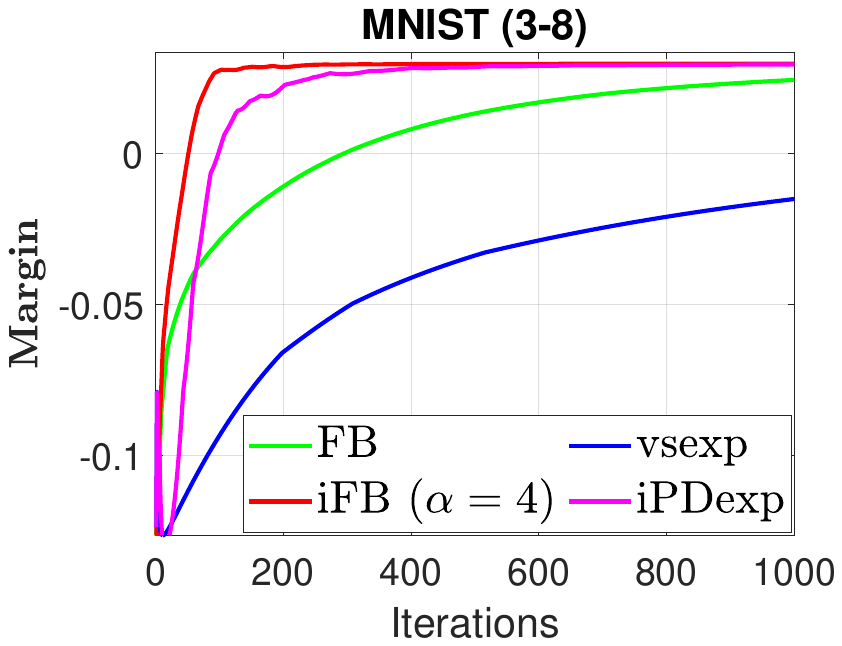}
		\includegraphics[trim=3cm 89mm 42mm 89mm, scale=0.298]{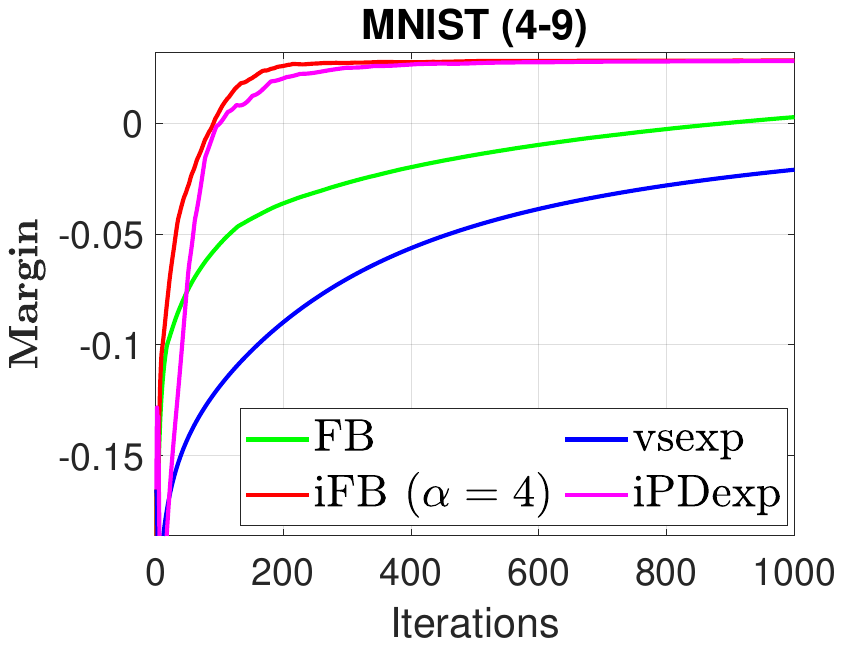} 
		\includegraphics[trim=3cm 89mm 85mm 89mm, scale=0.298]{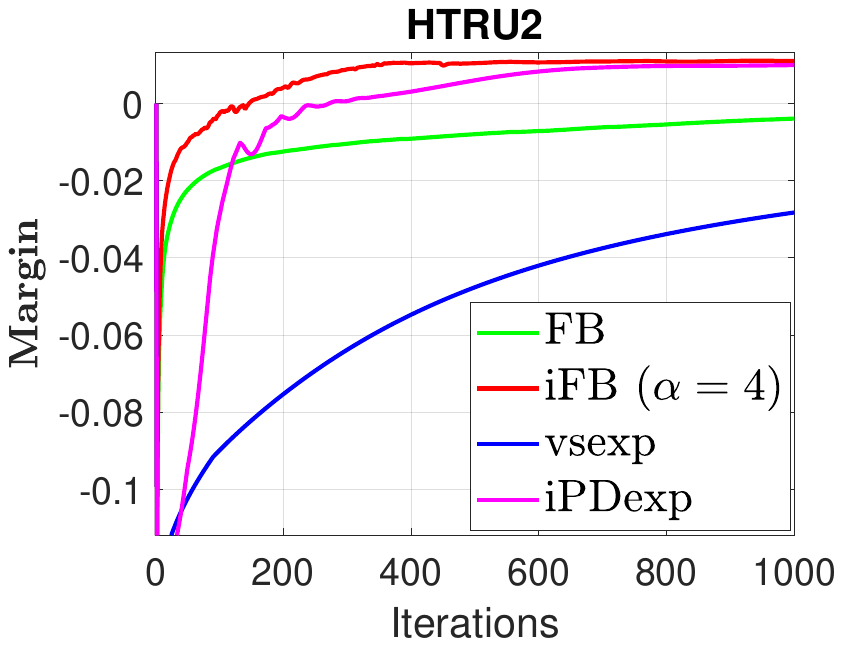} 
		\\
		\includegraphics[trim=10cm 99mm 42mm 95mm, scale=0.298]{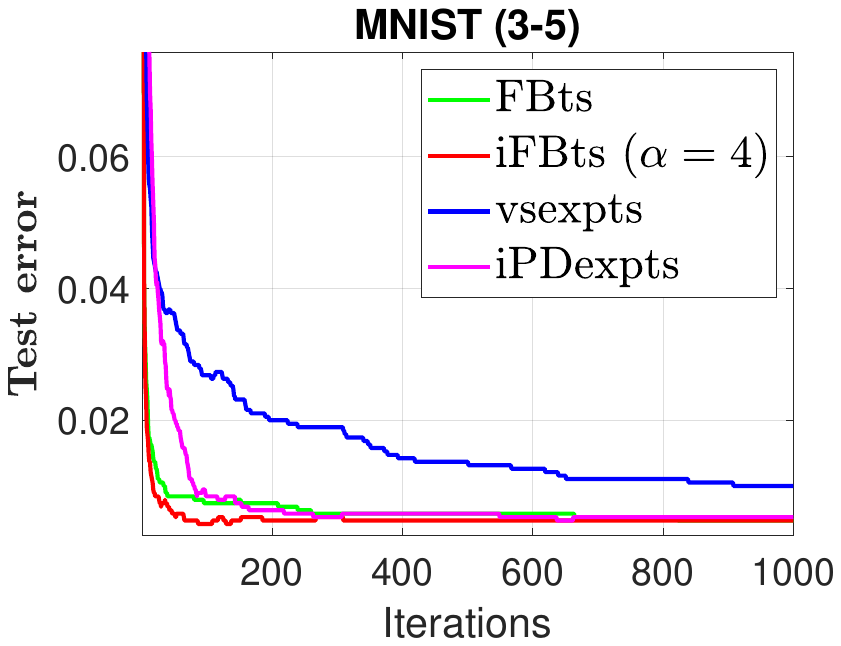}
		\includegraphics[trim=3cm 99mm 42mm 95mm, scale=0.298]{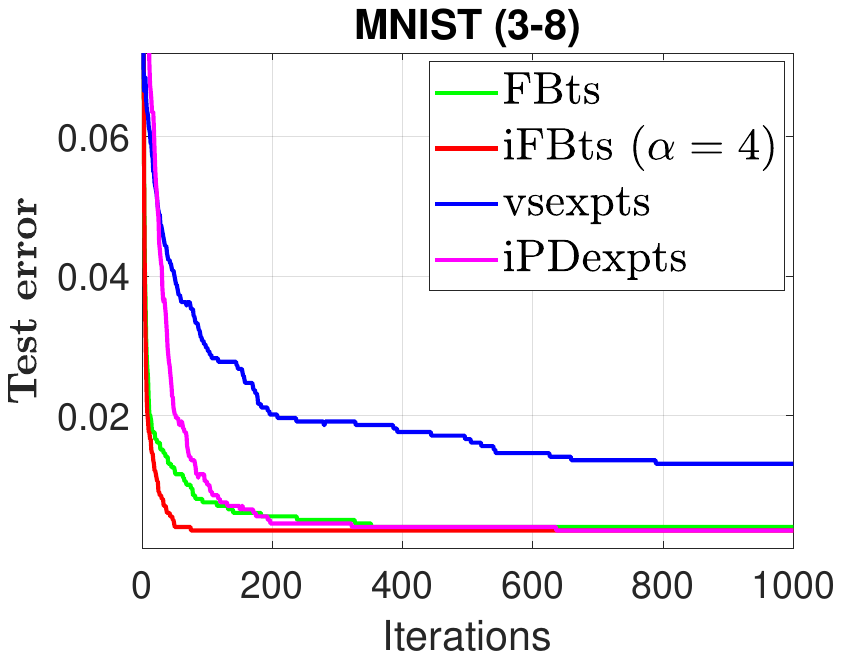}
		\includegraphics[trim=3cm 99mm 42mm 95mm, scale=0.298]{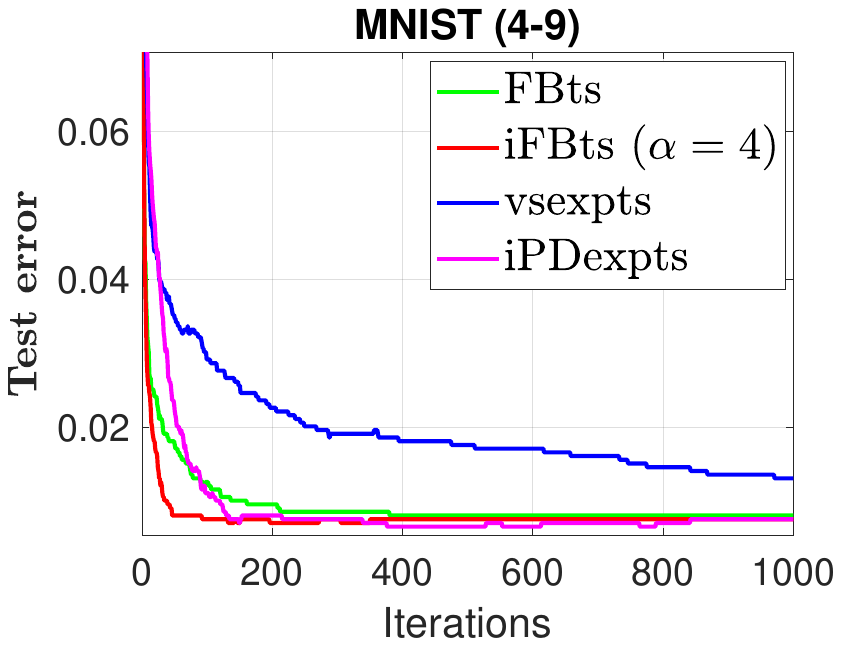} 
		\includegraphics[trim=3cm 99mm 85mm 95mm, scale=0.298]{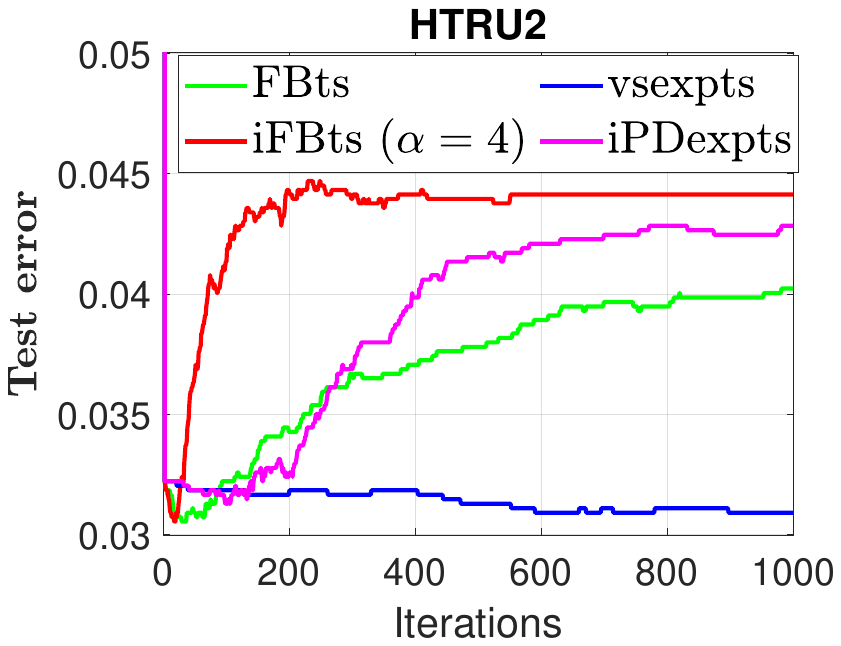} 
	\end{center}
	\caption{Margin and test error comparison on three pairs of MNIST digits (from first to third column) and the HTRU$2$ dataset (fourth column). The first row corresponds to margin gap, while the second one to the test error. In green and red are Algorithms \ref{algodualprojGD} and \ref{algodualinertialGD} respectively. In blue is the variable stepsize normalized gradient method proposed in \cite{soudry2018implicitGD,nacson2018convergence} and in magenta the inertial mirror descent approach in \cite{ji2021fast}. }\label{Figure5}
\end{figure}

All the data were standardized with zero mean and unit standard deviation and the binary labels of each dataset are set to be $-1$ and $1$. The MNIST dataset is restricted to a  two-digits comparison of $3$ vs $5$, $3$ vs $8$ and $4$ vs $9$. For the HTRU$2$ dataset the training part was chosen by randomly picking $70\%$ of the whole dataset, while the test part is composed of the remaining $30\%$ (see e.g. \cite{rando2022ada}). 
All methods are implemented with the Gaussian kernel with parameter $\sigma^2=3.4$ for MNIST digits and $\sigma^{2}=4$ for the HTRU$2$ dataset. Each update use only the kernel evaluation and thus all the schemes have the same computational complexity per iteration. The test error is measured by the standard zero-one loss.

As a general remark Algorithm \ref{algodualinertialGD} has the best performance in terms of margin convergence and test error in the MNIST datasets. For HTRU$2$, while Algorithm \ref{algodualinertialGD} still provides the fastest margin convergence, it performs slightly worse than the method in \cite{ji2021fast} in terms of test error,  whereas Algorithm \ref{algodualprojGD} and the method proposed in \cite{nacson2018convergence} seem to give better test error results.

\section{Conclusion and possible future work}

In this work,  we propose and study  iterative regularization for classification in machine learning.
Considering the hinge loss function, we derive an iterative regularization method  defined by  a dual diagonal optimization approach  and further consider its accelerated variant, see \cite{garrigos2018iterative,calatroni2019accelerated}. 
We provide convergence results, as well as convergence rates,  to the minimum norm separating solution. 
Moreover we derive stability results for a natural classification noise model. 

Several further research directions can be explored. For example it would be interesting to consider other form of regularization, extending  to classification the results for  linear inverse problems, see   \cite{molinari2021iterative} and references therein. Moreover, it would be interesting to consider 
different data model including stochastic noise and random input data as often done in statistical learning theory \cite{vapnik2013nature}. Finally it would be very interesting to consider nonlinear models following \cite{rangamani2021dynamics,pmlr-v97-nacson19a}.

\ackn{}{We acknowledge the financial support of the European
	Research Council (grant SLING 819789), the AFOSR ((European Office of Aerospace
	Research and Development)) project FA9550-18-1-7009 and FA8655-22-1-7034, the EU H2020-MSCA-RISE project NoMADS - DLV-777826, the H2020-MSCA-ITN Project Trade-OPT 2019;  L. R. acknowledges  the Center
	for Brains, Minds and Machines (CBMM), funded by NSF STC award CCF-1231216 and IIT. 
	S.V. and L.R. acnowledge the support of the  Ministry of Education, University and Research (PRIN 202244A7YL project "Gradient Flows and Non-Smooth Geometric Structures with Applications to Optimization and Machine Learning") is part of the Indam group "Gruppo Nazionale per
	l'Analisi Matematica, la Probabilit\`a e le loro applicazioni". This work has been partially supported by project MIS 5154714 of the National Recovery and Resilience Plan Greece 2.0 funded by the European Union under the NextGenerationEU Program.}

\appendix

\section{Appendix}\label{appendixa}

The Appendix is organized as follows: In paragraph \ref{subsection general lemmas} we recall some basic results concerning Tikhonov regularization and gradient descent for linear problems, as also some basic facts on the max-margin problem \ref{max_sphere}. Paragraph \ref{sectionconv} contains some auxiliary results that are needed for the proofs of main Theorems \ref{basicteoGD}, \ref{basicteoiGD} and \ref{basicteostability}.
In paragraphs \ref{subsectionFBdualhinge} and \ref{subsectioniFBdualhinge} we analyze  Algorithm \ref{algodualprojGD} and \ref{algodualinertialGD} (respectively) and provide the proofs of Theorems \ref{basicteoGD} and \ref{basicteoiGD}. Finally in paragraph \ref{subsectionproofstability} we provide the proof of Theorem \ref{basicteostability} regarding stability of Algorithm \ref{algodualprojGD}. 

\subsection{General Lemmas}\label{subsection general lemmas}
In the following two lemmas we establish the regularization properties of Tikhonov regularization and gradient descent to the minimal norm interpolating solution, for general losses.

\begin{lem}[Tikhonov]\label{tikhonovlemma}Let $V:\R^{d}\to \R_{+}$ be a continuous and convex function. For all $\lambda>0$, consider:
\be\label{losstik}
w_\la= \argmin_{w\in R^d} V(w)+\la\nor{w}^2.
\ee
\begin{enumerate}
	\item  If $\argmin V \neq \emptyset$, then 
	\[
	\lim_{\la\to 0} w_\la=w^{\dagger}:=\argmin_{}\{\norme{w} ~ : ~ w\in \argmin V\}. 
	\]
	\item If $\argmin V = \emptyset$, then  $\norme{w_{\lambda}}\underset{\lambda\to 0}{\to}+\infty$
\end{enumerate}

\end{lem}

\bp

For the first point, let $w^{\ast}\in \argmin V$ and $J_{\la}(w)= V(w)+\la\nor{w}^2$. By the definition of   $w_\la$ and $w^{\ast}$.
\begin{equation*}
J_{\la}(w_{\la}) \leq J_{\la}(w^{\ast}) \leq V(w_{\lambda})+\lambda\norme{w^{\ast}}^2
\end{equation*}
yielding \begin{equation}\label{lscnorme}
\norme{w_{\lambda}}\leq \norme{w^{\ast}}
\end{equation} which allows to deduce that the sequence $\{w_{\la}\}_\la{>0}$ is uniformly bounded.
This implies  that (up to a subsequence) $w_{\la}$ converges to an element $\bar{w}$.
By \textit{lower semi-continuity} of $V$, we have : $$V(\bar{w})\leq \liminf V(w_{\la}) \leq \liminf J_{\lambda}(w_{\la})\leq \liminf J_{\lambda}(w^{\dagger})=V(w^{\ast})$$
which shows that $\bar{w}\in \argmin V$. In addition from the (weak) lower semi-continuity of the norm and relation \eqref{lscnorme}, we have :
$$\nor{\bar{w}}\leq \liminf\nor{w_{\lambda}}\leq \nor{w^{\ast}}$$
and since the last inequality holds for an arbitrary $w^{\ast}\in \argmin V$, we deduce that $\bar{w}=w^{\dagger}$.

The second point can be proven by contrapositive. Let us assume the existence of some $M>0$, such that $\norme{w_{\lambda}}\leq M$. By following the arguments of the proof of the first point, we deduce the existence of a limit point $w_{\lambda}\to \bar{w}\in \argmin V$, which allows to conclude. 

\ep

\begin{lem}[Gradient descent]\label{gradientdescentlemma}
	Let $X\in \R^{n\times d}$ be a linear operator, with $d>n$ , $y\in \Im(X)$ and $\mathcal{L}:\R^{n}\times\R^{n}\to \R_{+}$ be a proper, convex function with $L$-Lipschitz gradient, such that
	\begin{equation}
		(\forall (u,v)\in \R^{n}\times\R^{n}) \quad \mathcal{L}(u,v)=0 \Leftrightarrow u=v
	\end{equation}
Let also $w_{0}\in \Im(X^{\top})$, $0<\gamma\leq L\norme{X}_{op}^{2}$	and consider the gradient iteration on the first argument of $\mathcal{L}\circ X$, i.e.
\begin{equation}\label{GDiteration}
	w_{t+1}=w_{t}-\gamma X^{\top}\nabla\mathcal{L}(Xw_{t},y)
\end{equation}
	Then 
	\begin{equation}
		\lim_{t\to +\infty} w_t=w^{\dagger}:=\argmin_{}\{\norme{w} ~ : ~ Xw=y\}
	\end{equation}
\end{lem}

\begin{proof}
First of all, since $d>n$, the set of minimizers of $\mathcal{L}(X\cdot,y)$ is non-empty (i.e. $\exists w\in \R^{d}$, such that $Xw=y$).  By the standard gradient descent analysis (see for example \cite[Paragraph $2.1.5$]{nesterov2013introductory}), it holds that $w_{t}\underset{t\to\infty}{\to}\bar{w}\in \argmin \mathcal{L}(X\cdot,y)$.

In addition, since $w_{0}\in \Im (X^{\top})$, by using \eqref{GDiteration}, and a recurrence argument, we have that t$\{w_{t}\}_{t\geq0}\in \Im (X^{\top})$. By closeness of $\Im (X^{\top})$, it follows that $\bar{w}\in \Im (X^{\top})$.
By definition of $\mathcal{L}$, we have that the set of minimizers $\argmin\mathcal{L}(X\cdot,y)=\{w ~ : ~ Xw=y\}$ is an affine space, hence it can be expressed as $\argmin\mathcal{L}(X\cdot,y) = {\proj_{\argmin\mathcal{L}(X\cdot,y)}(0) }+V$, for some suitable space $V$, where $\proj_{\argmin\mathcal{L}(X\cdot,y)}(0) $ denotes the projection of the $0$ element in $\argmin\mathcal{L}(X\cdot,y)$ and is equal to $\{w^{\dagger}\}$, by definition. 

Let us show that $V = \ker(X)$. Indeed for any $w\in \argmin\mathcal{L}(X\cdot,y)$, we have $w=w^{\dagger} +v$, $v\in V$ and
$y=Xw=Xw^{\dagger} +Xv = y +Xv$, which gives $V\subset\ker(X)$. On the other hand, if $u\in \ker(X)$ and $w^{\prime}=w^{\dagger}+u$, we have that $Xw^{\prime}=Xw^{\dagger}=y$, which allows to conclude that $\argmin\mathcal{L}(X\cdot , y)=\{w^{\dagger}\}+\ker(X)$. By definition of the minimal norm solution it holds that $w^{\dagger}\in (\ker X)^{\perp}$ (see \cite[Proposition $2.3$ and Theorem $2.5$]{engl1996regularization}). By combining the facts $\bar{w} = w^{\dagger}+ u$, $u\in \ker X$, $\bar{w}\in \Im(X^{\top})=(\ker X)^{\perp}$ and $w^{\dagger}\in (\ker X)^{\perp}$, it follows that $\bar{w}=w^{\dagger}$.
\end{proof}

The next lemma establishes the equivalence between the max-margin problem \eqref{max_sphere} and the min-norm problem \eqref{min_norm}. In fact the result holds true for general positively $1$-homogeneous features, i.e. by replacing $M(w)= \min_{i=1, \dots, n} y_i \scal{w}{x_{i}}$ in \eqref{max_sphere} and \eqref{min_norm}, by the general margin:
\begin{equation}\label{marginh}
M(w)= \min_{i=1, \dots, n} y_i h_{i}(w).
\end{equation}
\begin{lem}
[Max margin \& min norm]\label{lemmaClassificationEquivalence}
Let $h_{i}:\R^d\to \R $ be a family of positively $1$-homogeneous functions for all $i\leq n$ and consider the following optimization problems:
\begin{equation}\label{max_sphereh}
w_{+}=\argmax\{M(w)= \min_{i=1, \dots, n} y_i h_{i}(w) ~ : ~ \nor{w}=1\}
\end{equation}

\begin{equation}\label{min_normh}
w_{\ast}=\argmin\{\nor{w}~ : ~ M(w)= \min_{i=1, \dots, n} y_i h_{i}(w)\geq 1\}.
\end{equation}
Then Problem~\eqref{max_sphereh} is equivalent to Problem~\eqref{min_normh}. 
In particular, if $w_\ast$ is the solution of Problem~\eqref{min_normh} then $w_+=w_\ast/\nor{w_\ast}$ is a solution of Problem~\eqref{max_sphereh}. 
Further, if $w_+$ is the solution of Problem~\eqref{max_sphereh} then $w_\ast=\frac{w_+ }{M(w_+)}$ is a solution of Problem~\eqref{min_normh}.
\end{lem}
\bp
First of all, from the definition of $M$, see~\eqref{marginh},  we have that for all $w\in \R^d\setminus\{0\}$
\be\label{marginslack}
M(w)=  \max_{\gamma>0} \gamma
\qquad  \text{s.t. } ~ \norme{w}=1 ~ ~ \& \quad y_i h_i(w)\ge \gamma, \qquad i=1, \dots, n.
\ee

By making the change of variable $w=\frac{w^{\prime}}{\norme{w^{\prime}}}$ and taking into consideration \eqref{marginslack}, we can rewrite the max margin problem related to~\eqref{max_sphereh} as 
$$
\max_{w'\in \R^d\setminus\{0\}, \gamma>0} \gamma \qquad  \text{s.t.} \qquad y_i h_i(w'/\nor{w'})\ge \gamma, \qquad i=1, \dots, n,
$$
and using the homogeneity property of $h_{i}$, 
\be\label{equivalence2}
\max_{w'\in \R^d\setminus\{0\}, \gamma>0} \gamma \qquad  \text{s.t.} \qquad y_i \frac{1}{\nor{w'}}h_i(w')\ge \gamma, \qquad i=1, \dots, n.
\ee
Then setting $\tilde \gamma = \gamma \nor{w'}$,  \eqref{equivalence2} can be equivalently written as
\be\label{equivalence3}
\max_{w'\in \R^d\setminus\{0\}, \tilde  \gamma>0} \frac{\tilde \gamma}{\nor{w'}} \qquad  \text{s.t.} \qquad y_i h_i(w')\ge \tilde \gamma, \qquad i=1, \dots, n.
\ee
Since the above problem is still scale invariant by letting $v=\frac{w^{\prime}}{\tilde{\gamma}}$,
\eqref{equivalence3} is equivalent to
$$
\max_{v\in \R^d\setminus\{0\}} \frac{1}{\nor{v}} \qquad  \text{s.t.} \qquad y_i h_i(v)\ge 1, \qquad i=1, \dots, n.
$$
which is equivalent to the min-norm problem related to ~\eqref{min_normh}  (in the sense that the they have the same set of solutions) . By taking into consideration all the change of variables, it follows that $w_{+}=\frac{w_{\ast}}{\norme{w_{\ast}}}$ and $w_{\ast}=\frac{w_{+}}{M(w_{+})}$, as also that $M(w_{\ast})=1$ and $M(w_{+})=\frac{1}{\norme{w_{\ast}}}$.
\ep

The following Lemma gives an expression of the dependence of the margin and the angle gap as a function of the gap of a method's iterates approaching to the hard-margin solution $w_{\ast}$.  

\begin{lem}[Lemma $2$ \cite{molitor2020bias}: Bounds for angle and margin]\label{implicationrates}
Let $\delta>0$ and $c>0$, and $w_{\ast}\neq 0$ be the min-norm solution as defined in \eqref{min_norm}. Let $\{w_{t}\}_{t\geq1}$ be a sequence such that $\norme{w_{t}}\geq \delta$. Then the following  estimates hold true :
\begin{equation} 1-\frac{\scal{w_{t}}{w_{\ast}}}{\norme{w_{t}}\norme{w_{\ast}}}\leq \frac{1}{2\delta\norme{w_{\ast}}}\norme{w_{t}-w_{\ast}}^2
\end{equation}
\begin{equation}
M\left(\frac{w_{\ast}}{\norme{w_{\ast}}}\right)-M\left(\frac{w_{t}}{\norme{w_{t}}}\right) \leq \frac{\norme{X}_{F}}{\delta\norme{w_{\ast}}}\norme{w_{t}-w_{\ast}}
\end{equation}
\end{lem}

\begin{proof}

Since $\norme{w_{t}}\geq \delta$, by using the identity \begin{equation}
	\norme{u-v}^{2}=\norme{u}^{2}+\norme{v}^2-2\scal{u}{v} \quad \forall u,v\in \R^{d}
\end{equation}
we find :
\begin{equation}\label{tferlako2}
\begin{aligned}
1-\frac{\scal{w_{t}}{w_{\ast}}}{\norme{w_{t}}\norme{w_{\ast}}}&=\frac{2\norme{w_{t}}\norme{w_{\ast}}+\norme{w_{t}-w_{\ast}}^{2}-\norme{w_{t}}^{2}-\norme{w_{\ast}}^{2}}{2\norme{w_{t}}\norme{w_{\ast}}} =\frac{\norme{w_{t}-w_{\ast}}^{2}-\big(\norme{w_{t}}-\norme{w_{\ast}}\big)^{2}}{2\norme{w_{t}}\norme{w_{\ast}}} \\ & \leq \frac{1}{2\delta\norme{w_{\ast}}}\norme{w_{t}-w_{\ast}}^{2}
\end{aligned}
\end{equation}
which allow to prove the first point.

Finally if $j=\argmin_{i\leq n}\{y_{i}\scal{w_{t}}{x_{i}}\}$, then by definition of the margin $M$ and $w_{\ast}$, since $M(w_{\ast})= 1$ (see Lemma \ref{lemmaequivalencemaxmin}), we obtain:
\begin{equation}\label{tferlako}
\begin{aligned}
M\left(\frac{w_{\ast}}{\norme{w_{\ast}}}\right)-M\left(\frac{w_{t}}{\norme{w_{t}}}\right)&=\frac{1}{\norme{w_{\ast}}}-\frac{y_{j}\scal{w_{t}}{x_{j}}}{\norme{w_{t}}} \leq \frac{y_{j}\scal{w_{\ast}}{x_{j}}}{\norme{w_{\ast}}}-\frac{y_{j}\scal{w_{t}}{x_{j}}}{\norme{w_{t}}}=y_{j}\scal{\frac{w_{\ast}}{\norme{w_{\ast}}}-\frac{w_{t}}{\norme{w_{t}}}}{x_{j}}\\ &
\leq \norme{\frac{w_{\ast}}{\norme{w_{\ast}}}-\frac{w_{t}}{\norme{w_{t}}}}\norme{x_{j}} \leq \norme{X}_{F}\norme{\frac{w_{\ast}}{\norme{w_{\ast}}}-\frac{w_{t}}{\norme{w_{t}}}}
\end{aligned} 
\end{equation}
and since 
\begin{equation}
	\norme{\frac{w_{t}}{\norme{w_{t}}}-\frac{w_{\ast}}{\norme{w_{\ast}}}}^{2}=1+1-2\frac{\scal{w_{t}}{w_{\ast}}}{\norme{w_{t}}\norme{w_{\ast}}}=2\bigg(1-\frac{\scal{w_{t}}{w_{\ast}}}{\norme{w_{t}}\norme{w_{\ast}}}\bigg)
\end{equation}
the conclusion follows by combining \eqref{tferlako2} and \eqref{tferlako}.
\end{proof}

The next lemma is a general descent lemma that is classically used for proximal gradient methods applied to structured composite convex optimization problems, such as \eqref{dualminproblem}. The interested reader can find a proof in \cite[Lemma $2.3$]{beck2009fast} 
\begin{lem}\label{descentlemma}
	Let $F=f+g:\R^{n}\longrightarrow\R$, where $f$ and $g$ are convex lower-semi continuous function and $f$ is continuously differentiable  with $L$-Lipschitz gradient. For all $0<\gamma\leq \frac{1}{L}$, consider the operator $T_{\gamma}:\R^n\longrightarrow\R^n$, such that $T_{\gamma}(x):=\prox_{\gamma g}\big(x-\gamma\nabla f(x)\big)$. Then for all $(x,y)\in (\R^n)^2$ it holds:
	
	\begin{equation}\label{descentlemmaeq}
		2\gamma\big(F(T_{\gamma}(y))-F(x)\big)\leq \norme{y-x}^2-\norme{T_{\gamma}(y)-x}^2 
	\end{equation}
\end{lem}

\subsection{Preliminary results}\label{sectionconv}

In this paragraph, we state some basic facts concerning the properties of the dual regularized problem \eqref{dualminproblem}, necessary for the analysis of Algorithms \ref{algodualprojGD} and \ref{algodualinertialGD}.

We first recall the objective function associated to the dual penalized hinge loss problem \eqref{dualminproblem} (see also \eqref{eq:ellstar}), where we take the regularization parameter to be given by a sequence $\{\la_t\}_t$. 
\begin{equation}\label{dualenergy}
D_{t}(u)=\frac{\norme{Z^{\top}u}^2}{2}+\frac{1}{\lambda_{t}}\mathcal{L}^{\ast}(\lambda_{t}u)=\frac{1}{2}\norme{X^{\top}u}^2 +\sum_{i=1}^{n}u^{i}+\iota_{[-1,0]^{n}}(\lambda_{t}u), 
\end{equation}
and  the dual problem  of \eqref{minorm} (see \eqref{dualconstrainedmin}), that is
\begin{equation}\label{dualobjective}
D_{\infty}(u)=\frac{1}{2}\norme{Z^{\top}u}^2 +\sum_{i=1}^{n}u^{i}+\iota_{(-\infty,0]^{n}}(u).
\end{equation}
Below we state some fundamental properties of the dual objective function $D_{t}$, that will be useful for the convergence analysis of both Algorithms \ref{algodualprojGD} and \ref{algodualinertialGD}.
We start by showing that the sequence of regularized dual  functions $D_t$ is monotonically pointwise decreasing to $D_\infty$. 
\begin{lem}\label{remarkdecreasingDt}
Let $\{\lambda_{t}\}_{t\geq0}$ be a  sequence of positive parameters decreasing to zero,  and consider the functions $D_{t}$ and $D_{\infty}$ defined in \eqref{dualenergy} and \eqref{dualobjective} respectively. Then,  for all $u\in \R^n$ the sequence $\{D_{t}(u)\}_{t\geq0}$ is non-increasing. In addition, for any $u\in \R^n$, it holds, 
\begin{equation}
D_{t}(u) \underset{t\to\infty}{\longrightarrow}D_{\infty}(u).
\end{equation}
\end{lem}

\begin{proof}[\textbf{Proof of Lemma \ref{remarkdecreasingDt}}]
	
Since $\lambda_{t}$ is non-increasing, for every $u\in \R^{n}$ it follows that the function \(t\mapsto\frac{1}{\lambda_{t}}\mathcal{L}^{\ast}(\lambda_{t}u)=\sum_{i=1}^{n}u^{i} + \iota_{[-1,0]}(\lambda_{t}u^{i})\) 	is non-increasing in $[0,+\infty]$.

In addition, by direct computation for all $u\in \R^n$, it holds: 
\begin{equation}
\frac{1}{\lambda_{t}}\mathcal{L}^{\ast}(\lambda_{t}u)=\sum_{i=1}^{n}u^{i}+\iota_{(-1,0]^{n}}(\lambda_{t}u)\underset{t\to\infty}{\longrightarrow}\sum_{i=1}^{n}u^{i}+\iota_{(-\infty,0]^{n}}(u)
\end{equation} 

Thus, for all $u\in \R^n$, the function $D_{t}(u)$ is non-increasing in $t\geq0$ and that \(D_{t}(u) \underset{t\to\infty}{\longrightarrow}D_{\infty}(u).\)
\end{proof}
The following lemma, is well-known (see e.g. \cite[Proposition 94]{SalVil21}), and we recall it here for the sake of completeness. It plays a fundamental role in our 
analysis since it provides a bound of the distance of the primal iterates from the min norm solution \eqref{min_norm} in terms of the distance of the dual objective function from its minimum. 
\begin{lem}\label{lemmadualprimalsequence}
	Let $w_{\ast}$ be the minimal norm separating solution defined in \eqref{min_norm}, and $D_{\infty}$ the associated dual problem defined in \eqref{dualobjective}. Then
	\begin{equation}\label{conditionexistencedual}
		\argmin D_{\infty} \neq \emptyset.
	\end{equation}
	In addition for every $u \in \R^n$ and every $w=-Z^\top u$ and $u_{\ast}\in \argmin D_{\infty}$, we have,
	\begin{equation}\label{E:primal-dual value-iterate bound}
		\frac{1}{2} \norme{w - w_{\ast}}^2 \leq D_{\infty}(u) - D_{\infty}(u_{\ast}).
	\end{equation}
\end{lem}

\begin{proof}[\textbf{Proof of Lemma \ref{lemmadualprimalsequence}}]
	From the separability assumption \eqref{assumptionseparate} there exists some \(\tilde{w}\in \R^d\), such that \(\scal{\tilde{w}}{z_{i}}>0 ~, \forall i\leq n\). Let \(i_{min}=\argmin\{\scal{\tilde{w}}{z_{i}}~ : ~ i\leq n  \}\) and let $M(\tilde{w})=\scal{\tilde{w}}{z_{i_{min}}}>0$. Then by setting \(w^{\prime}=\frac{2\tilde{w}}{M(\tilde{w})}\), we have \(\scal{w^{\prime}}{z_{i}}=\frac{\scal{2\tilde{w}}{z_{i}}}{M(\tilde{w})}\geq2>1 ~ , ~ \forall i\leq n\). Thus we deduce the existence of an element $w^{\prime}\in \dom \frac{\norme{\cdot}^2}{2} = \R^n$, such that \(\mathcal{L}\) is continuous at $Xw^{\prime}$ (since $\mathcal{L}$ is continuous and $w^{\prime}\in \dom\mathcal{L}$). By \cite[Corollary $3.31$]{peypouquet2015convex} and the optimality condition for the min-norm separating solution $w_{\ast}$, the sum rule for the subdifferential holds and thus we have
	\begin{equation}\label{dualexistance}
		0\in \partial \left(\frac{\norme{\cdot}^2}{2}\ +\iota_{[1,+\infty)^{n}}\circ Z\right)(w_{\ast})=\nabla\left(\frac{\norme{\cdot}^2}{2}\right)(w_{\ast})+Z^{\top}\partial \iota_{[1,+\infty)^{n}}(Zw_{\ast})= \{w_{\ast}\} + Z^\top Z^{\top}\partial \iota_{[1,+\infty)^{n}}(Zw_{\ast})
	\end{equation}
	
	From \eqref{dualexistance}, there exists some $u_{\ast}\in \partial \iota_{[1,+\infty)^{n}}(Zw_{\ast})$ such that \(w_{\ast}=-Z^{\top} u_{\ast}\) or equivalently \(Zw_{\ast}\in \partial \iota^{\ast}_{[1,+\infty)^{n}}(u_{\ast})\).  Hence, by replacing $w_{\ast}=-Z^{\top}u_{\ast}$, we obtain:
	\begin{equation}
		\begin{aligned}
			& -ZZ^{\top}u_{\ast}\in \partial \iota^{\ast}_{[1,+\infty)^{n}}(u_{\ast}) ~ \Leftrightarrow ~ \\
			& ~ \Leftrightarrow ~ 0\in -Z\nabla \left(\frac{\norme{\cdot}^2}{2}\right)(-Z^{\top}u_{\ast})+\partial \iota^{\ast}_{[1,+\infty)^{n}}(u_{\ast}) \subset \partial \left(\frac{\norme{-Z^{\top}\cdot}^2}{2}+\iota^{\ast}_{[1,+\infty)^{n}}(\cdot)\right)(u_{\ast})
		\end{aligned}
	\end{equation}
	which shows that $u_{\ast}\in \argmin D_{\infty}$.
	In order to prove \eqref{E:primal-dual value-iterate bound}, let $u\in \R^{n}$ and $w=-Z^{\top}u$.
	
	Since  $Zw_{\ast}\in \partial \iota_{[1,+\infty)^{n}}(u_{\ast})$ and \(-Z^{\top}u=w\), by using the Fenchel-Young equality for \(\frac{\norme{-Z^{\top}u_{\ast}}^2}{2}\)and  \(\frac{\norme{-Z^{\top}u}^2}{2}\), we find:
	\begin{equation}
		\begin{aligned}
			D_{\infty}(u)-D_{\infty}(u_{\ast})&=  \frac{\norme{-Z^{\top}u}^2}{2}- \frac{\norme{-Z^{\top}u_{\ast}}^2}{2} +\iota^{\ast}_{[1,+\infty)^{n}}(u)-\iota^{\ast}_{[1,+\infty)^{n}}(u_{\ast}) \\
			& = \scal{-Z^{\top}u}{w} -\frac{\norme{w}^2}{2} +\scal{Z^{\top}u_{\ast}}{w_{\ast}} + \frac{\norme{w_{\ast}}^2}{2}+\iota^{\ast}_{[1,+\infty)^{n}}(u)-\iota^{\ast}_{[1,+\infty)^{n}}(u_{\ast}) \\
			& = \frac{\norme{w_{\ast}}^2}{2}-\frac{\norme{w}^2}{2}- \scal{-Z^{\top}u}{w_{\ast}-w} \\
			& \quad  +\iota^{\ast}_{[1,+\infty)^{n}}(u)-\iota^{\ast}_{[1,+\infty)^{n}}(u_{\ast}) -\scal{Zw_{\ast}}{u-u_{\ast}}
		\end{aligned}
	\end{equation}
	By using the strong convexity of \(\frac{\nor{\cdot}^2}{2}\) (with parameter $1$) and the convexity of $\iota^{\ast}_{[1,+\infty)^{n}}(\cdot)$, we conclude that:
	\begin{equation}
		D_{\infty}(u)-D_{\infty}(u_{\ast}) \geq \frac{\nor{w-w_{\ast}}^2}{2}
	\end{equation}
\end{proof}

As a consequence of Lemma \ref{remarkdecreasingDt}, the following  Lemma provides some basic estimates for the sequence $\{D_{t}(u_{t+1})\}_{t\geq 0}$  generated by Algorithm \ref{algodualprojGD}.
\begin{lem}\label{lemmaenergyFB}
Let $u_{\ast}\in argmin D_{\infty}$ and $\{u_t\}_{t\geq0}$ be the sequence generated by Algorithm \ref{algodualprojGD}. Then the following estimate holds for all $u\in \R^n$ :
\begin{equation}
D_{t}(u_{t+1}) - D_{\infty}(u_{\ast})+\frac{1}{2\gamma}\norme{u_{t+1}-u}^2\leq D_{t}(u)- D_{\infty}(u_{\ast})+\frac{1}{2\gamma}\norme{u_{t}-u}^2
\end{equation}
In addition, $\{D_{t}(u_t)\}_{t\geq 0}$ is non-increasing and \begin{equation}\label{relationdescentlemmafordual2}
D_{t}(u_{t})\underset{t\to\infty}{\longrightarrow}D_{\infty}(u_{\ast})
\end{equation}
\end{lem}

\begin{proof}[\textbf{Proof of Lemma \ref{lemmaenergyFB}}]
	
Let $u_{\ast}\in \argmin D_{\infty}$. Since \(\frac{\norme{\cdot}^2}{2}\circ Z^{\top}\) has $\norme{XX^{\top}}_{op}$-Lipschitz gradient, by applying the Descent Lemma \ref{descentlemma} with $y=u_{t}$, for all $u\in \R^n$ and $t\geq 0$ it holds:
	\begin{equation}\label{relationdescentlemmafordual}
		D_{t}(u_{t+1}) +\frac{1}{2\gamma}\norme{u_{t+1}-u}^2\leq D_{t}(u)+\frac{1}{2\gamma}\norme{u_{t}-u}^2
	\end{equation}
	which allows to deduce \eqref{relationdescentlemmafordual2}.
	
	By choosing $u=u_{t}$ in \eqref{relationdescentlemmafordual} and using the non-increasing property of $D_{t}$ in $t\geq 0$ (see Lemma \ref{remarkdecreasingDt}), we obtain:
	\begin{equation}
		D_{t+1}(u_{t+1}) -D_{\infty}(u_{\ast}) \leq D_{t}(u_{t+1})-D_{\infty}(u_{\ast})   \leq D_{t}(u_{t})-D_{\infty}(u_{\ast})  -\frac{1}{2\gamma}\norme{u_{t+1}-u_{t}}^2
	\end{equation} 
	which shows that the function $r_{t}=D_{t}(u_{t})-D_{\infty}(u_{\ast})$ is non-increasing in $t$. Since $D_{t}(u)$ is also non-increasing in $t$ and convergent to $D_{\infty}(u)$, from \eqref{relationdescentlemmafordual} we also have that $\frac{1}{2\gamma}\norme{u_{t}-u}^2$ is bounded for all $u\in \R^n$. In addition $r_{t}=D_{t}(u_{t})-D_{\infty}(u_{\ast})\geq D_{\infty}(u_{t})- D_{\infty}(u_{\ast})\geq 0$, which shows that $r_{t}$ is also bounded from below by zero, and therefore converges to a non-negative limit. 
	
	By adding and subtracting $D_{\infty}(u_{\ast})$ in \eqref{relationdescentlemmafordual}, for all $u\in \R^n$, we find:
	\begin{equation}
		r_{t+1}-(D_{t}(u)-D_{\infty}(u_{\ast})) \leq \frac{1}{2\gamma}\norme{u_{t}-u}^{2}-\frac{1}{2\gamma}\norme{u_{t+1}-u}^{2}
	\end{equation} 
	which by summing up to $t\geq 1$ gives:
	\begin{equation}\label{enddescentlemmadual}
		\sum_{s=1}^{t}\big(r_{s}-(D_{s-1}(u)-D_{\infty}(u_{\ast}))\big) \leq \frac{1}{2\gamma}\norme{u_{0}-u}^{2} \quad \forall u\in \R^n
	\end{equation}
	The last relation allows to conclude that $\lim\limits_{t\to\infty}\big(r_{t} -(D_{t-1}(u)- D_{\infty}(u_{\ast}))\big)=0$ and since $\lim\limits_{t\to\infty}D_{t-1}(u)=D_{\infty}(u)\geq  D_{\infty}(u_{\ast})$ and $\lim\limits_{t\to\infty}r_{t}\geq 0$, we have that $\lim\limits_{t\to\infty}r_{t}=0$.
\end{proof}
In the next lemma we prove that each of the regularized dual problems $D_t$ satisfies the \L ojasiewicz condition \eqref{PLinequality} with a common constant $\mu$, not depending on $t$.
The proof of Lemma \ref{lemmaPL} (see below) is inspired by the analysis presented in \cite[Lemma $2.5$]{beck2017linearly} with some modifications.  More precisely we will prove that there exists some positive constants $M,~R,~\theta$, such that for all $t\geq 0$, $D_{t}$ satisfies the following growth condition:

\begin{equation}\label{eq: GC for Dt}
	\left(\forall u\in [D_{t}\leq \min D_{t} + M] \cap \mathbb{B}(\mathbf{0},R)\right), \qquad \frac{\theta}{2}\text{dist}\left({u,\argmin D_{t}}\right)^2\leq D_{t}(u)-\min D_{t}
\end{equation}
In fact, relation \eqref{eq: GC for Dt} can be met under the name quadratic growth (see e.g. \cite{garrigos2017convergence}) and is equivalent to \eqref{PLinequality}.

It is worth mentioning that we cannot apply directly  the proof in  \cite[Lemma $2.5$]{beck2017linearly}, due to the fact of possible unboudedness of $\argmin D_{t}$, when $t\to\infty$. Indeed by applying directly \cite[Lemma $2.5$]{beck2017linearly} we can deduce the existence of $\theta_{t}$, such that \eqref{eq: GC for Dt} hold true. However this is not sufficient for establishing linear convergence of the proposed scheme (see Proposition \ref{ratesdualenergy}), since, in general, $\theta_{t}$ may vanish asymptotically.	

\begin{lem}\label{lemmaPL}
	Let $R>0$, $M>0$ and $\{\lambda_{t}\}_{t\geq0}$ be a sequence of positive parameters decreasing to zero.  Consider the functions $D_{t}$ and $D_{\infty}$ defined in \eqref{dualenergy} and \eqref{dualobjective} respectively. Then there exists some $\mu>0$, such that for all $t\geq 0$, $D_{t}$ satisfies the $\mu$-\L ojasiewicz condition \eqref{PLinequality} in $[D_{t}\leq \min D_{t} + M] \cap \mathbb{B}(\mathbf{0},R)$, with $\mu$ given by \eqref{eq:mu}, i.e.
	\begin{equation}\label{PLinequality2}
	\left(\forall u\in [D_{t}\leq \min D_{t} + M] \cap \mathbb{B}(\mathbf{0},R)\right), \qquad		D_{t}(u) -\min D_{t} \leq \frac{1}{2\mu}\text{dist}\left(\partial D_{t}(u),0\right)^2,
	\end{equation}

\end{lem}

\begin{proof}[\textbf{Proof of Lemma \ref{lemmaPL}}]
For all $t\geq 0$, the problem $\underset{u\in \R^n}{\min} D_{t}(u)$ is equivalent to the following constrained optimization one:
	\begin{equation}\label{eq: minproblemF}
		\min\left\{F(u):=\frac{1}{2}\norme{Z^{\top}u}^2+\scal{\bf{1}}{u} ~ : ~ u\in \mathcal{U}_{t}:=[-\lambda_{t},0]^n \right\}.
	\end{equation}
If $u \notin \mathcal{U}_{t}$, relation \eqref{eq: GC for Dt} holds, so without loss of generality, we assume that $u\in \mathcal{U}_{t}$, so that $F(u)=D_{t}(u)$.
	
	For all $t\geq 0$, let $\bar{u}_{t} \in \underset{u\in \R^n}{\argmin} D_{t}(u)$. By the optimality  conditions for problem \eqref{eq: minproblemF}, for all $u\in \mathcal{U}_{t}$, it holds:
	\begin{equation}\label{KKTF}
		\scal{\nabla F(\bar{u}_{t})}{u-\bar{u}_{t}} \geq 0
	\end{equation}

	Notice that $\mathcal{U}_{t}$ is a polyhedral set since $\mathcal{U}_{t}=\{u\in \R^{n} ~ : ~ Au\leq a_{t}\}$, with $A=\begin{bmatrix}
		\text{Id}_{n} \\
		-\text{Id}_{n}
	\end{bmatrix}\in \R^{2n\times n}$ and $a_{t}=-\lambda_{t}^{-1}\begin{bmatrix}
		\mathbf{0} \\ \mathbf{1}
	\end{bmatrix}\in \R^{2n}$.
	
	By \cite[Lemma $2.3$ ]{beck2017linearly}, there exist a unique vector $\bar{v} \in \R^{n}$ and a scalar $\bar{s}\in \R$ such that 
	\off{
		for any $\bar{u}_{t}\in \argmin D_{t}(u)$ we have $E\bar{u}_{t}=\begin{bmatrix}
		\bar{v}\\ \bar{s}
	 \end{bmatrix}$, with $E=\begin{bmatrix} X^{\top} \\ \mathbf{1}^{\top}
	 \end{bmatrix} \in \R^{(d+1)\times n}$ and $A\bar{u}_{t}\leq a_{t}$. In particular
}
 the following equivalence holds true:
	\begin{equation}\label{forHoffman}
		\bar{u}_{t}\in \argmin D_{t} ~ \Leftrightarrow ~ \begin{bmatrix} Z^{\top} \\ \mathbf{1}^{\top}
		\end{bmatrix} \bar{u}_{t}=\begin{bmatrix}
			\bar{v}\\ \bar{s}
		\end{bmatrix} \quad \text{and } \quad  A\bar{u}_{t}\leq a_{t}
	\end{equation}
	so that $ \argmin D_{t}=S\cap \mathcal{U}_{t}$, where $S=\left\{u\in \R^n ~ : \begin{bmatrix} Z^{\top} \\ \mathbf{1}^{\top}
	\end{bmatrix} u=\begin{bmatrix}
		\bar{v}\\ \bar{s}
	\end{bmatrix} \right\}$ and $\mathcal{U}_{t}=\{u\in \R^n ~ : Au\leq a_{t}\}$.
	
	According to Hoffman's lemma \cite{hoffman1952approximate} (see also \cite[Lemma $15$]{wang2014iteration}) for the polyhedral sets $S$ and $\mathcal{U}_{t}$, by setting  $E=\begin{bmatrix} Z^{\top} \\ \mathbf{1}^{\top}
	\end{bmatrix} \in \R^{(d+1)\times n}$,  there exists some positive constant $\tau$ given by \eqref{hoffmansconstant}, such that
	\begin{equation}\label{HoffmansBound}
		\text{dist}\left(u,\argmin D_{t}\right)= \text{dist}\left(u,S \cap \mathcal{U}_{t}\right)\leq \tau \norme{Eu-\begin{bmatrix}
				\bar{v} \\ \bar{s}
		\end{bmatrix} }
	\end{equation}
	It is important to stress out that the Hoffman's error bound constant $\tau$ only depends on the matrices $E$ and $A$ (see e.g. \cite[Remark $1$]{zualinescu2003sharp} and the associated references). In our setting this means that the constant $\tau$ only depends on the data-matrix $Z=\text{diag(Y)X}$.
	
	By taking the squares in \eqref{HoffmansBound}, we find:
	\begin{equation}\label{hoffmansquared}
		\text{dist}\left(u,\argmin D_{t}\right)^{2} \leq \tau^2\left( \norme{Z^{\top}(u-\bar{u}_{t})}^2 + (\scal{1}{u-\bar{u}_{t}})^2\right)
	\end{equation}
	
	Let us now bound appropriately the two terms in the right-hand-side of \eqref{hoffmansquared}.
	
	For the first term, by developing the square we find: 
	\begin{equation}\label{eqfor1vector}
		\begin{aligned}
			\frac{1}{2}\norme{Z^{\top}(u-\bar{u}_{t})}^2 & = \frac{1}{2}\norme{Z^{\top}u}^{2}-\scal{Z^\top u}{Z^{\top} \bar{u}_{t}} +\frac{1}{2}\norme{Z^{\top}\bar{u}_{t}}^2  \\
			& = \frac{1}{2}\norme{Z^{\top}u}^{2}-\frac{1}{2}\norme{Z^{\top}\bar{u}_{t}}^2-\scal{Z^\top \bar{u}_{t}}{Z^{\top} (u-\bar{u}_{t})} \\
			& \leq \scal{\nabla F(\bar{u}_{t})}{u-\bar{u}_{t}}+\frac{1}{2}\norme{Z^{\top}u}^{2}-\frac{1}{2}\norme{Z^{\top}\bar{u}_{t}}^2-\scal{Z^\top \bar{u}_{t}}{Z^{\top} (u-\bar{u}_{t})}
		\end{aligned}
	\end{equation}
	where in the last inequality we used the optimality condition \eqref{KKTF}.
	In addition, since $\nabla F(\bar{u}_{t})=\mathbf{1} +ZZ^{\top}\bar{u}_{t}$, from \eqref{eqfor1vector} it follows:
	\begin{equation}\label{boundbyF}
		\frac{1}{2}\norme{Z^{\top}(u-\bar{u}_{t})}^2 \leq \frac{1}{2}\norme{Z^{\top}u}^{2}+\scal{\mathbf{1}}{u}-\frac{1}{2}\norme{Z^{\top}\bar{u}_{t}}^2-\scal{\mathbf{1}}{\bar{u}_{t}}= D_{t}(u)-\min D_{t}
	\end{equation}
	
	Let us now provide an upper bound for the term $(\scal{1}{u-\bar{u}_{t}})^2$. On the one hand we have:
	\begin{equation}\label{firstfor1u}
		\begin{aligned}
			\scal{\bf{1}}{\bar{u}_{t}-u} &= \scal{\nabla F(\bar{u}_{t})}{\bar{u}_{t}-u} -  \scal{Z^{\top}\bar{u}_{t}}{Z^{\top}(\bar{u}_{t}-u)} \\ 
			&= \scal{\nabla F(\bar{u}_{t})}{\bar{u}_{t}-u} -\norme{Z^{\top}(u-\bar{u}_{t})}^2 +\scal{Z^{\top}u}{Z^{\top}(\bar{u}_{t}-u)} \\
			& \leq \scal{\nabla F(\bar{u}_{t})}{\bar{u}_{t}-u} +\norme{Z^{\top}u}\norme{Z^{\top}(\bar{u}_{t}-u)} \\
			& \leq \norme{Z^{\top}u}\norme{Z^{\top}(\bar{u}_{t}-u)} 
		\end{aligned}
	\end{equation}
	where in the first inequality we used the Cauchy-Schwarz inequality and for the last one the optimality condition \eqref{KKTF}.
	On the other hand we find:
	\begin{equation}\label{secondfor1u}
		\begin{aligned}
			\scal{\bf{1}}{u-\bar{u}_{t}}&=\scal{\nabla F(\bar{u}_{t})}{u-\bar{u}_{t}} -\scal{Z^{\top}\bar{u}_{t}}{Z^{\top}(u-\bar{u}_{t})} \\
			& = \scal{\nabla F(\bar{u}_{t})}{u-\bar{u}_{t}} +\norme{Z^{\top}(u-\bar{u}_{t})}^2 -\scal{Z^{\top}u}{Z^{\top}(u-\bar{u}_{t})} \\
			&\leq \scal{\nabla F(\bar{u}_{t})}{u-\bar{u}_{t}} +\norme{Z^{\top}(u-\bar{u}_{t})}^2 +\norme{Z^{\top}u}\norme{Z^{\top}(u-\bar{u}_{t})} \\
			& = \scal{\nabla F(\bar{u}_{t})}{u-\bar{u}_{t}} +\left(\norme{Z^{\top}(u-\bar{u}_{t})} +\norme{Z^{\top}u}\right)\norme{Z^{\top}(u-\bar{u}_{t})}\\
			& \leq D_{t}(u)-\min D_{t} +\left(\norme{Z^{\top}(u-\bar{u}_{t})} +\norme{Z^{\top}u}\right)\norme{Z^{\top}(u-\bar{u}_{t})}
		\end{aligned}
	\end{equation}
	where in the first inequality we used the Cauchy-Schwarz inequality and in the last one the convexity of $F$. 
	By relations \eqref{firstfor1u} and \eqref{secondfor1u}, it follows that:
	\begin{equation}\label{boundforonescalar}
		\begin{aligned}
			\abs{\scal{\bf{1}}{u-\bar{u}_{t}}} & \leq D_{t}(u)-\min D_{t} +\left(\norme{Z^{\top}(u-\bar{u}_{t})} +\norme{Z^{\top}u}\right)\norme{Z^{\top}(u-\bar{u}_{t})} \\
			& \leq \left(3\sqrt{D_{t}(u)-\min D_{t}} +\sqrt{2}\norme{Z^{\top}u}\right)\sqrt{D_{t}(u)-\min D_{t}} \\
			& \leq \left(3\sqrt{M}+\sqrt{2}R\norme{X}_{\text{op}}\right)\sqrt{D_{t}(u)-\min D_{t}}
		\end{aligned}
	\end{equation}
	where in the second inequality we used \eqref{boundbyF} and in the last one the fact that $u\in [D_{t}\leq \min D_{t} + M] \cap \mathbb{B}(0,R)$ and the definition of the norm operator $\norme{Z}_{op}=\norme{X}_{op}$.
	By injecting relations \eqref{boundbyF} and \eqref{boundforonescalar} into \eqref{hoffmansquared}, for all $u\in [D_{t}\leq \min D_{t} + M] \cap \mathbb{B}(0,R) $ we find:
	\begin{equation}
		\begin{aligned}
			\text{dist}\left(u,\argmin D_{t}\right)^{2} & \leq \tau^2\left( 2\left(D_{t}(u)-\min D_{t}\right) + \left(3\sqrt{M}+\sqrt{2}R\norme{X}_{\text{op}}\right)^{2}\left(D_{t}(u)-\min D_{t}\right)\right)\\
			& = \tau^{2}\left(9M +6\sqrt{2M}R\norme{X}_{\text{op}} +2R^{2}\norme{X}_{\text{op}}^{2}+2\right)\left(D_{t}(u)-\min D_{t}\right)
		\end{aligned}
	\end{equation}
	which shows that for all $t\geq 0$, $D_{t}$ satisfies the growth condition \eqref{eq: GC for Dt} in $[D_{t}\leq \min D_{t} + M] \cap \mathbb{B}(0,R) $, with \(\theta=\left(2\tau^{2}\left(9M +6\sqrt{2M}R\norme{X}_{\text{op}} +2R^{2}\norme{X}_{\text{op}}^{2}+2\right)\right)^{-1}\)

	Finally by using the equivalence between the $\theta$-growth condition \eqref{eq: GC for Dt} and the $\mu$-\L ojasiewicz condition \eqref{PLinequality} with $\mu=\frac{\theta}{4}$ (see e.g. \cite[Proposition $1$]{apidopoulos2022convergence} or \cite[Theorem $5$]{bolte2017error}), we deduce that for all $t\geq 0$, $D_{t}$ satisfies \eqref{PLinequality}, with
	\begin{equation}
		\mu = \frac{1}{8\tau^{2}\left(9M +6\sqrt{2M}R\norme{X}_{\text{op}} +2R^{2}\norme{X}_{\text{op}}^{2}+2\right)}
	\end{equation}
	and $\tau$ is the Hoffman's constant as defined in \eqref{hoffmansconstant}.
\end{proof}


\subsection{Proof of Theorem \ref{basicteoGD}}\label{subsectionFBdualhinge}

\off{
If $\lambda=\lambda_{t}$ (diagonal/iterative regularization), with $\lambda_{t}\underset{t\to\infty}{\longrightarrow} 0$ (since the functional $\mathcal{L}$ is $1$-well conditioned), then Theorem $1$ of \cite{garrigos2018iterative} can be applied \vas{(to recheck conditions : boundedness of $u_{t}$! )}, which allow to conclude with the convergence to a min-norm solution. In particular (see \cite{garrigos2018iterative}), we have:
\begin{equation}
\norme{w_{t}-w_{\ast}}=\grandO{t^{-\frac{1}{2}}}
\end{equation}

}

In this paragraph we provide the proof of Theorem \ref{basicteoGD}, concerning Algorithm \ref{algodualprojGD}. 
We start with the following proposition which allows to deduce an upper bound for the gap of the dual objective function $D_{t}(u_{t})$ and its minimum value. 

\begin{prop}\label{ratesdualenergy}
Let $u_{\ast}\in argmin D_{\infty}$ and let $\{u_{t}\}_{t\geq 0}$ be the sequence generated by Algorithm \ref{algodualprojGD} with $\lambda_{0}\leq \norme{u_{\ast}}^{-1}$. Then, the following estimate holds for all $t\geq 1$ :
\begin{equation}
	D_{t}(u_t)-D_{\infty}(u_{\ast})\leq\left(D_{0}(u_{0})-D_{\infty}(u_{\ast})\right)\left(1-\frac{\gamma\mu}{1+\gamma\mu}\right)^{t}
\end{equation}
where $\mu$ is given by \eqref{eq:mu}. 
\end{prop}

\begin{proof}[\textbf{Proof of Proposition \ref{ratesdualenergy}}]
Let $u_{\ast}\in \argmin D_{\infty}$  and $\lambda_{0}>0$, such that $\lambda_{0}\leq \norme{u_{\ast}}^{-1}$. Following the proof of Lemma \ref{lemmaenergyFB}, by choosing $u=u_{\ast}$ in \eqref{relationdescentlemmafordual}, we find:
\begin{equation}\label{omg1}
	D_{t}(u_{t+1})-D_{t}(u_{\ast}) +\frac{1}{2\gamma}\norme{u_{t+1}-u_{\ast}}^{2}\leq \frac{1}{2\gamma}\norme{u_{t}-u_{\ast}}^{2}
\end{equation}
Since $D_{t}(u_{\ast}) \leq D_{t}(u)$, for all $u\in \R^n$, by neglecting the non-negative term and summing over $t\geq 0$ relation \eqref{omg1}, we deduce that for all $t\geq 0$ it holds $\norme{u_{t}-u_{\ast}}\leq \norme{u_{0}-u_{\ast}}$, therefore the sequence $\{u_{t}\}_{t\geq 0}$ is bounded by $R=2\norme{u_{\ast}}+\norme{u_{0}}$.

On the other hand, by choosing $u=u_{t}$ in \eqref{relationdescentlemmafordual}and using the non-increasing property of $D_{t}$ (Lemma \ref{remarkdecreasingDt}), we find:
\begin{equation}\label{equation98bis}
	D_{t+1}(u_{t+1})-D_{t}(u_{t})\leq D_{t}(u_{t+1})-D_{t}(u_{t})\leq -\frac{1}{2\gamma}\norme{u_{t+1}-u_{t}}^2
\end{equation}
which allows to conclude that the sequence $\{u_{t}\}_{t\geq 0}\in [D_{t} \leq D_{0}(u_{0})]=\{u\in \R^{n} ~ : ~ D_{t}(u)\leq D_{0}(u_{0})\}$.

By definition of $u_{t+1}=\lambda_{t}^{-1}\prox_{\gamma\lambda_{t} \mathcal{L}^{\ast}}\left(\lambda_{t}\left(u_{t}-\gamma ZZ^{\top}u_{t}\right)\right)$ (see Algorithm \ref{algodualprojGD}) and the characterization of the proximal operator,
for all $t\geq 0$, we have \(\lambda_{t}u_{t+1}+\gamma\lambda_{t}\partial\mathcal{L}^{\ast}(\lambda_{t}u_{t+1})\ni \lambda_{t}\left(u_{t}-\gamma ZZ^{\top}u_{t}\right) \) or equivalently 
\begin{equation}\label{forlojads}
	\left(\text{Id}-\gamma ZZ^{\top}\right)(u_{t}-u_{t+1})\in  \gamma\partial\mathcal{L}^{\ast}(\lambda_{t}u_{t+1})+\gamma ZZ^{\top}u_{t+1} =\partial \gamma D_{t}(u_{t+1})
\end{equation}
Hence \eqref{forlojads} together with the contraction property of the operator \(\text{Id}-\gamma ZZ^{\top}\) for all $\gamma \leq \frac{1}{\norme{XX^{\top}}}$, yields:
\begin{equation}\label{forlajads2}
\text{dist}\left(\partial D_{t}(u_{t+1}),0\right) \leq \gamma^{-1}\norme{	\left(\text{Id}-\gamma ZZ^{\top}\right)(u_{t}-u_{t+1})} \leq	\gamma^{-1}\norme{u_{t+1}-u_{t}} 
\end{equation}

By combining relations \eqref{equation98bis} and \eqref{forlajads2}, we obtain:
\begin{equation}\label{FBforPL}
	D_{t}(u_{t+1})-D_{t}(u_{t})\leq -\frac{\gamma}{2}\text{dist}\left(\partial D_{t}(u_{t+1}),0\right)^2 
\end{equation}
Since $u_{t}\in [D_{t}\leq \min D_{t} +M]\cap \mathbb{B}_{R}(\mathbf{0})$ with $M=D_{0}(u_{0})-D_{\infty}(u_{\ast})$ and $R=\norme{u_{0}}+2\norme{u_{\ast}}$, by using \eqref{PLinequality} and Lemma \ref{lemmaPL} we find:
\begin{equation}
	D_{t}(u_{t+1})-D_{t}(u_{t})\leq -\gamma\mu\left( D_{t}(u_{t+1})-\min D_{t}(u)\right)
\end{equation}
By adding and subtracting $D_{\infty}(u_{\ast})$  and using that $D_{t+1}(u_{t+1})\leq D_{t}(u_{t+1})$ and $D_{\infty}(u) \leq \min D_{t}(u)$ (see Lemma \ref{remarkdecreasingDt}), for all $t\geq 1$, we derive:
\begin{equation}\label{forinductionexponential}
	(1+\gamma\mu)\left(D_{t+1}(u_{t+1})-D_{\infty}(u_{\ast})\right) \leq \left(D_{t}(u_{t})-D_{\infty}(u_{\ast})\right)
\end{equation}
By induction in \eqref{forinductionexponential}, for all $t\geq 0$, it follows 
\begin{equation}
	D_{t}(u_{t})-D_{\infty}(u_{\ast}) \leq \left(D_{0}(u_{0})-D_{\infty}(u_{\ast})\right)\left(1-\frac{\gamma\mu}{1+\gamma\mu}\right)^{t}
\end{equation}
which allows to conclude the proof.
\end{proof}

Next, the proof of Theorem \ref{basicteoGD} can then be derived by using Proposition~\ref{ratesdualenergy}.
\begin{proof}[\textbf{Proof of Theorem \ref{basicteoGD}}] 
Let $u_{\ast}\in \argmin D_{\infty}$ such that $\lambda_{0}\leq \norme{u_{\ast}}^{-1}$. Since $w_{t}=-Z^{\top}u_{t}$, Lemma \ref{lemmadualprimalsequence} yields
\begin{equation*}
\norme{w_{t}-w_{\ast}} \leq \sqrt{2\left(D_{\infty}(u_{t}) - D_{\infty}(u_{\ast})\right)}
\end{equation*}
Lemma \ref{remarkdecreasingDt} and the non-increasing property of $D_{t}(u_t)$ proved in Lemma~\ref{lemmaenergyFB} imply
\begin{equation*}
\norme{w_{t}-w_{\ast}} \leq \sqrt{2\left(D_{t}(u_{t}) - D_{\infty}(u_{\ast})\right)}
\end{equation*}
We then derive from Proposition \ref{ratesdualenergy} that
\begin{equation}\label{eqforGDmarginangle}
	\begin{aligned}
\norme{w_{t}-w_{\ast}} \leq \sqrt{2\left(D_0(u_0)-D_\infty(u_*)\right)}\left(1-\frac{\gamma\mu}{1+\gamma\mu}\right)^{\frac{t}{2}},
\end{aligned}
\end{equation}
from which  \eqref{eq: basicteoGD} follows using the definition of $D_0$ and $D_\infty$. The  inequality \(\norme{w_{\ast}}\leq \norme{w_{t}}+\norme{w_{t}-w_{\ast}}\) and the bound \eqref{eqforGDmarginangle} give
\begin{equation}
	\begin{aligned}
	\norme{w_{t}}\geq \norme{w_{\ast}} -\norme{w_{t}-w_{\ast}}&\geq \norme{w_{\ast}}-\sqrt{\norme{w_{0}}^{2}-\norme{w_{\ast}}^2 +\scal{\mathbf{1}}{u_{0}-u_{\ast}}}\left(1-\frac{\gamma\mu}{1+\gamma\mu}\right)^{\frac{t}{2}}\\
	&=\norme{w_{\ast}}\left(1-\sqrt{\frac{\norme{w_{0}}^{2}}{\norme{w_{\ast}}^2}-1 +\frac{1}{\norme{w_{\ast}}^2}\scal{\mathbf{1}}{u_{0}-u_{\ast}}}\left(1-\frac{\gamma\mu}{1+\gamma\mu}\right)^{\frac{t}{2}}\right)
\end{aligned}
\end{equation}
By setting
\begin{equation}\label{eq: tstar}
	t^{\ast}=\frac{\log(\frac{1}{4})-\log\left(\frac{\norme{w_{0}}^{2}}{\norme{w_{\ast}}^2}-1 +\frac{1}{\norme{w_{\ast}}^2}\scal{\mathbf{1}}{u_{0}-u_{\ast}}\right)}{\log\left(1-\frac{\gamma\mu}{1+\gamma\mu}\right)},
\end{equation} 
for all $t\geq t^{\ast}$, it holds $\norme{w_t}\geq \frac{1}{2}\norme{w_{\ast}}$.
From Lemma \ref{implicationrates}  we derive:
 \begin{equation} 1-\frac{\scal{w_{t}}{w_{\ast}}}{\norme{w_{t}}\norme{w_{\ast}}}\leq \frac{1}{\norme{w_{\ast}}^{2}}\norme{w_{t}-w_{\ast}}^2
\quad \text{ and } \quad 
M\Big(\frac{w_{\ast}}{\norme{w_{\ast}}}\Big)-M\Big(\frac{w_{t}}{\norme{w_{t}}}\Big)\leq \frac{2\norme{X}_{F}}{\norme{w_{\ast}}^{2}}\norme{w_{t}-w_{\ast}}
\end{equation}
which, together with \eqref{eqforGDmarginangle}, allows to conclude the proof of Theorem \ref{basicteoGD}.

\end{proof}

\subsection{Proof of Theorem \ref{basicteoiGD}}\label{subsectioniFBdualhinge}

In this paragraph, we turn our attention to the convergence properties of Algorithm \ref{algodualinertialGD}, hence the proof of Theorem \ref{basicteoiGD}. The analysis is based on discrete Lyapunov-energy techniques that have recently  become very popular  for studying inertial schemes like Algorithm \ref{algodualinertialGD}. Our analysis follows the line of study adopted in a recent stream of papers such as \cite{su2016differential,chambolle2015convergence,attouch2018rate,attouch2018fast,apidopoulos2017convergence,apidopoulos2020convergence,calatroni2019accelerated} and their  related references. The proof of Theorem \ref{basicteoiGD} is based on the following proposition which provides some bounds for the dual objective function $D_{t}(u_{t})-D_{\infty}(u_{\ast})$.

\begin{prop}\label{ratesdualenergyinertial}
Let $u_{\ast}$ be a solution of the dual problem \eqref{dualminproblem} and let $\{u_{t}\}_{t\geq 0}$ be the sequence generated by  Algorithm \ref{algodualinertialGD} with $\alpha\geq 3$ and $\lambda_{0}\leq \nor{u_{\ast}}^{-1}$. Then  the following estimate holds true for the dual-objective function:
\begin{equation}\label{dualobjectiveestimate}
D_{t}(u_t)-D_{\infty}(u_{\ast})\leq \frac{C^2}{(t+\alpha-1)^{2}},
\end{equation}
where \begin{equation}
\label{eq:C}
C=(\alpha-1)\bigg(\big(D_{0}(u_{0})-D_{\infty}(u_{\ast})\big)+\frac{\norme{u_{0}-u_{\ast}}^2}{2\gamma}\bigg)^{1/2}.
\end{equation}
\end{prop}

\begin{proof}[\textbf{Proof of Proposition \ref{ratesdualenergyinertial}}]
	
Let $u_\ast\in \argmin D_{\infty}$ and $u_{t}$ the sequence generated by Algorithm \ref{algodualinertialGD} with $\lambda_{0}\leq \norme{u_{\ast}}^{-1}$. For all $\nu>0$, let us define the following auxiliary sequences:
\begin{align}
	r_{t}&=D_{t}(u_{t})-D_{\infty}(u_{\ast}) ~, \quad 
	\delta_{t} = \norme{u_{t}-u_{t-1}}^2 ~, \quad
	h_{t}=\norme{u_{t}-u_{\ast}}^2 \\
	k_{t}&=t+\alpha-1 ~ \text{ and } ~ v_{t}=\norme{\nu(u_{t-1}-u_{+})+k_{t}(u_{t}-u_{t-1})}^2
\end{align}

The following energy sequence that will play a fundamental role in our analysis
\begin{equation}\label{deflyapunov}
	\begin{aligned}
		E_{t}&= (t+\alpha-1)^2\big(D_{t}(u_{t})-D_{\infty}(u_{\ast})\big) +\frac{1}{2\gamma}\norme{\nu(u_{t-1}-u_{\ast})+k_{t}(u_{t}-u_{t-1})}^2   \\ &=k_{t}^2r_{t}  +\frac{1}{2\gamma}v_{t} 
	\end{aligned}
\end{equation}
We will show that by tuning properly the parameters $\alpha$ and $\nu$, the sequence $E_{t}$ is non-increasing.

The Descent Lemma \ref{descentlemma} (with $y= q_{t}$ and $x=u_{\ast}$), implies
\begin{equation}\label{proofinertial1}
\begin{aligned}
2\gamma \big(D_{t}(u_{t+1})-D_{t}(u_{\ast})\big) & \leq \norme{u_{t}-u_{\ast} + \alpha_{t}(u_{t}-u_{t-1})}^2-h_{t+1} \\ & = h_{t}-h_{t+1} +\alpha_{t}^2\delta_{t}+2\alpha_{t}\scal{u_{t}-u_{t-1}}{u_{t}-u_{\ast}} \\ & = h_{t}-h_{t+1} +\alpha_{t}^2\delta_{t} + \alpha_{t}\big(\delta_{t}+h_{t}-h_{t-1}\big)
\end{aligned}
\end{equation}
By multiplying \eqref{proofinertial1} with $\nu k_{t+1}$ and using the definition of $\alpha_{t}$ we derive:
\begin{equation}\begin{aligned}
2\gamma\nu k_{t+1}\big(D_{t}(u_{t+1})-D_{t}(u_{\ast})\big) &\leq \nu k_{t+1}\big(h_{t}-h_{t+1}\big) +\nu k_{t+1}\alpha_{t}^2\delta_{t}+ \nu k_{t+1}\alpha_{t}\big(\delta_{t}+h_{t}-h_{t-1}\big) \\&= \nu k_{t+1}\big(h_{t}-h_{t+1}\big) +\frac{ \nu
t^2}{k_{t+1}}\delta_{t}+ \nu t\big(\delta_{t}+h_{t}-h_{t-1}\big)
\end{aligned}
\end{equation}

On the other hand, the Descent Lemma \ref{descentlemma} (with $y= q_{t}$ and $x=u_{t}$) yields:
\begin{equation}
2\gamma\big(D_{t}(u_{t+1})-D_{t}(u_{t})\big) \leq \alpha_{t}^2\delta_{t}-\delta_{t+1}
\end{equation}
which by multiplying by $k_{t+1}^2\big(1-\frac{\nu}{k_{t+1}}\big)$ (here $\nu\leq \alpha+1$) and using the definition of $\alpha_{t}$, gives
\begin{equation}\label{proofinertial2}
2\gamma \big(k_{t+1}^2-\nu k_{t+1}\big)\big(D_{t}(u_{t+1})-D_{t}(u_{t})\big) \leq \big(t^2-\frac{\nu t^2}{k_{t+1}}\big)\delta_{t}-\big(k_{t+1}^2-\nu k_{t+1}\big)\delta_{t+1}
\end{equation}

By adding relations \eqref{proofinertial1} and \eqref{proofinertial2}, we obtain
\begin{equation}\label{proofinertial3}
\begin{aligned}
2\gamma \big(k_{t+1}^2 D_{t}(u_{t+1})-k_{t+1}^2D_{t}(u_{t}) & -  \nu k_{t+1}(D_{t}(u_{t})-D_{t}(u_{\ast}))\big) \leq   \big(t^2-\frac{\nu t^2}{k_{t+1}}\big)\delta_{t} -\big(k_{t+1}^2-\nu k_{t+1}\big)\delta_{t+1} \\ & \qquad \qquad +\nu k_{t+1}\big(h_{t}-h_{t+1}\big) +\frac{ \nu
t^2}{k_{t+1}}\delta_{t}+ \nu t\big(\delta_{t}+h_{t}-h_{t-1}\big) \\
&=t^2\delta_{t}-k_{t+1}^2\delta_{t+1} + \nu t\big(\delta_{t}+h_{t}-h_{t-1}\big)-\nu k_{t+1}\big(h_{t+1}-h_{t}-\delta_{t+1}\big)
\end{aligned}
\end{equation}

Next, by using the non-increasing property of $D_{t}$ (in particular $D_{t+1}(\cdot)\leq D_{t}(\cdot)$, see Lemma \ref{remarkdecreasingDt}) and adding and subtracting $(k_{t+1}^2-\nu k_{t+1}) D_{\infty}(u_{\ast})$ on the left-hand side of \eqref{proofinertial3} and using the definition of $r_{t}$, we derive:
\begin{equation}
\begin{aligned}
2\gamma \big(k_{t+1}^2r_{t+1}-(k_{t+1}^2-&\nu k_{t+1})r_{t}- \nu k_{t+1}(D_{t}(u_{\ast})-D_{\infty}(u_{\ast}))\big) \leq \\ & t^2\delta_{t}-k_{t+1}^2\delta_{t+1} + \nu t\big(\delta_{t}+h_{t}-h_{t-1}\big)-\nu k_{t+1}\big(h_{t+1}-h_{t}-\delta_{t+1}\big).
\end{aligned}
\end{equation}
Since $\lambda_{t}\leq \norme{u_{\ast}}^{-1}$, we have $\lambda_{t}u_{\ast}\in [-1,0]^{n}$ and $D_{t}(u_{\ast})=D_{\infty}(u_{\ast})$, for all $t\geq 1$, thus from the previous relation we get:
\begin{equation}
2\gamma \big(k_{t+1}^2r_{t+1}-(k_{t+1}^2-\nu k_{t+1})r_{t}\big) \leq  t^2\delta_{t}-k_{t+1}^2\delta_{t+1} + \nu t\big(\delta_{t}+h_{t}-h_{t-1}\big)-\nu k_{t+1}\big(h_{t+1}-h_{t}-\delta_{t+1}\big)
\end{equation}

By adding and subtracting $k_{t}^2r_{t}$ on both sides of the previous inequality, we have
\begin{equation}\label{forproofinertial1}
\begin{aligned}
2\gamma k_{t+1}^2r_{t+1}- 2\gamma k_{t}^2r_{t} & \leq 2\gamma\rho_{t+1}r_{t}  + t^2\delta_{t}-k_{t+1}^2\delta_{t+1} +\nu t\big(\delta_{t}+h_{t}-h_{t-1}\big)  \\ & \qquad  -\nu k_{t+1}\big(h_{t+1}-h_{t}-\delta_{t+1}\big)
\end{aligned}
\end{equation}
with $\rho_{t+1}=k_{t+1}^2-\nu k_{t+1}-k_{t}^2$.

By considering now the variation of the sequence $v_{t}$ and performing some basic algebraic computations we have:
\begin{equation}\label{forproofinertial2}
\begin{aligned}
v_{t+1}-v_{t} & = \nu^2\big(\norme{u_{t}-u_{\ast}}^2-\norme{u_{t-1}-u_{\ast}}^2\big) + k_{t+1}^2\norme{u_{t+1}-u_{t}}^2- k_{t}^2\norme{u_{t}-u_{t-1}}^2 \\
& \qquad + 2\nu k_{t+1}\scal{u_{t+1}-u_{t}}{u_{t}-u_{\ast}} -2\nu k_{t}\scal{u_{t}-u_{t-1}}{u_{t-1}-u_{\ast}} \\
&= \nu^2\big(h_{t}-h_{t-1}\big) + k_{t+1}^2\delta_{t+1}- k_{t}^2\delta_{t} + \nu k_{t+1}\big(h_{t+1}-h_{t}-\delta_{t+1}\big) -\nu k_{t}\big(h_{t}-h_{t-1}-\delta_{t}\big)
\end{aligned}
\end{equation}

By adding $v_{t+1}-v_{t}$ on both sides of \eqref{forproofinertial1} and using \eqref{forproofinertial2}, we obtain:
\begin{equation}\label{forproofinertial3}
\begin{aligned}
2\gamma k_{t+1}^2r_{t+1} +v_{t+1}- 2\gamma k_{t}^2r_{t}-v_{t} & \leq  2\gamma\rho_{t+1}r_{t}+ \big(t^2 +\nu t+\nu k_{t}-k_{t}^2\big)\delta_{t} \\ &
\quad  +\nu\big(\nu+t-k_{t}\big)\big(h_{t}-h_{t-1}\big)
\end{aligned}
\end{equation}
Since $k_{t}=t+\alpha-1$, \eqref{forproofinertial3} and the definition of $E_{t}$ \eqref{deflyapunov}, yield:
\begin{equation}
\begin{aligned}
2\gamma\big(E_{t+1}-E_{t}\big) & \leq  2\gamma\rho_{t+1}r_{t}+(\nu-\alpha+1)\big(2t+\alpha-1\big)\delta_{t}   +\nu(\nu-\alpha+1)\big(h_{t}-h_{t-1}\big)
\end{aligned}
\end{equation}
Thus, setting $\nu=\alpha-1$ and $\alpha\geq3$ (here notice that $\rho_{t+1}=k_{t+1}^2-\nu k_{t+1}-k_{t}^2=(3-\alpha)k_{t}+2-\alpha$), we get:
\begin{equation}
E_{t+1}-E_{t} \leq 0,
\end{equation}
which shows that the energy sequence $E_{t}$ is non-increasing.
In addition by neglecting the non-negative terms in the definition of $E_{t}$ \eqref{deflyapunov} and using its non-increasing property, we derive: 
\begin{equation}
k_{t}^2\left(D_{t}(u_{t}))-D_{\infty}(u_{\ast}\right)\leq E_{t}\leq E_{0}
\end{equation}

Finally by the definition of $E_{t}$ \eqref{deflyapunov}, we get:
\begin{equation}
r_{t}=D_{t}(u_{t})-D_{\infty}(u_{\ast})\leq \frac{C}{\big(t+\alpha-1\big)^2}
\end{equation}
\begin{equation}
\begin{aligned}
\text{with : } \qquad C& =E_{0}=(\alpha-1)^2\bigg(2\gamma\big(D_{0}(u_{0})-D_{\infty}(u_{\ast})\big)+\frac{\norme{u_{0}-u_{\ast}}^2}{2\gamma}\bigg)  \hspace{5cm}
\end{aligned}
\end{equation}
which concludes the proof of Proposition \ref{ratesdualenergyinertial}.
\end{proof}

\begin{proof}[\textbf{Proof of Theorem \ref{basicteoiGD}}] 
Lemma \ref{lemmadualprimalsequence} and the non-increasing property of $D_{t}$ given in Lemma \ref{remarkdecreasingDt}, together with the upper bound \eqref{dualobjectiveestimate} in Proposition \ref{ratesdualenergyinertial} for the dual iterates yield
\begin{equation}\label{eqformarginangleiGD}
\norme{w_{t}-w_{\ast}}\leq \sqrt{2\big(D_{\infty}(u_{t}) -  D_{\infty}(u_{\ast})\big)} \leq \sqrt{2\big(D_{t}(u_{t}) -  D_{\infty}(u_{\ast})\big)} \leq \frac{C}{t+\alpha-1},
\end{equation}
where $C=(\alpha-1)\bigg(2\big(D_{0}(u_{0})-D_{\infty}(u_{\ast})\big)+\frac{\norme{u_{0}-u_{\ast}}^2}{\gamma}\bigg)^{1/2}$ as follows by \eqref{eq:C}. By using the definition of $D_{0}$ and $D_{\infty}$ in the bound \eqref{eqformarginangleiGD} allows to conclude the proof of the first part of Theorem \ref{basicteoiGD}.
 Regarding the margin and the angle gap rates,  we proceed as in the proof of Theorem \ref{basicteoGD}. By using the triangle inequality and \eqref{eqformarginangleiGD}, and setting
\begin{equation}\label{eq tstar2}
	t^{\ast} = (\alpha-1)\left(2{\left({\frac{\norme{w_{0}}^2}{\norme{w_{\ast}}^2}-1+2\scal{\mathbf{1}}{u_{0}-u_{\ast}}+\frac{\norme{u_{0}-u_{\ast}}^2}{\gamma\norme{w_{\ast}}^2} }\right)^{1/2}}-1\right),
\end{equation}
we deduce that for all $t\geq t^{\ast}$, we have $\norme{w_{t}}\geq \frac{1}{2}\norme{w_{\ast}}$.
Therefore, thanks to Lemma \ref{implicationrates}, the proof follows directly from the bound \eqref{eqformarginangleiGD}.
\end{proof}

\subsection{Proof of Theorem \ref{basicteostability}}\label{subsectionproofstability}

In this paragraph we provide the proof of Theorem \ref{basicteostability} regarding the stability properties of Algorithm \ref{algodualprojGD}, in presence of noise, as discussed in section \ref{section stability} .

\begin{proof}[\textbf{Proof of Theorem \ref{basicteostability}}] 
Without loss of generality let us assume that $S_{N}=\{1,\dots,N\}$, i.e. the set of flipped labels consists of the first $N$ indices (notice that up to a re-indexation one can always retrieve this case).	
Let also $\tilde{u}_{0}=u_{0}=0$ and $\{(\tilde{w}_{t}, \tilde{u}_{t})\}_{t\geq 1}$, $\{w_{t},u_{t}\}_{t\geq0}$ be the pair of sequences generated by Algorithm \ref{algodualprojGD} applied to the noisy $(X,\tilde{Y})$ and true $(X,Y)$ data respectively.

 

By using the triangle inequality, the Algorithm \ref{algodualprojGD} for $\tilde{w}_{t}$ and $w_{t}$ and the definition of the operator norm, for all $t\geq 0$, it holds:
\begin{equation}\label{eq 1 stability}
	\begin{aligned}
			\norme{\tilde{w}_{t} - w_{\ast}} & \leq \norme{\tilde{w}_{t}-w_{t}} + \norme{w_{t}-w_{\ast}}  \\
			& = \norme{\tilde{Z}^{\top}\tilde{u}_{t}-Z^{\top}u_{t}} +\norme{w_{t}-w_{\ast}}	\\
			& \leq \norme{\tilde{Z}^{\top}}_{op}\norme{\tilde{u}_{t}-u_{t}} + \norme{(\tilde{Z}^{\top}-Z^{\top})u_{t}} +\norme{w_{t}-w_{\ast}}
	\end{aligned}
\end{equation}

The bound in the right-hand side of \eqref{eq 1 stability} is formed by three parts, the first two are related with the stability properties of Algorithm \ref{algodualprojGD}, while the last one to its optimization one. By using the convergence result \eqref{eq: basicteoGD} in Theorem \ref{basicteoGD}, the bound \eqref{eq 1 stability} takes the following form:
\begin{equation}\label{eq stability primal}
	\norme{\tilde{w}_{t}-w_{\ast}} \leq \norme{\tilde{Z}^{\top}(\tilde{u}_{t}-u_{t})} + \norme{(\tilde{Z}^{\top}-Z^{\top})u_{t}} +C(1-\rho)^{\frac{t}{2}}
\end{equation}
where	$C=\sqrt{\norme{\Z^{\top}u_{0}}^2-\norme{\Z^{\top}u_{\ast}}^2 +2\scal{\mathbf{1}}{u_{0}-u_{\ast}}}$ and $\rho=\frac{\gamma\mu}{1+\gamma\mu}$.

From the definition of the noise model $\tilde{y}_{i}y_{i}=-1$, for all $i\in S_{N}$, the $i$-th row of matrix $\tilde{Z}^{\top}-Z^{\top}$ can be expressed as follows
\begin{equation}\label{matrix A}
(\tilde{Z}^{\top}-Z^{\top})_{i} =(\tilde{y}_{i}x_{i}-y_{i}x_{i}) = \begin{cases}
		-2x_{i} ~ & \text{ if }  \\
		0 & \text{ if } i \in S_{N}
  \end{cases}
\end{equation}

By using the expression \eqref{matrix A} the second term in the right-hand side of \eqref{eq 1 stability} can be bounded as follows:
\begin{equation}
\begin{aligned}\label{eq stability for second term}
	\norme{(\tilde{Z}^{\top}-Z^{\top})u_{t}} &=2\left(\sum_{i=1}^{N}(u_{t}^{i})^{2}\norme{x_{i}}^{2}\right)^{\frac{1}{2}} \leq 2K\sqrt{N}\norme{u_{t}} \\
	& \leq 2K\sqrt{N}\left(\norme{u_{t}-u_{\ast}} + \norme{u_{\ast}}\right)
	\end{aligned}
\end{equation}
where we used that $\norme{x_{i}}\leq K$, for all $i\in\{1,\dots,n\}$ and $\sum_{i=1}^{N}\abs{u_{t}^{i}}\leq \norme{u_{t}}$. In addition since $u_{t}$ is generated by Algorithm \ref{algodualprojGD}, it satisfies the Fejer property with respect to any minimizer $u_{\ast}$ of $D_{\infty}$, i.e. $\norme{u_{t}-u_{\ast}} \leq \norme{u_{0}-u_{\ast}}$ (this can be seen as an immediate consequence of Lemma \ref{lemmaenergyFB}). Therefore from \eqref{eq stability for second term}, for all $t\geq 1$, it follows:
\begin{equation}\label{eq stability 2}
\norme{(\tilde{Z}^{\top}-Z^{\top})u_{t}} \leq 2K\left(\norme{u_{0}-u_{\ast}} + \norme{u_{\ast}}\right)\sqrt{N}
\end{equation}

For the first term in the right-hand side of \eqref{eq 1 stability}, by using Algorithm \ref{algodualprojGD}, we obtain:	
\begin{equation}\label{recursive for u}
	\begin{aligned}
		\norme{\tilde{u}_{t+1}-u_{t+1}} &=\norme{\prox_{\frac{\gamma}{\lambda_{t}}\mathcal{L}^{\ast}(\lambda_{t}\cdot)}\left(\tilde{u}_{t}-\gamma \tilde{Z}\tilde{Z}^{\top}\tilde{u}_{t}\right) - 	\prox_{\frac{\gamma}{\lambda_{t}}\mathcal{L}^{\ast}(\lambda_{t}\cdot)}\left(u_{t}-\gamma ZZ^{\top}u_{t}\right)} \\
	&  = \norme{\lambda_{t}^{-1}\prox_{\gamma\lambda_{t}\mathcal{L}^{\ast}}\left(\lambda_{t}\left(\tilde{u}_{t}-\gamma \tilde{Z}\tilde{Z}^{\top}\tilde{u}_{t}\right)\right) - 	\lambda_{t}^{-1}\prox_{\gamma\lambda_{t}\mathcal{L}^{\ast}}\left(\lambda_{t}\left(u_{t}-\gamma ZZ^{\top}u_{t}\right)\right)} \\
 &	\leq \norme{\tilde{u}_{t}-\gamma\tilde{Z}\tilde{Z}^{\top}\tilde{u}_{t}-u_{t} +\gamma ZZ^{\top}u_{t}}  \\
 & \leq \norme{\left(\text{Id}-\gamma\tilde{Z}\tilde{Z}^{\top}\right)(\tilde{u}_{t}-u_{t})} + \gamma\norme{\left(ZZ^{\top}-\tilde{Z}\tilde{Z}^{\top}\right)u_{t}} \\
 & \leq \norme{\tilde{u}_{t}-u_{t}} + \gamma\norme{\left(ZZ^{\top}-\tilde{Z}\tilde{Z}^{\top}\right)u_{t}}
	\end{aligned}
\end{equation}
where in the second equality we used the scaling property of the proximal operator \cite[Theorem $6.12$]{beck2017first} and in the first inequality its non expansiveness property \cite[Proposition $12.28$]{bauschke2011convex}. The second inequality is the triangular inequality and the last one follows from the non-expansiveness of the operator $\text{Id}-\gamma\tilde{Z}\tilde{Z}^{\top}$ (since $\gamma\leq \norme{\tilde{Z}\tilde{Z}^{\top}}_{op}^{-1}$).

By applying recursively relation \eqref{recursive for u}, since $\tilde{u}_{0}=u_{0}$, for all $t\geq1$, it follows:
\begin{equation}\label{stability first term}
	\norme{\tilde{u}_{t}-u_{t}} \leq \sum_{s=0}^{t-1}\norme{Bu_{s}}
\end{equation}
where we used the notation $B:=\tilde{Z}\tilde{Z}^{\top} - ZZ^{\top}$.

By using the noise model $\tilde{y}_{i}y_{i}=-1$, for all $i\in S_{N}$, the matrix $B=\tilde{Z}\tilde{Z}^{\top} - ZZ^{\top}$ can be expressed element-wise as follows
\begin{equation}\label{matrix B}
	(B)_{i,j} =(\tilde{y}_{i}\tilde{y}_{j}-y_{i}y_{j})\scal{x_{i}}{x_{j}} = \begin{cases}
		-2\scal{x_{i}}{x_{j}} ~ & \text{ if }\left((i,j)\in S_{N}^{c}\times S_{N}\right)\cup\left((i,j)\in S_{N}\times S_{N}^{c}\right)  \\
		0 & \text{ if } \left((i,j)\in S_{N}\times S_{N}\right)\cup\left((i,j)\in S_{N}^{c}\times S_{N}^{c}\right)
	\end{cases}
\end{equation}

By the definition of the euclidean norm, and the expression \eqref{matrix B}, for any $u\in \R^{n}$, the term $\norme{Bu}$ can be bounded as follows:
\begin{equation}\label{stability for first term}
\begin{aligned}
	\norme{Bu}^{2} & = \norme{\left(\tilde{Z}\tilde{Z}^{\top} - ZZ^{\top}\right)u}^{2} \\
	& =  4\left( \left(\sum_{i=N+1}^{n}u_{i}\scal{x_{1}}{x_{i}}\right)^{2}  +\dots +\left(\sum_{i=N+1}^{n}u_{i}\scal{x_{N}}{x_{i}}\right)^{2}  + \left(\sum_{i=1}^{N}u_{i}\scal{x_{N+1}}{x_{i}}\right)^{2} +\dots + \left(\sum_{i=1}^{N}u_{i}\scal{x_{n}}{x_{i}}\right)^{2}  \right) \\
	& \leq 8\left( \left(\norme{x_{1}}^{2}+\dots+\norme{x_{N}}^{2}\right)\sum_{i=N+1}^{n}(u_{i})^{2}\norme{x_{i}}^{2}  +\left(\norme{x_{N+1}}^{2}+\dots+\norme{x_{n}}^{2}\right)\sum_{i=1}^{N}(u_{i})^{2}\norme{x_{i}}^{2}\right) \\
	& \leq 8\left(K^{4}N\sum_{i=N+1}^{n}(u_{i})^{2} + K^{4}(n-N)\sum_{i=1}^{N}(u_{i})^{2} \right) \\
	& = 8K^{4}\left(N\norme{u}^{2}+ (n-2N)\sum_{i=1}^{N}(u_{i})^{2} \right) \\
	& \leq 8K^{4}(n+1-2N)N\norme{u}^{2}
\end{aligned}
\end{equation}
where in the first inequality we used the convexity of the squared norm, in the second one the definition of $K=\underset{i\leq n}{\max}\{\norme{x_{i}}\}$ and in the last one, the convention $N\leq \frac{n}{2}$.

Therefore, by using \eqref{stability for first term} in \eqref{stability first term}, it follows that:
\begin{equation}\label{eq stability 3}
	\begin{aligned}
	\norme{\tilde{u}_{t}-u_{t}} & \leq 2\sqrt{2}K^{2}\sqrt{2(n+1-2N)N}\sum_{s=0}^{t-1}\norme{u_{s}} \\
	& \leq 2\sqrt{2}K^{2}\sqrt{2(n+1-2N)N}\sum_{s=0}^{t-1}\left(\norme{u_{s}-u_{\ast}} +\norme{u_{\ast}}\right)
	\\
		& \leq 2\sqrt{2}K^{2}\sqrt{2(n+1-2N)N}\left(\norme{u_{0}-u_{\ast}}+\norme{u_{\ast}}\right)t
\end{aligned}
\end{equation}
where we used the triangular inequality $\norme{u_{s}} \leq \norme{u_{s}-u_{\ast}} + \norme{u_{\ast}}$ and the  Fejer property of the sequence $u_{s}$ with respect to any minimizer $u_{\ast}$ of $D_{\infty}$, i.e. $\norme{u_{s}-u_{\ast}} \leq \norme{u_{0}-u_{\ast}}$ (as an immediate consequence of Lemma \ref{lemmaenergyFB}).

Finally, by using the bounds \eqref{eq stability 2} and \eqref{eq stability 3}, in the estimate \eqref{eq stability primal}, for all $t\geq 1$, it holds:
\begin{equation}\label{eq: stability final}
\norme{\tilde{w}_{t}-w_{\ast}} \leq C_{1}\sqrt{2(n+1-2N)N}t +  C_{2}\sqrt{N} +C(1-\rho)^{\frac{t}{2}}
\end{equation}
where $C_{1}=2\sqrt{2}K^{2}\left(\norme{u_{0}-u_{\ast}}+\norme{u_{\ast}}\right)$, $C_{2}=2K\left(\norme{u_{0}-u_{\ast}}+\norme{u_{\ast}}\right)$, 	$C=\sqrt{\norme{\Z^{\top}u_{0}}^2-\norme{\Z^{\top}u_{\ast}}^2 +2\scal{\mathbf{1}}{u_{0}-u_{\ast}}}$ and $\rho=\frac{\gamma\mu}{1+\gamma\mu}$.

The (worst-case) optimal stopping time $t_{\ast}(N):=\max\left\{1, \frac{2}{\ln\left(\frac{1}{1-\rho}\right)}\ln\left(\frac{C\ln\left(\frac{1}{1-\rho}\right)}{2C_{1}\sqrt{2(n+1-2N)N}}\right)\right\}$, as expressed in Theorem \ref{basicteostability} follows by optimizing the right-hand side of \eqref{eq: stability final} over $t$ and then evaluating it in \eqref{eq: stability final}, which yields \eqref{optimal bound stability} and conclude the proof of Theorem \ref{basicteostability}. 
\end{proof}

\phantomsection 
\addcontentsline{toc}{chapter}{Bibliography} 
\bibliographystyle{siam}
\bibliography{reference}
\end{document}